\documentclass[12pt, a4paper]{article}
\usepackage[utf8]{inputenc}
\usepackage[ruled,vlined]{algorithm2e}
\usepackage[title]{appendix}
\usepackage[T1]{fontenc} 
\usepackage[margin=1.15in, bottom =1.15in]{geometry}
\usepackage{graphicx} 
\usepackage{amssymb, amsmath, amsthm}
\usepackage{color}

\usepackage{soul}
\usepackage{tikz}
\usepackage[compatibility=false,labelfont=it,textfont={it}]{caption}
\usetikzlibrary{mindmap,shadows}
\usepackage{lscape}
\newtheoremstyle{mystyle}
  {}{}
  {\itshape}{}
  {\bfseries}{.}
  { }{\thmname{#1}\thmnumber{ #2}\thmnote{ (#3)}}

\theoremstyle{mystyle}

\newcommand{\prob}{\rho}

\newcommand{\Lie}[2]{\pounds_{#1}{#2}}  
\newcommand{\Ucal}{\mathcal{U}}
\numberwithin{equation}{section}
\def\contract{\makebox[1.2em][c]{\mbox{\rule{.6em}
{.01truein}\rule{.01truein}{.6em}}}}

\RequirePackage{color}
\definecolor{dullred}{RGB}{0,0,190}   
\usepackage{tikz}
\usepackage{tikz-cd}
\usetikzlibrary{intersections, backgrounds}
\PassOptionsToPackage{colorlinks, breaklinks%
     ,urlcolor=dullred
           ,citecolor=dullred
           ,linkcolor=dullred
           ,pagecolor=white
     }{hyperref}
\IfFileExists{hyperref.sty}{\RequirePackage{hyperref}}%
 {\typeout{hyperref.sty is not loaded so you will not get active hyperlinks}%
 \providecommand{\href}[2]{#2}%
}
\newtheorem{theorem}{Theorem}[section]
\newtheorem{corollary}[theorem]{Corollary}

\newtheorem{lemma}[theorem]{Lemma}
\newtheorem{definition}[theorem]{Definition}

\newtheorem{problem}[theorem]{Problem}
\newtheorem{remark}[theorem]{Remark}

\def\contrac{\mathbin{\hbox to 6pt{%
                 \vrule height0.4pt width5pt depth0pt
                 \kern-.4pt
                 \vrule height6pt width0.4pt depth0pt\hss}}}

\usepackage{tikz}
\usetikzlibrary{bayesnet}

\title{Multisymplectic Formulation of Deep Learning Using Mean--Field Type Control and Nonlinear Stability of Training Algorithm}
 
\author{Nader T. Ganaba \thanks{\href{mailto:nganaba@outlook.com}{\texttt{nganaba@outlook.com}}}}
\date{}

\begin{document}

\maketitle

\begin{abstract}
As it stands, a robust mathematical framework to analyse and study various topics in deep learning is yet to come to the fore. Nonetheless, viewing deep learning as a dynamical system allows the use of established theories to investigate the behaviour of deep neural networks.  In order to study the stability of the training process, in this article, we formulate training of deep neural networks as a hydrodynamics system, which has a multisymplectic structure. For that, the deep neural network is modelled using a stochastic differential equation and, thereby, mean--field type control is used to train it. The necessary conditions for optimality of the mean--field type control reduce to a system of Euler--Poincar\'{e} equations, which has the a similar geometric structure to that of compressible fluids. The mean--field type control is solved numerically using a multisymplectic numerical scheme that takes advantage of the underlying geometry.  Moreover, the numerical scheme, yields an approximated solution which is also an exact solution of a hydrodynamics system with a multisymplectic structure and it can be analysed using backward error analysis.  Further, nonlinear stability yields the condition for selecting the number of hidden layers and the number of nodes per layer, that makes the training stable while approximating the solution of a residual neural network with a number of hidden layers approaching infinity.
\end{abstract}
 
\section{Introduction}
Machine learning has gained an incredible amount of attention in recent years, partly thanks to the recent achievement and successes in solving real-world problems. Problems range from  data-driven drug discovery such as using autoencoders \cite{gomez2018automatic} and reinforcement learning to generate new molecules \cite{olivecrona_molecular_2017}, to others for reducing the dimensionality of single-cell RNA sequencing \cite{wang_vasc_2018}. Not just computational bioinformatics, but also image recognition \cite{farabet_learning_2013,szegedy_going_2014, krizhevsky_imagenet_2017}, and controlling vehicles \cite{zhu_safe_2020,folkers_controlling_2019}  to name a few. As architectures become more elaborate, training them becomes more challenging and this has motivated the influx of techniques from other mathematical disciplines, such as optimal control \cite{weinan2017proposal, li2017maximum,weinan2019mean} and optimal transport \cite{salimans_improving_2018}. Despite that, there is still no mathematical framework in which to examine the behaviour of deep networks, or the convergence of training methods.   In \cite{benning2019deep}, optimal control theory was nominated as a framework for studying the stability of training, and in \cite{Haber_2017}, it was studied by examining the stability discretisation methods for ordinary differential equations. As optimal control problems are linked to fluid dynamics \cite{bloch2000optimal}, we extrapolate on the ideas of \cite{Haber_2017,benning2019deep} by considering the optimal control problem from the perspective of multisymplectic formulation of hydrodynamics.  Along with \cite{ganaba_deep_2021-2}, the aim is to translate deep learning into the language of geometric mechanics in order to gain more insight into the comportment of deep networks. Profiting from the multisymplectic nature formulation, in this article, we introduce a training algorithm based on a multisymplectic numerical scheme. This type of schemes is the exact solution of another dynamical system whose Lagrangian and Hamiltonian functions are found using backward error analysis. Knowing the Hamiltonian for the numerical scheme, we use energy--Casimir method to find the conditions for the nonlinear stability of the numerical scheme. One of the benefits of studying the nonlinear stability is that it provides a method to heuristically select the network's parameters such number of hidden layers and nodes per layer.
\\

The route to this formulation starts with the fact that, as pointed out in \cite{weinan2017proposal}, residual neural networks, also referred to as \emph{ResNets}, can be regarded as the Euler discretisation of a dynamical system. Thus, it is possible to take the continuous limit and treat it as a neural network with an infinite number of hidden layers. To train such networks, one possible way is to use optimal control.  The solution of the optimal control problem is obtained using Pontryagin's Maximum Principle \cite{doi:10.1002/zamm.19630431023}, which states that the trajectory of the dynamical system solves a Hamiltonian system.  Thus, Pontryagin's Maximum Principle provides a passage to symplectic geometry \cite{agrachev1990symplectic}. 
\\

Optimal control training methods use ordinary differential equation for ResNets \cite{li2017maximum,Haber_2017,benning2019deep}, here, we use a more general representation that accounts for uncertainties. Thus, optimal control problem becomes a stochastic one and that introduces additional challenges. There are many ways to solve stochastic optimal control problems and here we limit ourselves to only mean--field type control. The reasoning behind this choice is that mean--field type control formulates a stochastic control problem as a deterministic one, and in theory. Mean--field type control originated from mean--field games, which is a mathematical framework developed to control an ensemble of identical agents and it was formulated in \cite{huang_large_2006}. At the same time and independently, mean--field games was introduced in \cite{lasry_jeux_2006, lasry_jeux_2006-1,lasry_mean_2007} as a framework for modelling differential games of an infinite number of agents in a population. Its treatment for optimal control problem was later discussed in \cite{bensoussan_mean_2013,bensoussan_general_2013} and references therein. For deep learning, it was employed in in \cite{weinan2019mean}, where training was done using mean--field games, although consists of a system of forward and backward stochastic differential equations, known as McKean--Vlasov dynamics \cite{carmona_control_2013,carmona2015forward}.  On the contrary, the mean--field games system used here consists of a system of a coupled Fokker--Planck and Hamilton--Jacobi--Bellman equations and both are partial differential equations.  As mentioned, for the deterministic case, the optimal solution satisfies a Hamiltonian system, which is a covariant Hamiltonian system in this case. That said, symplectic geometry does not cover covariant Hamiltonian systems, and for this reason, multisymplectic geometry is required, which is a generalisation of symplectic geometry \cite{marsden1998multisymplectic}.
\\

Despite it alluded to in \cite{marsden1998multisymplectic}, the link between multisymplectic geometry and Pontryagin's Maximum Principle is not well explored. When the network parameters, treated as control, are elements of the Lie algebra associated with the group of diffeomorphisms, the optimal control problem becomes a special case known as \emph{Clebsch optimal control} \cite{gay2011clebsch}. This is vital also for reduction. In \cite{doi:10.1098/rspa.2007.1892}, it was shown that the stationary variations of the Clebsch variational principle for hydrodynamic systems yield a multysymplectic system, which is the a system of three equations: Fokker--Planck, Hamilton--Jacobi--Bellman and the optimality condition. The Fokker--Planck equation is a forward partial differential equation, while the Hamilton--Jacobi--Bellman equation is a backward partial differential equation. Using the condition that the network's weights are vector fields of the group of diffeomorphisms, the mean--field games system is reduced, by eliminating the Hamilton--Jacobi--Bellman equation, to a system of forward equations. The result is a system of hydrodynamic equations known as \emph{Euler--Poincar\'{e} equation with advected density}\cite{holm1998euler}. Many hydrodynamic equations, such as the Kortweg-de Vries equation \cite{ovsienko1987korteweg}, the Camassa-Holm equation \cite{PhysRevLett.71.1661, doi:10.1063/1.532690} are part of the Euler--Poincar\'{e} family of equations. From the calculus of variations, the trajectory that minimises a cost functional satisfies the  Euler-Lagrange equation, and this case is the mean--field system, and when there is symmetry, the variational problem is reduced and the optimal trajectory satisfies an equivalent the Euler--Poincar\'{e} equation. With that, the claim here is that the weights of the deep residual neural network satisfy a variant of the Euler--Poincar\'{e} equation with advected quantity on the group of diffeomorphisms.
\\

As the nature of data is discrete, the optimal control problem needs to be projected to a finite-dimensional space by using numerical schemes. For systems whose dynamics are described by an ordinary differential equation, the Hamiltonian equations of motion for the optimal control problems have a symplectic structure, hence geometric integrators can be used to solve optimal control problems numerically as it was done \cite{junge_discrete_2005,ober-blobaum_discrete_2011,de_leon_discrete_2007}. The advantage of using geometric integrators is their stability and preservation of the underlying geometry, thus eliminating the possibility of the emergence of unrealistic solutions.   Due to the fact that mean-field type optimal control is multisymplectic in nature, we thus use multisymplectic integrators derived using a discrete variational method \cite{marsden1998multisymplectic}, as an extension to discrete mechanics optimal control of \cite{junge_discrete_2005,ober-blobaum_discrete_2011,de_leon_discrete_2007}. This projection is important, as it provides a systematic way to train a network with a finite number of hidden layers, and concurrently simulates the behaviour of a network with an infinite number of hidden layers. To verify the presented training algorithm using a multisymplectic variational integrator, two toy problems are solved: a density matching problem from variational inference and images of handwritten digits generation of using a variational autoencoder. It is worth pointing out that a multisymplectic integrator can be applied to problems in machine learning in a different way as in \cite{barbaresco2020lie}, where the data points are elements of a Lie group. Whereas, here, the network parameters belong to a Lie algebra, and the multisymplectic structure arises from mean--field type control for stochastic optimal control. The data need not to belong to a Lie group to use our approach to multisymplectic integrators.
\\

Reducing the optimal control problem to a hydrodynamics problem and solving it using a multisymplectic scheme allows us to employ the same methods used to analyse the stability of fluids \cite{arnold_priori_2014,arnold_conditions_2014,HOLM19851,HOLM198315} to determine the stability of the training algorithm. An intermediate step is required to carry out stability analysis is backward error analysis. In short, backward error analysis seeks to find an equation, known as \emph{the modified equation}, whose exact solution is the numerical solution obtained by using the multisymplectic integrator. In \cite{islas_backward_2005}, it was shown that multisymplectic integrators do have a multisymplectic modified equations. For our analysis, the modified equation helps to derive the error equation for the numerical error and show that it is nonlinearly stable. This means if the error is small enough, then it remains bounded.  To establish nonlinear stability of a system, we use a method called \emph{energy--Casimir method} \cite{HOLM19851}. It is an algorithmic approach to constructing a norm for determining if the equilibrium solution is nonlinearly stable. This is done by adding a conserved quantity known as the Casimir function to the Hamiltonian.  Network parameters such as the number of hidden layers, number of nodes and learning rate appear in the modified equation, and nonlinear stability provides conditions for the said parameters to ensure that the training problem converges.  The main difference between the approach taken here and \cite{benning2019deep}, is that the latter is based on exponential stability and uses Lyapunov's notion of stability. 
\\

\emph{The outline of this paper:} Section \ref{sec:prelim} gives a brief review of topics, and notions that are used throughout this article. The recap starts with giving the most important definitions of multisymplectic geometry, in section \ref{sec:mul_geo}, and then moves to optimal control theory in section \ref{sec:opt_cont}. We also give a statement of the Pontryagin's Maximum Principle, without providing the proof as it is too technical to review here. The review is concluded with a thumbnail sketch of mean--field type control \ref{sec:mfg}. In section \ref{sec:mul_geo_dl}, a formulation of deep learning problem as a Fokker--Planck Hamilton--Jacobi--Bellman mean--field games system is presented and such problem is shown to be a multisymplectic system.  The mean--field games system is reduced to obtain a Euler--Poincar\'{e} equation with advected parameter. A numerical scheme for the multisymplectic system for deep learning is presented in  section \ref{sec:var_int} and it is based on multisymplectic variational integrator and it is demonstrated how it can be used for practical problems: density matching, and variational autoencoder, discussed in section \ref{sec:density} and \ref{sec:autoencoder}, respectively. Backward error analysis for the multisymplectic scheme is carried out in section \ref{sec:bea}. Section \ref{sec:conserv_dl} concerns with the study of the nonlinear stability of multisymplectic scheme, where the conserved quantities and Noether's theorem for deep learning are discussed in section \ref{sec:conserv_dl}. A Hamiltonian formulation is presented in section \ref{sec:Hamiltonian} and it is need to apply the energy--Casimir method, which is reviewed in section \ref{sec:ECM}. The main result of nonlinear stability is presented in section \ref{sec:nonlinear_stability} along with conditions on the number of nodes and hidden layers that guarantee stable training.

\section{Preliminaries}\label{sec:prelim}
\subsection{Review of multisymplectic geometry}\label{sec:mul_geo}
In this section, we provide a brief description of multisymplectic geometry to familiarise the reader with the notation and basic concepts used throughout this paper. Most of the definitions and concepts mentioned here are discussed in detail in \cite{marsden1998multisymplectic,marsden_multisymplectic_1999}. The core of the numerical algorithm of section \ref{sec:var_int} is rooted in multisymplectic geometry and the idea is to discretise the infinite-dimensional equations to finite-dimensional ones in a way that retains most of the geometry. Although this adds an extra level of complexity, it is important in understanding the behaviour of the algorithm. 
\\

The interest regarding multisymplectic geometry stemmed from the attempts at finding a formulation that would generalise classical Hamiltonian dynamics to account for covariant field-theoretic Hamiltonian \cite{de2011generalized,carinena_multisymplectic_1991,gotay_multisymplectic_1991, echeverria_enriquez_variational_1992, gotay_multisymplectic_1991-1}. One of the first encouraging examples of the promises of multisymplectic formulation was the recovery of the symplectic structure of the Korteweg-de Vries equation, which has a covariant Hamiltonian density \cite{gotay_multisymplectic_1988}. The generalisation to a larger family of Hamiltonian partial differential equations, written in an abstract form, 
\begin{align}
    M \partial_t z + K \partial_x z = \partial_z S(z),
\end{align}
was introduced in \cite{bridges1997multi}. A Lagrangian equivalent to the multisymplectic formulation based on variational principle was presented in \cite{marsden1998multisymplectic}. Most importantly, the aforementioned paper introduced variational integrators for partial differential equations. The link between both Hamiltonian and Lagrangian versions of multisymplectic theory was explored in detail in \cite{marsden_multisymplectic_1999}.
\\

Let $Q$ be a smooth manifold. In Hamiltonian mechanics, the trajectories belong to the cotangent bundle, $T^{\ast}Q$, and for configuration manifold, as opposed to the tangent bundle, $TQ$, in Lagrangian mechanics. A curve on $Q$ is parameterised by an independent variable, in many cases the independent variable is time $t$, and without the loss of generality, let $t \in B$, where $B \subset \mathbb{R}$, be the base manifold. The image generated by $q(t)$ defines a subset of manifold  $Q$. We also have the preimage $\pi: Q \to B$. When the curve $q$ depends on more than one variable, the base manifold $B$ is an $n+1$ dimensional manifold, with $n$ spatial dimensions and one temporal dimension. For example, when we have a $(1 + 1)$ partial differential equation the base manifold is chosen to be $B := [l_0, l_1]\times [t_0,t_1] \subset \mathbb{R}^2$. The configuration manifold is instead a fibre bundle $\pi_{BE}: E \to B$. The preimage $\pi^{-1}_{BE}: B \to E$  of $E$ over a point $x \in B$ denotes the \emph{fibre at point $x$}.  A continuous function $\phi: B \to E$ is the \emph{section} of $\pi_{BE}$ and we use $\Gamma(E)$ to denote all the section on $E$. In the language of partial differential equations, a section in $\Gamma \left( \pi_{BE} \right)$ corresponds to the solution of the partial differential equation. 
\\

The rate of change of the curve $q(t) \in Q$ with respect to $t$, denoted by $\dot{q}(t)$, is the derivative of $q(t)$ and it belongs to the tangent bundle $TQ$, and the covariant analogue of that is \emph{the first jet bundle} $J^1( E)$. It is defined as a bundle over $ E$, whose fibres are the mapping $\gamma: T_x  B \to T_y  E$  such that $T_{\pi_{BE} }\circ \gamma = \mathsf{Id}_{T_{x}B } $. To put it into other words, a first jet bundle is an equivalence relation between two curves evaluated at the same point that are equal up to their first-order derivatives. Using the same construction, once can construct higher jet bundles, the $(n+1)^{th}$ jet bundle is a fibre bundle on the $(n)^{th}$ jet bundle.
\begin{figure}[h]
\centering
\includegraphics[width=0.45\textwidth]{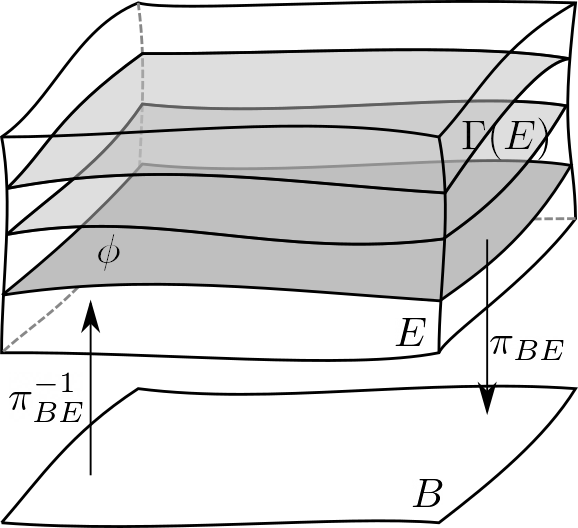}
\caption{A pictorial representation of a fibre bundle $\pi_{BE}: E \to B$. }
\end{figure}
As mentioned before, the base manifold $ B $ is an $n+1$ dimensional manifold, and its coordinates are $x^{\mu}$ where $\mu = 1, \dots, n, 0$ and the fibres on  $E$ have coordinates $y^A$, with $A = 1, \dots, N$. Meanwhile, the first jet bundle  $J^1( E)$'s coordinates are denoted by $v_{\mu}^A$. Now, we are interested in sections on $J^1( E)$ and specifically a section for $B \to J^1( E)$  and to do so, we define the section $\phi: B \to E$ and to ``lift it'' to $J^1( E)$, we take its tangent map $T_x \phi$. The section on  $J^1( E)$ over $B$ is denoted by $j^1(\phi)$ and it is known as the \emph{first jet prolongation} and in coordinates it is given by
\begin{align*}
x^{\mu} \mapsto \left(x^{\mu}, \phi^A, \frac{\partial}{\partial x^{\mu} }\phi^A \right).
\end{align*}
Cotangent bundles are regarded as the dual of the tangent bundle. The same is true here, \emph{the dual of the first jet bundle} $J^1( E)^{\ast}$ is the covariant analogue of the cotangent bundle and it is the space of maps of $J^1( E)$ to $\Lambda^{n+1}(B)$, the bundle of $n+1$-forms over $B$. The coordinates of $J^1( E)^{\ast}$ are given by
\begin{align*}
(p, p_A^{\mu}), \quad v_{\mu}^A \to (p + p_A^{\mu}v_{\mu}^A) d^{n+1}x,
\end{align*}
and here the differential form $d^{n+1}x$ is $dx^1 \wedge dx^2 \wedge \dots \wedge dx^n \wedge dt$. Next quantities that are important to the multisymplectic formulation are the $n+1$ and $n+2$ forms that are covariant counterparts of the canonical one and two-forms in symplectic geometry. These forms play a pivotal role in multisymplectic geometry as we shall see later. We first define the projection $\pi_{E(J^1( E)^{\ast})} :J^1( E)^{\ast} \to E$ is the projection from the bundle of $\Lambda^{n+1}$ forms to $E$, then we define the sub-bundle $Z$ by
\begin{align*}
Z:= \left\{ z \in \Lambda^{n+1}_y \, | \,V \contract (W \contract z)  = 0, \quad V, W, \in V_yE\right\},
\end{align*}
where $V_yE$ is the vertical sub-bundle of $TE$ defined by
\begin{align*}
V_yE := \left\{ W \in T_yE \, | \,  T\pi_{BE} \cdot W = 0 \right\},
\end{align*}
and it can be described as the space of elements of $T_yE$, which renders the derivative of $\pi_{BE}$ along their direction to be zero. The elements of $Z$ are defined as
\begin{align*}
z = p d^{n+1}x + p_A^{\mu} \mathrm{d}y^A \wedge \left( \partial_{\mu} \contract d^{n+1}x  \right).
\end{align*}
The canonical $n+1$ form $\Theta_{\Lambda}$ on $\Lambda^{n+1}$ is defined by
\begin{align*}
\Theta_{\Lambda}(z) \left( u_1, \dots, u_{n+1} \right) = \left( \pi_{E \Lambda}^{\ast} z\right) \left( u_1, \dots, u_{n+1} \right),
\end{align*}
where $\pi_{E \Lambda}: J^1( E)^{\ast} \to E$ the projection from the space of $n+1$ forms to $E$ and $u_1, \dots, u_{n+1}  \in T_{\Lambda} \Lambda^{n+1}$ and in coordinates that is
\begin{align*}
\Theta_{\Lambda} = p_A^{\mu} dy^A \wedge \left( \partial_{\mu} \contract d^{n+1}x \right) + p d^{n+1}x.
\end{align*}
The canonical $n+2$ form $\Omega$ is defined by $\Omega = - d \Theta_{\Lambda}$ and in coordinates it is given by
\begin{align}
\Omega = dy^A \wedge dp_A^{\mu}\wedge \left( \partial_{\mu} \contract d^{n+1}x \right) - dp\wedge d^{n+1}x.
\end{align}
The Lagrangian density $\mathcal{L}: J^1(E) \to \Lambda^{n+1}(B)$
\begin{align}\label{eq:lagrange_density}
\mathcal{L} = L(j^1(\phi)) d^{n+1}x,
\end{align}
and the Legendre transformation is the map $\mathbb{F}\mathcal{L} : J^1(E) \to J^1(E)^{\ast}$ and it is defined by $\mathbb{F}\mathcal{L}(\gamma) \cdot \gamma' = \mathcal{L}(\gamma) + \left. \frac{d}{d\epsilon} \right|_{\epsilon = 0} \left( \mathcal{L}( \gamma + \epsilon (\gamma' - \gamma)) \right)$, where $\gamma \in J^1(E)$. This is used to define the \emph{Cartan form}, which is the result of the variational principle. The Cartan form is known as the Lagrange one-form in the finite dimensional case. What is different regarding the Cartan form, $\Theta_{\mathcal{L}}$ known as \emph{the multisymplectic form}, is that it an $n+1$ form defined on $J^1(E)$ instead of $J^1(E)^{\ast}$
\begin{align}\label{eq:multisymplectic_form}
\Theta_{\mathcal{L}} = \left( \mathbb{F} \mathcal{L} \right)^{\ast} \Theta,
\end{align}
and in coordinates it is expressed explicitly by
\begin{align}
\Theta_{\mathcal{L}} = \left( \frac{\partial L}{\partial  \phi_{\mu}^A} \right)d\phi_{\mu}^A \wedge \left( \partial_{\mu} \contract d^{n+1}x \right) + \left( L - \frac{\partial L}{\partial  \phi_{\mu}^A} \phi_{\mu}^A \right) d^{n+1}x
\end{align}
Similarly, the $n+2$-from $\Omega_{\mathcal{L}}$ on $J^1(E)$ is defined by
\begin{align*}
\Omega_{\mathcal{L}} = - d \Theta_{\mathcal{L}} = \left( \mathbb{F} \mathcal{L} \right)^{\ast} \Omega.
\end{align*}
Using the Lagrangian density \eqref{eq:lagrange_density}, the action functional is defined as
\begin{align}\label{eq:action_multisymplectic}
\mathcal{S}(\phi) = \int \mathcal{L}(j^1(\phi \circ \phi_B^{-1})),
\end{align}
where $\phi: U \to E$,  $U$ is a smooth manifold with smooth closed boundaries  and $\phi_{B} = \pi_{BE} (\phi) : U \to B$. Let $\Phi : G \times \mathcal{F} \to \mathcal{F}$, where $\mathcal{F}$ is the set of smooth maps, defined by
\begin{align}
\Phi(\eta_{E}, \phi) = \eta_{E} \circ \phi.
\end{align}
The map $\eta_{E}: E \to E$ is a diffeomorphism and analogously, on the base space  $\eta_{B}: B \to B$. The purpose of using such diffeomorphism is that we can construct a family of sections on $E$. In this particular case, let $\eta^{\epsilon}_{E} $ be a smooth curve in $G$ parameterised by $\epsilon$ and $\eta^{0}_{E}  = e$ the identity of $G$. Also, let
\begin{align}
V = \left. \frac{d}{d\epsilon} \right|_{\epsilon = 0} \Phi(\eta^{\epsilon}_{E}, \phi), \quad \text{and} \quad V_{B} = \left. \frac{d}{d\epsilon} \right|_{\epsilon = 0} \eta^{\epsilon}_{B} \circ \phi,
\end{align}
where $V$ is a vector field for the transformation group on $E$ and $V_{B}$ is a vector field for the transformation group on $B$. With that, the variations of the action functional \eqref{eq:action_multisymplectic} is  
\begin{align}
\delta \mathcal{S}(\phi) = \left. \frac{d}{d \epsilon} \right|_{\epsilon = 0} \mathcal{S}(\Phi(\eta^{\epsilon}_{E}, \phi)) =  \left. \frac{d}{d \epsilon} \right|_{\epsilon = 0} \int \mathcal{L}(j^1(\Phi(\eta^{\epsilon}_{E}, \phi) \circ \phi_B^{-1})),
\end{align}
and using the property of diffeomorphisms $\eta^{\epsilon}_{E} $ and $\eta^{\epsilon}_{B} $ on sections 
\begin{align*}
    j^1(\Phi(\eta^{\epsilon}_{E}, \phi) \circ \phi_B^{-1}) = j^1( \eta_{E} \circ \phi \circ \phi_B^{-1} \circ \eta_{B}^{-1}) &= j^1(\Phi(\eta^{\epsilon}_{E}, \phi) \circ \phi_B^{-1}) \\
    &= j^1( \eta_{E}) \circ j^1(\phi \circ \phi_B^{-1} ) \circ \eta_{B}^{-1}
\end{align*}
and then the variations of the action along the vector field $V$ become
\begin{align*}
    d \mathcal{S}(\phi) \cdot V =  \int_{U_{B}} j^1( \phi \circ \phi_{B}^{-1})^{\ast} \pounds_{j^1(V)} \Theta_{\mathcal{L}},
\end{align*}
and using the Cartan's formula definition of Lie derivative \cite{marsden1998multisymplectic}, we arrive at
\begin{align} \label{eq:var_action_multisymplectic}
    d \mathcal{S}(\phi) \cdot V =  -\int_{U_{B}} j^1( \phi \circ \phi_{B}^{-1})^{\ast} \left[ j^1(V) \contract \Omega_{\mathcal{L}} \right] + \int_{\partial U_{B}}j^1( \phi \circ \phi_{B}^{-1})^{\ast} \left[ j^1(V) \contract \Theta_{\mathcal{L}} \right]
\end{align}
Setting the $d \mathcal{S}(\phi) \cdot V = 0$, the first integral is zero when $\phi$ satisfies the Euler-Lagrange equation
\begin{equation}\label{eq:EL_multisymplectic}
\begin{aligned}
    &j^1( \phi \circ \phi_{B}^{-1})^{\ast} \left[ j^1(V) \contract \Omega_{\mathcal{L}} \right] = 0, \\
    \text{or} \quad &\frac{\partial L}{\partial y^A}(j^1( \phi \circ \phi_{B}^{-1}) - \frac{\partial}{\partial x^{\mu}} \left( \frac{\partial L}{\partial v_{\mu}^A}(j^1( \phi \circ \phi_{B}^{-1}))\right) = 0.
\end{aligned}
\end{equation}
When $\phi$ satisfies \eqref{eq:EL_multisymplectic} and $V$ and $W$ are two vector fields generated by the one-parameter transformation group of $\phi$ and the boundary integral in \eqref{eq:var_action_multisymplectic} gives rise to the following boundary integral
\begin{align}\label{eq:multisymplectic_form_formula}
    \int_{\partial U_{B}}j^1( \phi \circ \phi_{B}^{-1})^{\ast} \left[  j^1(W) \contract j^1(V) \contract \Omega_{\mathcal{L}} \right] = 0,
\end{align}
and it is known as \emph{the multisymplectic form formula}. The derivation of this formula can be found in \cite{marsden1998multisymplectic}. The importance of it is that it shows that the solution of the Euler-Lagrange equation \eqref{eq:EL_multisymplectic} preserves the multisymplectic form \eqref{eq:multisymplectic_form}  \cite{marsden1998multisymplectic}.  

\subsection{Review of optimal control} \label{sec:opt_cont}
As optimal control is central to this work, the basic definitions are stated here in order to familiarise the reader with the notation that will be used throughout this paper. We follow the notation used in \cite{FB/ADL:04supp}. However the topic is not covered here in great detail and the reader is advised to consult \cite{sontag2013mathematical, agrachev2004control,jurdjevic2016optimal} and references therein for detailed exposition.
\\

Given the quadruple $\Sigma = (Q, f, \Ucal)$ where $Q$ denotes the configuration manifold, $f : Q \times \Ucal \to TQ $ denotes the vector field that describes the motion of the system, and $\Ucal$ denotes the space of admissible controls.  We say the quadruple $\Sigma = (Q, f, \Ucal)$ is a \emph{controlled stochastic system}. The stochastic differential equation that governs the controlled system  $\Sigma$ is  the following 
\begin{equation} \label{eq:controll_SDE}
d_tq =f(q, \theta) \, dt + \nu \sum_k \xi_k(q) \circ dW_t^k,
\end{equation}
where $q \in Q$ and $\xi_k$ are vector fields that couple the independent and identically distributed Brownian motions $W^k(t)$ to the tangent bundle $TQ$.  Before giving a more rigorous definition of optimal control, we first need to define the main quantities of interest: the cost Lagrangian and cost functional. For the stochastic control system $\Sigma = (Q, f, \Ucal)$, we say a continuous, convex  map  $\ell: \Ucal \times Q \mapsto \mathbb{R}$, integrable in the time interval $[t_0, t_1]$, is the \emph{cost Lagrangian}. In some sources, it is referred to as the \emph{cost function} and that is to avoid confusion with other Lagrangians. While the Lagrangian, $\ell$, gives the cost of the state and control at time $t$, but for the cost in the time interval $[t_0, t_1]$, a functional is used instead.  \emph{The cost functional} $\mathcal{S}_{\Sigma}: \Ucal \times Q  \mapsto \mathbb{R}$ is used as a performance metric and it is defined as
$$
\mathcal{S}_{\Sigma} = \int_{t_0}^{t_1} \ell(q(t), \theta(t)) dt.
$$ 
With that, the definition of a optimal control problem can be stated as follows:
\begin{definition}[Optimal control problem]
Given a stochastic control system $\Sigma = (Q, f, \Ucal)$, let $Q_0$ and $Q_T$ be two distinct submanifolds of $Q$, then find the pair $(q^{\ast}, \theta^{\ast})$, where $q^{\ast} \in Q$ and $\theta^{\ast} \in \Ucal$, such that the curve $q^{\ast}(0) \in Q_0$ and $q^{\ast}(T) \in Q_T$ and in the same time satisfying $\mathcal{S}_{\Sigma}(q^{\ast}, \theta^{\ast}) < \mathcal{S}_{\Sigma}(q, \theta)$ for every $q \in Q$ and $\theta \in \Ucal$. 
\end{definition}

Due to the fact that it is indispensable, we only sketch Pontryagin's Maximum Principle.  As the proof is technical and out of the scope of this paper, nevertheless, it can be found in \cite{doi:10.1002/zamm.19630431023,agrachev2004control,jurdjevic2016optimal,nla.cat-vn542048} and the supplement chapter \cite{FB/ADL:04supp} of \cite{FB-ADL:04}. One of the nuclei of Pontryagin's Maximum Principle is the Hamiltonian function for  $\Sigma$. For the control system $\Sigma = (Q, f, \Ucal)$ and the cost Lagrangian $\ell$, the \emph{Pontryagin's Hamiltonian function}, also known as \emph{control Hamiltonian}, $H_{\Sigma} : T^{\ast}Q \times \Ucal  \mapsto \mathbb{R}$, associated to the control system $\Sigma$ is defined as
$$
\begin{aligned}
H_{\Sigma}(q, \lambda, \theta) = \left\langle \lambda, f(q, \theta) \right\rangle_{ T^{\ast}Q \times  TQ } - \ell(q, \theta).
\end{aligned}
$$
Once the substituting the expression for optimal $\theta$ in terms of $(q, \lambda)$, we arrive at \emph{maximum Hamiltonian}, $H^{\ast}_{\Sigma}  :  T^{\ast}Q \mapsto \mathbb{R}$, and it is defined as
\begin{equation}\label{eq:Ham_ast}
\begin{aligned} 
H^{\ast}_{\Sigma}(q, \lambda) = \sup \left\{ H_{\Sigma}(q, \lambda, \theta) \quad | \quad \theta \in \Ucal \right\}.
\end{aligned}
\end{equation}
\begin{theorem}[Pontryagin's Maximum Principle \cite{doi:10.1002/zamm.19630431023}]\label{thm:pmp}
Let $\Sigma = (Q, f, \Ucal)$ be the control system, and let $Q_0$ and $Q_T$ be the disjoint sub-manifolds of $Q$. If the pair $(q, \theta)$ solves the optimal control problem, i.e. minimise the cost functional
\begin{equation}
\mathcal{S} = \int_0^T \ell(q(t), \theta(t)) dt
\end{equation}
subject to $\dot{q}(t) = f(q(t), \theta(t))$, with $q(0) = q_0 \in Q_0$ and $q(T) = q_T \in Q_T$, then there exists $\lambda_0 = \{0, 1 \}$ and a cotangent vector 
$$
\lambda:[0,T] \mapsto T^{\ast}Q
$$
along $q$ and the Hamiltonian $H^{\ast}_{\Sigma}( \lambda(t)) = \sup\limits_{\theta \in \mathcal{U}} H_{\Sigma}( \lambda(t))$ and $\lambda(t)$ evolves in time according to
$$
\dot{\lambda} = X_{H^{\ast}_{\Sigma}}(\lambda(t))
$$
In other words, $\lambda$ is an integral curve of a Hamiltonian vector field associated with $H^{\ast}_{\Sigma}$. Along with that, it also satisfies the following transversality condition:  $\lambda(0) \in \mathrm{ann}(T_{q(t_0)}q_0)$ and $\lambda(1) \in \mathrm{ann}(T_{q(t_1)}q_T)$ and we either have $\lambda_0 = 1$ or $\lambda(0) \neq 0$ and finally there exists a constant $C \in \mathbb{R}_{+}$ such that the maximum Hamiltonian is constant, i.e. ${H^{\ast}_{\Sigma}}(\lambda(t)) = C$.
\end{theorem}
For a specific setting, when the manifold $Q$ is acted by a Lie group action, the optimality condition yields a momentum map and Hamilton's equations of the Pontryagin's Hamiltonian system, generated by \eqref{eq:Ham_ast}, can be reduced to the Euler--Poincar\'{e} equation. As a result the initial condition $q_0$ and $q_T$ both lie to the same $G$-orbit. This observation was pointed out in \cite{gay2011clebsch, 10.1007/978-3-030-01593-0_5} and they formulated optimal control in terms of Clebsch variables. 
\\
\begin{definition}[Clebsch optimal control \cite{gay2011clebsch}] Let $\Sigma = (Q, f, \Ucal)$ be the control system, then the Clebsch variational principle for optimal control is 
\begin{align*}
\delta \mathcal{S} = \delta \int_0^T \ell(q(t), \theta(t)) \, dt + \int_0^T\int_{Q} \lambda \cdot \left(\dot{q}(t) - f(q(t), \theta(t)) \right) \, d^nx \, dt = 0,
\end{align*}
where $\lambda$ is the Lagrange multiplier that acts as a constraint to the probability density function $\rho$ to evolve according to 
\begin{align*}
\dot{q}(t) = f(q(t), \theta(t)) ,
\end{align*}
with $q(0) = q_0 \in Q_0$ and $q(T) = q_T \in Q_T$, where $Q_0, Q_T \subset Q$. 
\end{definition}

\subsection{Mean--field type optimal control}\label{sec:mfg}
Here, we briefly overview the mean--field type control of \cite{bensoussan_mean_2013,bensoussan_general_2013}. Consider a population of $N$ identical agents on $Q$ and each agent has control over the stochastic differential equation
\begin{equation}\label{eg:mfg_sde}
dX_t^i = f(X_t^i, \theta(t)) dt + \nu \sum_k \xi_k(X_t^i)  dW^k_t, \quad X^i_0 = x^i \in Q \subset \mathbb{R}^d,
\end{equation}
where $X_t^i$ are Markov processes and the drift is denoted by $\theta(x,t)$, and $W^i(t)$ are independent and identically distributed Brownian motions. Here, $\xi_k$ are diffusion vector fields and they are used to couple $W^i(t)$ to the tangent bundle.  For mean--field type control, the goal is to find $\theta = (u(x,t), b(x,t))$ such that it steers the whole population from an initial $X^i_0 = x_0$ to $X^i_T = x_T$. The solution of mean--field games, as the size of the population of agents goes to infinity, $N \to \infty$, converges to Nash equilibrium \cite{ lasry_jeux_2006-1}. 
\begin{remark}
In the previous sections, curves on $Q$ were denoted by $q(t)$. Here, $X_t$ is used for stochastic processes on $Q$ and to make the clear distinction between them and deterministic curves. 
\end{remark}
\begin{remark}
It is important to point out that when $Q$ is not $\mathbb{R}^d$, the stochastic differential equation \eqref{eg:mfg_sde} is replaced by a Stratonovich stochastic differential equation. 
When dynamics are manifolds, It\^{o} stochastic differential equations require extra construction to define them on smooth manifolds, hence it is practical to use Stratonovich stochastic differential equation \cite{LAZAROCAMI200865, holm2015variational}. In this paper, we consider $Q \subset \mathbb{R}^d$, in order to use It\^{o} stochastic differential equation \eqref{eg:mfg_sde}. 
\end{remark}
The idea of \cite{lasry_jeux_2006} and \cite{huang_large_2006}, instead of considering individual agents, their states are substituted by a probability density function. Thus, \eqref{eg:mfg_sde} is replaced by the Fokker--Planck equation, which is a partial differential equation that describes the time evolution of the probability density function associated a Markov process $X_t$, and it is given by
\begin{align}\label{eq:FP_eqn_1}
\frac{\partial }{\partial t} \rho( x, t  )  \, d^nx + \Lie{f}{( \rho( x, t  )\, d^nx )}  -  \frac{\nu^2}{2} \sum_{i,j} \Lie{\xi_j}{\Lie{\xi_i} { ( \rho( x, t  )\, d^nx )}}  = 0.
\end{align}
The term $\pounds_{f}$ is the Lie derivative along the vector field $f(x, u)$ and it is defined by
$$
\pounds_{f}{(\rho \, d^nx)}  = \left. \frac{d}{ds}  \right|_{s = 0} (\psi(s))^{\ast} (\rho \, d^nx) ,
$$
where $\psi^{\ast} $ is the pullback along the diffeomorphism $\psi: I \times \mathrm{Den}( Q) \to \mathrm{Den}( Q)$ and it is generated by the vector flow of $f(x,u)$. Since $\rho$ is a density, the Lie derivative of along $f$ in general coordinates, is written as
\begin{align*}
\pounds_{f}{(\rho \, d^nx)} = \frac{1}{\sqrt{|\mathbb{G}|}}\partial_i \left(\sqrt{|\mathbb{G}|} \rho f^i \right) \,   d^nx ,
\end{align*}
where $\mathbb{G}$ is the space's metric tensor.  In coordinates on $\mathbb{R}^N$ the Lie derivative is
$$
\pounds_{f}{(\rho \, d^nx)}  = \mathrm{div}\left( \rho f \right) \,  d^nx .
$$ 
\begin{remark}
The Fokker--Planck equation is usually written in terms of coordinates on $\mathbb{R}^N$, but in \eqref{eq:FP_eqn_1}, it is written in terms of Lie derivatives to stress that fact that we are using a coordinate free description. It is trivial to arrive at such description. Tout court, one computes the integral
\begin{align*} \label{eq:Int_time}
\int_Q f(x) \frac{\partial \prob( x, t  | x_0, t_0) }{\partial t} d^n x,
\end{align*}
where $f$ is a scalar function $f: Q \to \mathbb{R}$, by using the definition of a time derivative as the limit of a difference between $\rho$ at time $t + \delta t $ and $t$ and take the limit as $\delta t \to 0$. The density at $t + \delta t $ obeys the Chapman-Kolmogorov equation
\begin{equation}\label{eq:Chap_Kol}
\prob( x, t + \delta t | x_0, t_0) = \int_Q  \prob( x, t + \delta t | y, t)\prob( y, t  | x_0, t_0) d^ny
\end{equation}
and then using writes It\^{o}'s lemma \cite{10.3792/pia/1195572786,oksendal2013stochastic} for a scalar function $g: Q \to \mathbb{R}$ in terms of Lie derivatives 
\begin{align*}
g(x) = g(y) +  \pounds_{dX^i} g(y)  +  \frac{1}{2} \pounds_{dX^i} \pounds_{dX^i} g(y),
\end{align*}
we then arrive at the Fokker--Planck equation \eqref{eq:FP_eqn_1}.
\end{remark}

Going back to the stochastic optimal control problem, generally, the cost Lagrangian is defined as an expected value
\begin{equation}
\ell ( \theta, \prob) = \int_Q  \prob(x, t ) \widehat{\ell}(\theta) \, d^nx .
\end{equation}
In this note, we restrict ourselves to the case when the $\widehat{\ell}(\theta)$ is a quadratic function, and this case $\ell: \Ucal \times \mathbb{R} \mapsto \mathbb{R}$ becomes
\begin{equation}\label{eq:cost_lagrangian}
\ell ( \theta, \prob) = \frac{1}{2}\int_Q  \prob(x, t ) \left\| \theta  \right\|^2 \, d^nx 
\end{equation}
We now formulate a method for solving stochastic geometric optimal control problems. Here, the Fokker--Planck equation \eqref{eq:FP_eqn_1} is a constraint, in lieu of \eqref{eq:controll_SDE}. In this setting, the problem is no longer on the state manifold $Q$, instead it is on the space of probability densities of  $Q$, denoted by $\mathrm{Den}( Q)$. The space of probability density functions on $Q$ is defined as
\begin{align}
\mathrm{Den}( Q) = \left\{ \rho \in H^1( Q) \quad | \quad \rho > 0, \int_{ Q} \rho \, d^nx  = 1 \right\}.
\end{align}
It is worth noting that densities are usually defined as volume forms, but since the probability densities are related to the diffusion processes we treat them as functions with a reference volume form. Whichever $Q$ being finite or infinite-dimensional,  $\mathrm{Den}( Q)$ is considered as an infinite-dimensional manifold for all cases \cite{Khesin6165}. 
\\

The space of admissible controls $\mathcal{U}$ remains the same, however the input data of the training and testing sets are now elements of $\mathrm{Den}( Q)$ \cite{Khesin6165}.  It is worth noting that the setting used here is the same for computational fluid mechanics formulation of optimal transport \cite{benamou2000computational}. However, in  \cite{benamou2000computational} the equation for the probability density function is continuity equation, meaning there are no double Lie derivatives as in \eqref{eq:FP_eqn_1} and control vector field $\theta$ is a potential vector field. Then the control problem for a stochastic system can be stated as the following quadruple $\Sigma = (\mathrm{Den}( Q), f, \Ucal)$ where $\mathrm{Den}( Q)$ denotes the configuration manifold, $f(x,\theta): Q \times \mathcal{U} \to  T( Q)$ denotes the vector field that describes the motion of the system, $\Ucal$ denotes the space of admissible controls. The dynamics of the control system are governed by the Fokker--Planck equation 
\begin{align}
\partial_t \rho \, d^nx + \pounds_{f}{(\rho  \, d^nx )}  - \frac{\nu^2}{2} \sum_{i,j}\pounds_{\xi_i} \pounds_{\xi_j} (\rho  \, d^nx ) = 0,
\end{align}
where $\rho(x,t) \, d^nx  \in \mathrm{Den}( Q)$ is the probability density function. 

\begin{definition}[Mean--field type optimal control problem \cite{huang_large_2006}]\label{def:mfg}
Given a stochastic control system $\Sigma = (\mathrm{Den}( Q), f, \Ucal)$, let $Q_0$ and $Q_T$ be two distinct submanifolds of $Q$, then find the pair $(\prob^{\ast}, \theta^{\ast})$, where $\prob^{\ast} \in \mathrm{Den}( Q)$ and $\theta^{\ast} \in \Ucal$, such that the curve $\prob^{\ast}(x,0) \in \prob_0$ and $\prob^{\ast}(x,T) \in \prob_T$ and in the same time satisfying $\mathcal{S}_{\Sigma}(\prob^{\ast}, \theta^{\ast}) < \mathcal{S}_{\Sigma}(\prob, \theta)$ for every $ \prob \in Q$ and $\theta \in \Ucal$. Here, $\mathcal{S}_{\Sigma}$ is the cost functional of ${\Sigma}$ and it is defined by
\begin{equation}\label{prb:opt_cont_gen}
\mathcal{S}_{\Sigma} = \int_0^T \ell( \prob(x,t), \theta(x,t)) dt + \Phi(\rho(x, T), \rho_T)
\end{equation}
where $\ell: \Ucal \times \mathbb{R} \mapsto \mathbb{R}$ is the cost Lagrangian and $\Phi: \mathrm{Dens}(Q) \times \mathrm{Dens}(Q) \to \mathbb{R}$ is the cost for the final state.  
\end{definition}
To solve the optimal control problem Pontryagin's Maximum Principle is used and the cost functional \eqref{prb:opt_cont_gen} is now written in terms of Hamilton's principle 
\begin{equation} \label{eq:hamiltons_principle}
\mathcal{S}_{\Sigma} = \int_{t_0}^{t_1} \left(   \left\langle \lambda, \dot{\prob} \right\rangle  - H_{\Sigma} ( \lambda, \theta) \right)dt
\end{equation}
According to the Maximum/Minimum Principle, we first need to define a control Hamiltonian, $H_{\Sigma}: \Ucal  \times T^{\ast}\mathrm{Den}( Q) \mapsto \mathbb{R}$, and in this particular case it is
\begin{equation}\label{eq:control_Ham}
H_{\Sigma} ( \lambda , \theta)  =  \left\langle  \lambda , \Lie{f(x,\theta)}{\prob(x, t )} - \frac{\nu^2}{2} \Lie{\xi}{\Lie{\xi}{\prob(x, t )}}\right\rangle -  \ell ( \theta, \prob).
\end{equation}
\begin{remark}
Taking things one step further, we substitute the Legendre transformation \eqref{eq:control_Ham} into the action $\mathcal{S}_{\Sigma} $ to obtain 
\begin{equation}
\mathcal{S}_{\Sigma}  = \int_{t_0}^{t_1} \left(  \ell ( \theta, \prob) + \left\langle  \lambda , \dot{\prob} + \Lie{f(x,\theta)}{\rho} - \frac{\nu^2}{2} \Lie{\xi}{\Lie{\xi} \rho} \right\rangle  \right)dt
\end{equation}
When the control vector field $\theta$ belongs to Lie algebra and the underlying the dynamical system is described by a Lie group of volume preserving diffeomorphisms action on elements of $Q$ and the noise terms are absent, then this becomes the fluid mechanics formulation of Monge--Kantorovich optimal transport problem of \cite{benamou2000computational}. Our interest in the Lagrangian version with Lagrange multipliers is due to the fact that it allows us to derive the multisymplectic structure of the problem as we shall see in section \ref{sec:mul_geo_dl} as it was introduced in \cite{doi:10.1098/rspa.2007.1892}.
\end{remark}
If the following trajectory $(\prob, \theta)$ solves the optimal control problem of the definition \ref{def:mfg}, then according to Pontryagin's Maximum Principle there exists a dual vector field $\chi :[t_0, t_1] \mapsto T^{\ast}\mathrm{Den}( Q)$ along $\prob$ such that
$$
\frac{d}{dt} \chi = X_{H^{\ast}_{\Sigma}},
$$
where $ X_{H^{\ast}_{\Sigma}}$ is the Hamiltonian vector field associated with $H^{\ast}_{\Sigma} $ and 
\begin{equation}
H^{\ast}_{\Sigma} ( \lambda) = \sup \left\{ H_{\Sigma} ( \lambda , \theta) \quad |  \quad \theta \in \Ucal  \right\}.
\end{equation}
In other words, $\chi $ is the integral curve of the Hamiltonian vector field $ X_{H^{\ast}_{\Sigma}}$. Without getting much into technical details, the maximum Hamiltonian, $H^{\ast}_{\Sigma}$, is locally integrable smooth and  the covector field $\chi$ is locally absolutely continuous. That said, in local coordinates $(\prob, \lambda)$, $ X_{H^{\ast}_{\Sigma}}$ are governed by the following set of equations
\begin{equation}\label{eq:mfg_system}
\begin{aligned}
\partial_t {\prob} \, d^nx   &=  \frac{\delta }{\delta \lambda} H^{\ast}_{\Sigma}= -\Lie{f(x,\theta)}{(\prob  \, d^nx)}    + \frac{\nu^2}{2} \Lie{\xi}{\Lie{\xi}{(\prob  \, d^nx)}},   \\
\partial_t {\lambda} &=  -\frac{\delta}{\delta \prob} H^{\ast}_{\Sigma}= \Lie{f(x,\theta)}{\lambda} - \frac{\nu^2}{2} \Lie{\xi}{\Lie{\xi}{\lambda}} + \frac{\delta \ell(\theta, \rho)}{\delta \rho}, \\
\rho(x,0) & = \rho_0(x), \\
\lambda(x, T) &= \frac{\delta \mathcal{S}_{\Sigma} }{\delta \rho_T(x)},
\end{aligned}
\end{equation}
where $\frac{\delta}{\delta \lambda}$ and $\frac{\delta}{\delta \rho}$ denote the functional derivative with respect to $\lambda$ and $\rho$ respectively. Such derivative is defined by 
$$
\frac{\delta H}{\delta \alpha} = \lim\limits_{\epsilon \to 0} \frac{H^{\ast}_{\Sigma}(\alpha + \epsilon \delta \alpha) - H^{\ast}_{\Sigma}(\alpha)}{\epsilon}.
$$
This system of equations is known as \emph{mean--field games} \cite{lasry_jeux_2006, lasry_jeux_2006-1,huang_large_2006}. As mentioned before, the first equation is the Fokker--Planck equation, and the adjoint equation is the stochastic Hamilton--Jacobi--Bellman equation \cite{1556-1801_2012_2_337}. The Lie derivative in this case is for a scalar quantity and in general coordinates it is
\begin{align*}
    \pounds_{f(x,\theta)} \lambda = f(x,\theta) \cdot \nabla \lambda.
\end{align*}
The stochastic Hamilton--Jacobi--Bellman equation is a nonlinear backward equation and it appears as the continuous limit of dynamic programming solution of stochastic optimal control problem with cost functional
\begin{align*}
    \mathcal{S} &= \int_0^t \ell(\theta) \, dt + \Phi(X_t, x_T),
\end{align*}
subject to
\begin{align*}
d_t X_t = f(X_t,\theta) \, dt + \nu \sum\limits_i \xi_i(X_t)dW^i_t, \quad& \,  X_0 = x_0, \quad \text{and} \quad X_T = x_T.
\end{align*}
The variational principle form of the mean--field games was studied in a number of publications, mainly \cite{1556-1801_2012_2_337,MESZAROS20151135,Benamou2017}. Further, such minimisation is connected to the Schr\"{o}dinger bridge problem, where it seeks to find the flow of probability density that connects an initial density and a final density distribution while minimising the Kullback-Leibler divergence. In recent years, Schr\"{o}dinger bridge was regarded as a way to regularise the optimal transport problem \cite{doi:10.1137/050631264,1078-0947_2014_4_1533, leonard2010schr,chen2016relation,doi:10.1137/16M1061382,pavon2018data}. 
\\

The system of coupled equations is difficult to solve analytically, and that leaves us with a numerical approach. Moreover, an additional challenge arises because of the backward and nonlinear nature of the stochastic Hamilton–-Jacobi-–Bellman equation. Here, the adjoint equation is not discussed in detail, as later in \ref{sec:mul_geo_dl} we demonstrate how it can be eliminated from the system of equations. One numerical solution method was introduced in \cite{achdou_mean_2010}. In section \ref{sec:var_int}, we introduce a structure-preserving numerical scheme for mean--field games. The core element of the scheme is that it relies on variational principle for its derivation. 
\\

\section{Multisymplectic geometry for deep learning} \label{sec:mul_geo_dl}
Extending the philosophy of \cite{weinan2019mean, li2017maximum}, uncertainties in the form of Wiener processes are introduced to the dynamical model of deep neural networks and rendering the training problem as a stochastic optimal control problem. It is worth pointing out that there are numerous neural network architectures that can be used for deep learning, but the focus will solely be on residual neural networks. As in \cite{weinan2017proposal}, the residual neural networks are formulated as a dynamical system, but to consider a more general and robust case, the residual network dynamics are treated as a stochastic system
\begin{align}\label{eq:resNet_eq}
d_t X_t = f(X_t, \theta) \, dt + \nu \sum\limits_{k} \xi_k(X_t) dW_t^k, \quad X_0 = x_0,
\end{align}
where the initial condition $X_0 \in Q$ is the input to the network and the final condition, $X_T$, is the output. Here, we choose the parameters $\theta = u(x,t) \in \mathfrak{X}(Q)$, i.e. the network's weight and it is a vector field that belong to $\mathfrak{X}(Q)$, which  the space of vector fields on $Q$. 
This means $\mathfrak{X}(Q)$ is the space of admissible controls $\mathcal{U}$. The flow generated by vector fields on $\mathfrak{X}(Q)$ are elements of the group of diffeomorphisms on $Q$, $\mathrm{Diff}(Q)$. The function $f:  Q \times \mathfrak{X}(Q) \to  TQ$ is the activation function. 
\begin{remark}
In the language of optimal control, here the network's weights and biases, $\theta \in \mathfrak{X}(Q)$, play the role of the control, often denoted in control theory literature by $u$. This is used in order to distinguish the network's parameters, which are coefficients, from  the usual optimal control $u$, which is a vector field. 
\end{remark}
The problem of training, supervised one, is concerned with finding a suitable $\theta$ such that output $X_T = x_T$ would be as close as possible to the chosen target datum $c \in \mathcal{C}$, where $\mathcal{C}$ is the space of target data, with respect to a certain metric. Thus, as introduced in \cite{li2017maximum}, the supervised training can be formulated as an optimal control problem, with the dataset $ \{(x_0^{(i)}, c^{(i)}) \}_{i = 0, \dots, N}$ specifies the initial conditions and the final conditions. That means one needs to find $\theta$ that would minimise $\sum\limits_{i=0}^N \Phi( \pi(X^{(i)}_T), c^{(i)} )$, where $\Phi: \mathcal{C} \times \mathcal{C} \to \mathbb{R}$ is a loss function and $\pi : Q \to \mathcal{C}$ acts as projection from the state manifold $Q$ to the space of target data. The supervised learning problem can then be stated as:
\begin{problem}\label{prob:mfg_deep_0}
Let $(Q, \mathcal{F}, \mathbb{P})$ be a probability space and the stochastic control system $\Sigma = (Q, f, \mathfrak{X}(Q) )$. Let $x_0^{(i)} \in Q$ and $c^{(i)} \in \mathcal{C}$ for $i=0, \dots, N$. The cost functional is chosen as
\begin{align}
\mathcal{S}_{\Sigma} = \sum_{i=0}^N \Phi( \pi(X^{(i)}_T),c^{(i)} ) + \int_0^t \int_{Q} \widehat{\ell}(\theta)\, d\rho \, dt,
\end{align}
subject to the following constraint
\begin{align}
d_t X_t^{(i)} = f(X_t^{(i)}, \theta) \,dt + \nu \sum_k \xi_k(X_t^{(i)}) dW_t^k, \quad X^{(i)}_0 = x^{(i)}_0,
\end{align}
where $\widehat{\ell}: \mathfrak{X}(Q) \to \mathbb{R}$ is the cost Lagrangian and $\xi_k \in TQ$ are vector fields used to couple the independent and identically distributed Brownian motions $W^k_t$ to the tangent bundle $TQ$. The training of a deep residual network is done by finding the controlled trajectories $(x^{(i), \ast}, \theta^{\ast})$ with $X^{(i), \ast}_0 = x^{(i)}_0 \in Q$ and $\theta \in \mathfrak{X}(Q)$ such that 
$$
\mathcal{S}_{\Sigma}( \theta^{\ast}) < \mathcal{S}_{\Sigma}( \theta), \quad \forall ( \theta) \in  \mathfrak{X}(Q).
$$
\end{problem}
This stochastic optimal control problem can be formulated in terms of mean--field games as in \cite{weinan2019mean}. Although, as mentioned earlier, in \cite{weinan2019mean}, the McKean--Vlasov approach to mean--field games was used.  Here, we use the Fokker--Planck and Hamilton--Jacobi--Bellman approach, which seeks to control the probability density of the stochastic dynamical system directly. With that, problem \ref{prob:mfg_deep_0} is equivalent to
\begin{problem}\label{prob:mfg_deep_2}
Let $(Q, \mathcal{F}, \mathbb{P})$ be a probability space and the stochastic control system $\Sigma_{\rho} = (\mathrm{Dens}(Q), f , \mathfrak{X}(Q))$. Let $\rho^i_0 \in \mathrm{Dens}(Q)$ and $c^{(i)} \in \mathcal{C}$ for $i=0, \dots, N$. The cost functional is chosen as
\begin{align}\label{eq:control_action_section3}
\mathcal{S}_{\Sigma_{\rho}} = \sum\limits_{i=0}^N \int_{Q} \Phi( \pi(\rho(x,T)), c^{(i)} )d\rho + \int_0^t \ell(\rho, \theta) \, dt, \quad \text{with} \quad \ell(\rho, \theta) =  \int_{Q} \widehat{\ell}(\theta)\, d\rho,
\end{align}
subject to the following constraint
\begin{align}
\partial_t \rho   \, d^nx   + \pounds_{f(x,\theta)} \rho   \, d^nx   - \frac{\nu^2}{2} \sum_{i,j}\pounds_{\xi_i} \pounds_{\xi_j} \rho  \, d^nx  , \quad \rho(x,0) = \rho_0(x),
\end{align}
where $\widehat{\ell}: \mathfrak{X}(Q) \to \mathbb{R}$ is the cost Lagrangian. Then, the training of a deep residual network is done by finding the controlled trajectories $(\rho^{\ast}, \theta^{\ast})$ with $\rho^{ \ast}(x,0) = \rho_0 \in \mathrm{Dens}(Q)$ and $\theta \in \mathfrak{X}(Q)$ such that 
$$
\mathcal{S}_{\Sigma_{\rho}}(\rho^{ \ast}, \theta^{\ast}) < \mathcal{S}_{\Sigma_{\rho}}(\rho, \theta), \quad \forall (\rho, \theta) \in \mathrm{Dens}(Q) \times \mathfrak{X}(Q).
$$
\end{problem}
To obtain the weights $\theta$ that would yield the trajectory $(\rho^{ \ast}, \theta^{\ast})$, we apply Pontryagin's Maximum Principle, with the  Hamiltonian defined by
\begin{align}
&H_{\Sigma_{\rho}} \quad : \quad  T^{\ast}\mathrm{Dens}(Q) \times \mathfrak{X}(Q)  \mapsto \mathbb{R}, \nonumber\\
&H_{\Sigma_{\rho}}(\rho, \lambda, v) = \ell( \rho, \theta)+\left\langle \lambda,  \pounds_{f(x,\theta)} \rho - \frac{\nu^2}{2} \sum_{i,j}\pounds_{\xi_i} \pounds_{\xi_j} \rho  \right\rangle . \label{eq:hamilonian_generic}
\end{align}
Here $\langle \cdot, \cdot \rangle$ is pairing and it is a mapping $\langle \cdot, \cdot \rangle :  T^{\ast}\mathrm{Dens}(Q)\times  T\mathrm{Dens}(Q)  \to \mathbb{R}$. However, we apply Legendre transformation \cite{marsden2013introduction} to switch to the Lagrangian picture as it is needed when constructing numerical schemes based on variational principle.  The Lagrangian is then 
\begin{align}
     \ell(\rho, \theta) + \left\langle \lambda, \partial_t \rho - \pounds_{f(x,\theta)} \rho - \frac{\nu^2}{2} \sum_{i,j}\pounds_{\xi_i} \pounds_{\xi_j} \rho  \right\rangle_{ T^{\ast}\mathrm{Dens}(Q)\times  T\mathrm{Dens}(Q) },
\end{align}
and it is known as Clebsch Lagrangian \cite{holm_poisson_1983}. The stationary variations of the action \eqref{eq:control_action_section3} with constraints
\begin{align*}
    \delta \mathcal{S}_{\Sigma_{\rho}} = \sum_{i=0}^N \delta\int_{Q} \Phi( \pi(\rho(x,T)), c^{(i)}) )d\rho + \delta \int_0^t \ell(\rho, \theta)  \, dt, \\
    &+ \delta \int_0^t \left\langle \lambda, \partial_t \rho - \pounds_{f(x,\theta)} \rho - \frac{\nu^2}{2} \sum_{i,j}\pounds_{\xi_i} \pounds_{\xi_j} \rho  \right\rangle \, dt,  
\end{align*}
as in \eqref{eq:var_action_multisymplectic} yields the following Euler-Lagrange equations:
\begin{equation}\label{eq:mfg_system_1}
\begin{aligned}
\delta \lambda &: \quad \partial_t \rho  \, d^nx    +  \pounds_{f(x,\theta)}(\rho   \, d^nx)   - \frac{\nu^2}{2} \sum_{i,j}\pounds_{\xi_i} \pounds_{\xi_j}( \rho  \, d^nx )  = 0, \\
\delta \rho  &: \quad \partial_t \lambda +  \pounds_{f(x,\theta)}\lambda +  \frac{\nu^2}{2} \sum_{i,j}\pounds_{\xi_i} \pounds_{\xi_j} \lambda + \frac{\delta \ell}{\delta \rho} = 0, \\
\delta \theta &: \quad \frac{\delta \ell}{\delta \theta} + \rho \nabla\lambda \, dx \otimes d^nx = 0,
\end{aligned}
\end{equation}
with $\quad \rho(x,0) = \rho_0$ and $\lambda(x,T) = \frac{\delta \Phi}{\delta \rho_T}$.  The last equation is the optimality condition. The variational derivative $\frac{\delta \ell}{\delta \theta}$ is in fact a one-form density with a differential form  $dx \otimes d^nx $. Substituting the optimality condition into the Hamiltonian $H_{\Sigma_{\rho}}$ \eqref{eq:hamilonian_generic}, gives us the following Hamiltonian $H_{\Sigma_{\rho}}^{\ast}: T^{\ast}\mathrm{Dens}(Q) \to \mathbb{R}$ which is now a function in terms of canonical coordinates of the cotangent bundle $T^{\ast}\mathrm{Dens}(Q)$ and it recovers the system of equations \eqref{eq:mfg_system}.
\begin{remark}
The nomenclature for the partial differential equation for the evolution Lagrange multiplier $\lambda$ might cause confusion. Despite it being called the stochastic Hamilton--Jacobi--Bellman, in fact the equation is deterministic. The reason for the naming is that it is the continuous limit of dynamic programming when applied to stochastic optimal control problems. 
\end{remark}
\begin{remark}\label{rem:monge_ampere}
The mean--field games system \eqref{eq:mfg_system_1} becomes fluid mechanics formulation of the the Monge--Kantorovich optimal transport problem when  $\xi_k \to 0$ \cite{benamou2000computational}.  What is interesting regarding this formulation is that it relates the solution of the optimal transport problem to the solution of Euler's incompressible fluid equation. 
\end{remark}
Solving \eqref{eq:mfg_system_1} yields the training weights $\theta$. Still, finding the solution is challenging, even for numerical methods and for this reason we cast it as a multisymplectic system with the aim of using multisymplectic integrators. Here we are following the same approach to multisymplectic systems as \cite{doi:10.1098/rspa.2007.1892}, where the stationary variations of the Clebsch variational principle for back-to-label inverse mapping yield a multisymplectic system.
\\

Before applying Pontryagin's Maximum Principle, first, a change of variable is required to carry out the computation and it is done by decomposing the weights $\theta: Q \to TQ$ as a sum of two vector fields $w : Q \to TQ$ and osmotic velocity term, $\frac{\nu^2}{2}\nabla \mathrm{log}(\rho(x,t))$ , i.e.
\begin{align*}
\theta(x,t)= w(x,t) + \frac{\nu^2}{2}\left(\frac{\partial f}{\partial \zeta}\right)^{-1}\nabla \mathrm{log}(\rho(x,t)).
\end{align*}
This decomposition was used in \cite{nelson1967dynamical} to show the relation between stochastic diffusion process and the Schr\"{o}dinger equation. It was further used in \cite{chen2016relation,doi:10.1137/16M1061382,pavon2018data} for Schr\"{o}dinger bridge problem. 
\begin{remark}
The term $\nabla \mathrm{log}(\rho(x,t)) = \frac{\nabla \rho(x, t)}{\rho(x,t)}$ contains the density function in the denominator. In this article, the main assumption is that the probability density consists of a single or a sum of Gaussian functions. Depending on the configuration manifold $Q$, the $\rho(x,t) > 0$ for all $x \in Q$.
\end{remark}
Here, the vector field that describes the residual neural network becomes 
\begin{align}
f(x, \theta) = f \left( w(x,t) +\frac{\nu^2}{2}\left( \frac{\partial f}{\partial \zeta} \right)^{-1} \nabla \mathrm{log}(\rho(x,t)) \right),
\end{align}
where $f$ is the activation function and in this paper, $f$ is a linear function. Substituting it into the Fokker--Planck equation we get
\begin{align}
\partial_t \rho  \, d^nx  + \pounds_{f(w)}( \rho  \, d^nx)    = 0, \quad \rho(x,0)  =\rho_0,
\end{align}
Then the optimal control problem becomes slightly different, as required.
\begin{problem}\label{prob:mfg_deep}
Let $(Q, \mathcal{F}, \mathbb{P})$ be a probability space and the stochastic control system $\Sigma_{\rho} = (\mathrm{Dens}(Q), \nabla , \mathfrak{X}(Q))$. Let $\rho_0 \in \mathrm{Dens}(Q)$ and $c^{(i)} \in \mathcal{C}$ for $i=0, \dots, N$. The cost functional is chosen as
\begin{align}
\mathcal{S}_{\Sigma_{\rho}} = \sum_{i=0}^N \int_{Q} \Phi( \pi(\rho(x,T)), c^{(i)} )d\rho + \int_0^t \ell \left( \rho, w + \frac{\nu^2}{2}\left( \frac{\partial f}{\partial \zeta} \right)^{-1} \nabla \mathrm{log}(\rho) \right)  \, dt,
\end{align}
subject to the following constraint
\begin{align}
\partial_t \rho  \, d^nx  + \pounds_{f(w)}(\rho  \, d^nx)  = 0, \quad \rho(x,0) = \rho_0(x),
\end{align}
where $\ell:  T\mathrm{Dens}(Q) \times \mathfrak{X}(Q)  \to \mathbb{R}$ is the cost Lagrangian defined by
\begin{align}
\ell \left( \rho, w +\frac{\nu^2}{2}\left( \frac{\partial f}{\partial \zeta} \right)^{-1}  \nabla \mathrm{log}( \rho) \right) =\int_{Q} \left\| w + \frac{\nu^2}{2}\left( \frac{\partial f}{\partial \zeta} \right)^{-1} \nabla \mathrm{log}(\rho)  \right\|^2 \, d\rho,
\end{align}
and $\ell: \mathfrak{X}(Q)  \to \mathbb{R}$ is a convex function. Finding the controlled trajectories $(\rho^{\ast}, \theta^{\ast})$ with $\rho^{ \ast}(x,0) = \rho_0 \in Q$ and $\theta \in \mathfrak{X}(Q)$ such that 
$$
\mathcal{S}_{\Sigma_{\rho}}(\rho^{ \ast}, \theta^{\ast}) < \mathcal{S}_{\Sigma_{\rho}}(\rho, \theta), \quad \forall (\rho, \theta) \in \mathrm{Dens}(Q) \times \mathfrak{X}(Q),
$$
solves the training of a deep residual network.
\end{problem}
The evolution equation of the probability density is now the continuity equation, which is a first order partial differential equation, and this is essential to our approach as it allows us to reduce \eqref{eq:mfg_system_1} to Euler--Poincar\'{e} equation. The result now is that the Hamiltonian is affine in the derivatives with respect to independent variables, which allows us to obtain a multisymplectic formulation. With that the Pontryagin's Hamiltonian becomes
\begin{align*}
H_{\Sigma_{\rho}} \quad : \quad  T^{\ast}\mathrm{Dens}(Q) \times \mathfrak{X}(Q)  \mapsto \mathbb{R},
\end{align*}
\begin{align*}
H_{\Sigma_{\rho}}(\rho, \lambda, v) = \ell \left( \rho, w +  \frac{\nu^2}{2}\left( \frac{\partial f}{\partial \zeta} \right)^{-1} \nabla \mathrm{log}( \rho) \right) + \left\langle \lambda,  \pounds_{f(w)} \rho \right\rangle_{ T^{\ast}\mathrm{Dens}(Q)\times  T\mathrm{Dens}(Q) },
\end{align*}
and applying the Legndre transformation, we map the Hamiltonian to a Lagrangian function, allowing us to use Clebsch variational principle
\begin{equation}
\begin{aligned}
    0 = \delta \mathcal{S}_{\Sigma_{\rho}} &=  \delta \int_0^t \ell \left( \rho, w +  \frac{\nu^2}{2}\left( \frac{\partial f}{\partial \zeta} \right)^{-1}\nabla \mathrm{log}( \rho) \right)  + \left\langle \lambda, \partial_t \rho  + \pounds_{f(w)} \rho \right\rangle\, dt \\
    &+\sum_{i=0}^N \delta\int_{Q} \Phi( \pi(\rho(x,T)), c^{(i)} )d\rho, 
\end{aligned}
\end{equation}
As in \eqref{eq:EL_multisymplectic}, the stationary variations yield the Euler-Lagrange system of partial differential equations
\begin{equation}\label{eq:mfg_system2}
\begin{aligned}
\delta \lambda &: \quad \partial_t \rho   \, d^nx   +  \pounds_{f(w)}( \rho  \, d^nx)    = 0, \\
\delta \rho  &: \quad \partial_t \lambda +  \pounds_{f(w)} \lambda  + \frac{\delta \ell}{\delta \rho} = 0, \\
\delta w &: \quad \frac{\delta \ell}{\delta w} + \rho \nabla\lambda \, dx \otimes  \, d^nx    = 0,
\end{aligned}
\end{equation}
with $\rho(x,0) = \rho_0$ and $ \lambda(x,T) = \frac{\delta \Phi}{\delta \rho_T}$ and such system of equations has a multisymplectic structure:
\begin{align}\label{eq:multisymplectic_struct}
\begin{bmatrix}
0 & 0 & 0 \\
0 & 0 & 1\\
0 & -1  & 0
\end{bmatrix} \begin{bmatrix}
\partial_t w^i \\
\partial_t \lambda \\
\partial_t \rho
\end{bmatrix} + \begin{bmatrix}
0 & -\rho\frac{\partial f}{\partial \zeta} & \frac{\nu^2}{2} \\
\rho\frac{\partial f}{\partial \zeta} & 0 &  f(w) \\
-\frac{\nu^2}{2} &  - f(w) & 0
\end{bmatrix} \begin{bmatrix}
\partial_{x^i}w^i \\
\partial_{x^i}\lambda \\
\partial_{x^i}\rho
\end{bmatrix} = \nabla_z S(z),
\end{align}
where $z = (w^i, \lambda, \rho)^T$ and the Hamiltonian is given by
\begin{align}\label{eq:HamiltonianDeep}
S(z) = \frac{\rho}{2} \left\| w \right\|^2  + \frac{ \nu^4 \rho}{8} \left\| \left(\frac{\partial f}{\partial \zeta}\right)^{-1} \nabla \mathrm{log} {\rho} \right\|^2.
\end{align}
The last term in $S(z)$ is the \emph{Fisher information metric}. The structure in \eqref{eq:multisymplectic_struct} is the sum of pre-symplectic structures, each pre-symplectic structure corresponds to an independent variable and they are all skew-symmetric matrices. In this case, the pre-symplectic structure corresponding to the independent variable $t$ and the pre-symplectic structure corresponding to the spatial independent variable $x^i$ are
\begin{align} \label{eq:pre_time}
M = \begin{bmatrix}
0 & 0 & 0 \\
0 & 0 & 1 \\
0 & -1 & 0
\end{bmatrix}, \quad \text{and} \quad K_i = \begin{bmatrix}
0 & -\rho\frac{\partial f}{\partial \zeta} & \frac{\nu^2}{2} \\
\rho\frac{\partial f}{\partial \zeta} & 0 & f(w) \\
-\frac{\nu^2}{2} & -f(w) & 0
\end{bmatrix},
\end{align}
respectively. Writing the system of partial differential equations \eqref{eq:mfg_system2} as \eqref{eq:multisymplectic_struct} using anti-symmetric matrices
\begin{align}\label{eq:multisymplectic_bridges}
M \partial_t z + \sum\limits_i K_i \partial_{x^i} z = \partial_z S(z),
\end{align}
is an alternative approach to multisymplectic partial differential equations \cite{bridges1997multi}.  In fact, according to  \cite{bridges1997multi}, a system is multisymplectic if it can be written in the form of \eqref{eq:multisymplectic_bridges}. In a geometric setting, the equations \eqref{eq:multisymplectic_struct} of can be written intrinsically, like in equation \eqref{eq:EL_multisymplectic}, as 
\begin{align*}
(j^1\phi)^{\ast}\left[ j^1(V) \contrac \Omega_{\mathcal{L}} \right] = 0,
\end{align*}
where $V \in J^{1}(E)$ and the Lagrangian density $\mathcal{L}: J^1(E) \to \Lambda^{n+1}(Q)$,  here is defined by
\begin{align*}
\mathcal{L}:= \ell_c \, d^nx  \wedge dt &=  \frac{\rho}{2} \left\| w \right\|^2 +  \frac{\nu^2}{2} \rho \left\langle w, \left(\frac{\partial f}{\partial \zeta}\right)^{-1} \nabla \mathrm{log} {\rho} \right\rangle  + \frac{ \nu^4 \rho}{8} \left\| \left(\frac{\partial f}{\partial \zeta}\right)^{-1} \nabla \mathrm{log} {\rho} \right\|^2 \\
& + \left\langle \lambda, \partial_t \rho + \pounds_{f(w)}\rho \right\rangle \, d^nx  \wedge dt.
\end{align*}
With the Cartan form $\Theta_{\mathcal{L}} \in \Lambda^{n+1}(J^1(E))$, and its exterior derivative $\Omega_{\mathcal{L}}  = - d\Theta_{\mathcal{L}} \in \Lambda^{n+2}(J^1(E))$, is \begin{equation}\label{eq:cartan_form_resNet}
\begin{aligned}
\Theta_{\mathcal{L}} &= \frac{\partial \ell_c}{\partial \rho_{x_i}} d\rho \wedge \left( \partial_{x_i} \contract d^nx \right) \wedge dt +  \frac{\partial \ell_c}{\partial \rho_t} d\rho \wedge d^nx +  \frac{\partial \ell_c}{\partial w^j_{x_i} } dw^j \wedge  \left( \partial_{x_i} \contract d^nx \right) \wedge dt.   \\
\end{aligned}
\end{equation}
Without the loss of generality, when the base manifold $Q$ is $\mathbb{R}^2 \times \mathbb{R}$ with coordinates $(x,y,t)$ and the fibres are $w = (w^1,w^2)$, the form $\Theta_{\mathcal{L}} $ is expressed explicitly by
\begin{align*}
\Theta_{\mathcal{L}} &=  \left( \frac{\nu^2}{2} w^1 +\frac{\nu^4}{4}\frac{\partial_x \rho}{\rho} + \lambda w^1 \right)d\rho \wedge dx^2 \wedge dt  +  \left( \frac{\nu^2}{2} w^2 +\frac{\nu^4}{4}\frac{\partial_y \rho}{\rho} + \lambda w^2 \right) d\rho \wedge dx^1 \wedge dt \\
&+  \lambda d\rho \wedge dx^1 \wedge dx^2 + \left( \lambda \rho \right) dw^1 \wedge dx^2 \wedge dt + \left( \lambda \rho \right) dw^2 \wedge dx^1 \wedge dt + \left( \ell - \frac{\partial \ell_c}{\partial \rho_x}\partial_x \rho \right. \\
&\left.- \frac{\partial \ell_c}{\partial \rho_y}\partial_y \rho - \frac{\partial \ell_c}{\partial \rho_t}\partial_t \rho  - \frac{\partial \ell_c}{\partial w^1_x}\partial_x w^1  - \frac{\partial \ell_c}{\partial w^2_x}\partial_x w^2 \right) dx^1 \wedge dx^2 \wedge dt. 
\end{align*}

The Hamilton--Jacobi--Bellman equation, whose solution is $\lambda$, in \eqref{eq:multisymplectic_struct}  can be eliminated and the result is an equivalent system that is governed by the Euler--Poincar\'{e} equation \cite{holm1998euler}.  For multisymplectic equations, the elimination was presented in \cite{doi:10.1098/rspa.2007.1892} and here, we follow the same steps to obtain the Euler--Poincar\'{e} equation for deep learning.
\begin{theorem}\label{thm:elemination}
Consider the optimal control problem \ref{prob:mfg_deep} for supervised learning of stochastic residual network, it's solution, denoted by the tuple $\left(w, \lambda, \rho\right)$, satisfies 
\begin{align*}
(j^1\phi)^{\ast}\left[ j^1(V) \contrac \Omega_{\mathcal{L}} \right] = 0,
\end{align*}
where $\Omega_{\mathcal{L}} =-d\Theta_{\mathcal{L}}$ of \eqref{eq:cartan_form_resNet}. Then, $\frac{\delta \ell}{\delta w} = \rho w + \frac{\nu^2}{2} \nabla \rho = \rho \nabla \lambda $ and $\rho$ satisfy the following Euler--Poincar\'{e} equation on $\mathfrak{X}(Q) \times \mathrm{Dens}(Q)$
\begin{equation}\label{eq:ep_cont_eqn}
\begin{aligned}
&\left[ \frac{\partial}{\partial t}\frac{\delta \ell}{\delta w} + \mathrm{ad}^{\ast}_{f(w)} \frac{\delta \ell}{\delta w} - \rho \diamond \frac{\delta \ell}{\delta \rho} \right]_i  dx^i \otimes d^nx = 0,\\
&\partial_t  \rho   \, d^nx + \pounds_{f(w)}(\rho   \, d^nx ) = 0.
\end{aligned}
\end{equation}
\end{theorem}
\begin{proof}
This elimination is done by taking the time derivative of the inner product of $ \frac{\delta \ell}{\delta w}$ and an arbitrary element $\eta \in \mathfrak{X}(Q)$. The inner product is defined by
\begin{align}\label{eq:inner_product_pair}
\left\langle \frac{\delta \ell}{\delta w}, \eta \right\rangle_{\mathfrak{X}^{\ast}(Q) \times \mathfrak{X}(Q)} = \int_{Q} \frac{\delta \ell}{\delta w} \cdot \eta \, \, d^nx .
\end{align}
With the inner product defined and using the equations for $\partial_t \rho$ and $\partial_t \lambda$, the time derivative of $\frac{\delta \ell}{\delta w}$  is 
\begin{align*}
\frac{d}{dt} \left\langle \frac{\delta \ell}{\delta w}, \eta \right\rangle &= \frac{d}{dt} \left\langle \rho \nabla \lambda, \eta \right\rangle  \\
&=  -\left\langle \pounds_{f(w)} \rho , (\eta \cdot \nabla)  \lambda \right\rangle + \left\langle \rho , (\eta \cdot \nabla)  \pounds_{f(w)}  \lambda + (\eta \cdot \nabla) \frac{\delta \ell}{\delta \rho} \right\rangle  \\
&=  -\left\langle \frac{\delta \ell}{\delta w}, \mathrm{ad}_{f(w)} \eta \right\rangle + \left\langle \rho \diamond  \frac{\delta \ell}{\delta \rho}, \eta \right\rangle \\
&=  -\left\langle \mathrm{ad}^{\ast}_{f(w)} \frac{\delta \ell}{\delta w},\eta \right\rangle + \left\langle \rho \diamond  \frac{\delta \ell}{\delta \rho}, \eta \right\rangle.
\end{align*}
Since $\eta$ is arbitrary, in the end we obtain the Euler--Poincar\'{e} equation
\begin{equation}\label{eq:EP_deep_learning}
\begin{aligned}
&\left[ \frac{\partial}{\partial t}\frac{\delta \ell}{\delta w} + \mathrm{ad}^{\ast}_{f(w)} \frac{\delta \ell}{\delta w} - \rho \diamond \frac{\delta \ell}{\delta \rho} \right]_i \, dx^i \otimes d^nx = 0,\\
&\partial_t \rho   \, d^nx + \pounds_{f(w)}(\rho   \, d^nx )= 0.
\end{aligned}
\end{equation}
\end{proof}
The diamond operator, introduced in \cite{holm1998euler},  $\diamond : T^{\ast}\mathrm{Dens}(Q) \to \mathfrak{X}^{\ast}(Q)$ is defined, for any $\eta \in  \mathfrak{X}(Q)$, by
\begin{align}
\left\langle \rho \diamond \lambda, \eta \right\rangle_{ \mathfrak{X}^{\ast}(Q)\times \mathfrak{X}(Q) } = - \left\langle \lambda, \pounds_{\eta} \rho \right\rangle_{ T^{\ast}\mathrm{Dens}(Q)\times  T\mathrm{Dens}(Q) }. 
\end{align}
In terms of \eqref{eq:EP_deep_learning}, it is $\rho \diamond \frac{\delta \ell}{\delta \rho} = \rho \nabla  \frac{\delta \ell}{\delta \rho}$. Here the map $\mathrm{ad}^{\ast}: \mathfrak{X}(Q) \times \mathfrak{X}^{\ast}(Q) \to \mathfrak{X}^{\ast}(Q)$ is the coadjoint map on $\mathfrak{X}^{\ast}(Q)$, and it is obtained by taking the dual of the adjoint map $\mathrm{ad}: \mathfrak{X}(Q) \times \mathfrak{X}(Q) \to \mathfrak{X}(Q)$.  The adjoint map is defined as the commutator of two elements, $w, v$, of the Lie algebra $\mathfrak{X}(Q)$
\begin{align*}
[ w,v] = \mathrm{ad}_{w} v = \left( w \cdot \nabla\right) v- \left( v \cdot \nabla\right) w,
\end{align*}
where the bracket $[ \cdot, \cdot ]: \mathfrak{X}(Q)\times \mathfrak{X}(Q) \to \mathfrak{X}(Q)$ is the Lie bracket of the Lie algebra $\mathfrak{X}(Q)$. To move to the dual of Lie algebra is done using the pairing between $\mathfrak{X}^{\ast}(Q)$ and $\mathfrak{X}(Q)$ which is actualised by the inner product map $\langle \cdot, \cdot \rangle : \mathfrak{X}^{\ast}(Q) \times \mathfrak{X}(Q) \to \mathbb{R}$ defined in \eqref{eq:inner_product_pair}. This allows us to identify the coadjoint map   
\begin{align*}
\left\langle \mathrm{ad}^{\ast}_{w} m, v\right\rangle =  \left\langle  m, \mathrm{ad}_{w}v \right\rangle,  \quad \forall m \in \mathfrak{X}^{\ast}(Q),
\end{align*}
and using integration by parts of the integral \eqref{eq:inner_product_pair} that gives an expression for the coadjoint map
\begin{align}
\mathrm{ad}^{\ast}_{w} m =  \left( w \cdot  \nabla \right)m  + (\nabla w)^T \cdot m +   m \left( \nabla \cdot w \right),
\end{align}
and in coordinates in $\mathbb{R}^N$,
\begin{align}
\left( \mathrm{ad}^{\ast}_{w} m  \right)_{i} = \frac{\partial}{\partial x^j}(m_iw^j) + m_j \frac{\partial w^j}{\partial x^i}.
\end{align}
With that, the equation \eqref{eq:EP_deep_learning} is expressed in coordinates on $\mathbb{R}^N$ as
\begin{equation}\label{eq:EP_deep_learning_2}
\begin{aligned}
\partial_t& \left( \frac{\delta \ell}{\delta w^i} \right) + \frac{\partial}{\partial x^j} \left(\left( \frac{\delta \ell}{\delta w^i} \right)f(w)^j \right) + \left( \frac{\delta \ell}{\delta w^j} \right) \frac{\partial f(w)^j}{\partial x^i} = \rho \frac{\partial}{\partial x^i} \frac{\delta \ell}{\delta \rho}, \\
\partial_t& \rho + \frac{\partial}{\partial x^j}\left( \rho f( w)^j \right) = 0.
\end{aligned}
\end{equation}
This allows us to eliminate the stochastic Hamilton--Jacobi--Bellman equation which is a backward partial differential equation in \eqref{eq:mfg_system2}. Furthermore, both, the Euler--Poincar\'{e} equation with and the continuity equation  \eqref{eq:ep_cont_eqn} are forward in time differential equations and they describe dynamics on the dual of the semidirect product Lie algebra $\mathfrak{s}^{\ast} = \mathfrak{X}^{\ast}(Q) \times \mathrm{Dens}(Q)$. The system of equations \eqref{eq:ep_cont_eqn} does appear when applying Euler--Poincar\'{e} variational principle
\begin{align*}
\delta \int_0^t \int_{Q} \frac{1}{2} \rho w \cdot w + \frac{\nu^2}{2} w \cdot \nabla \rho + \frac{\nu^4}{8} \frac{\nabla \rho \cdot \nabla \rho }{\rho} \, dx \, dt = 0,
\end{align*}
and the variations of $u$ and $\rho$ has the form
\begin{align*}
\delta u = \dot{\eta} + [u, \eta],  & \quad \delta \rho = \nabla \cdot( \rho \eta) d^nx,
\end{align*}
where $\eta \in \mathfrak{X}(Q)$ is an arbitrary vector field with fixed endpoints at the boundary set to zero. The system of equations \eqref{eq:EP_deep_learning_2} can be solved using geometric numerical methods, such as multisymplectic method, and that is the goal of section \ref{sec:var_int}. 
\\

One important property that emerges from the necessary condition for optimality is in fact a momentum map for the action of the group of diffeomorphisms on the probability density function associated with stochastic ResNets.  This agrees with the Clebsch \cite{marsden_coadjoint_1983} framework for dynamics and optimal control \cite{gay2011clebsch}.
\begin{theorem}
The stationary variation of the control vector field $w$ defines an equivariant momentum map
\begin{equation}
\begin{aligned}
&\mathbf{J} : T^{\ast}\mathrm{Dens}(Q) \to \mathfrak{X}^{\ast}(Q) \\
&\mathbf{J}(\rho, \lambda) = \frac{\delta \ell}{\delta w} = \rho \nabla \lambda 
\end{aligned}
\end{equation}
\end{theorem}
\begin{proof}
To demonstrate that $\mathbf{J} = \rho \nabla \lambda$ is a momentum map, we begin by defining a function $J_{\beta} : T^{\ast}\mathrm{Dens}(Q) \to \mathbb{R}$, explicitly
\begin{align}\label{eq:momentum_map_definition}
J_{\beta} = \left\langle \mathbf{J}(\rho, \lambda), \beta \right\rangle_{\mathfrak{X}^{\ast} \times \mathfrak{X}} = \int \rho (\partial_{x^i} \lambda) \beta^i \,  dx,
\end{align}
where $\beta \in \mathfrak{X}(Q)$ and $\langle \cdot, \cdot \rangle : \mathfrak{X}^{\ast} \times \mathfrak{X} \to \mathbb{R}$ is a natural pairing between $\mathfrak{X}^{\ast}$ and $\mathfrak{X}$. For all $F \in T^{\ast}\mathrm{Dens}(Q)$, we compute the Poisson bracket on $T^{\ast}\mathrm{Dens}(Q)$ of $F$ and $J_{\beta}$ and it is
\begin{align}
\{ F, J_{\beta} \} = \int_Q -\frac{\delta F}{\delta \rho}\partial_{x^j}( \rho \beta^j ) - \rho \frac{\delta F}{\delta \lambda} (\partial_{x^i} \lambda) \beta^i \, dx
\end{align}
The right--hand side can be stated as 
\begin{align}
\{ F, J_{\beta} \} = \left\langle \mathbf{d}F, X_{J_{\beta}} \right\rangle, 
\end{align}
where $X_{J_{\beta}}$ is a vector field given by
\begin{align*}
X_{J_{\beta}} = \left( -\partial_{x^i}( \rho \beta^i), -  \partial_{x^i} \lambda \beta^i \right) = \left( \frac{\partial J_{\beta}}{\partial \lambda},- \frac{\partial J_{\beta}}{\partial \rho} \right),
\end{align*}
And this vector field is corresponds to the infinitesimal generator of the action of $\mathrm{Diff}(Q)$, whose associated vector field is $\beta$, on $T^{\ast}\mathrm{Den}(Q)$ and it is a Hamiltonian vector,  which verifies that \eqref{eq:momentum_map_definition} holds and thus $\mathbf{J}$ is a momentum map. 
\\

To show that the momentum map is equivariant, i.e. for any $g \in \mathrm{Diff}(Q)$ group action on all $(\rho, \lambda)$ the following holds
\begin{align}
\mathbf{J}(g \cdot \rho, g \cdot \lambda) = \mathrm{Ad}^{\ast}_{g^{-1}} \mathbf{J}(\rho, \lambda).
\end{align}
Pairing $\mathbf{J}(g \cdot \rho, g \cdot \lambda) $ with any $\beta \in \mathfrak{X}(Q)$, with $g(0) = e$ and $\dot{g}(0) = \eta \in \mathfrak{X}(Q)$. Taking the time derivative of the left--hand side, we have
\begin{align*}
\left. \frac{d}{dt} \right|_{t=0} J_{\beta}(g \cdot \rho, g \cdot \lambda) &= \left. \frac{d}{dt} \right|_{t=0} \int (g^{\ast} \rho) (\partial_{x^i} g^{\ast}\lambda) \beta^i \, dx \\
&=  \int \partial_{x^j} ( \rho \eta^j) (\partial_{x^i} \lambda) \beta^i \, dx +\int   \rho \partial_{x^j} ( (\partial_{x^i} \lambda) \eta^i ) \beta^j \, dx \\
&= -\int  \rho \partial_{x^j}( (\partial_{x^i} \lambda) \beta^i  ) \eta^j \, dx + \int \rho \partial_{x^j}( (\partial_{x^i} \lambda) \eta^i ) \beta^j \, dx \\
\end{align*}
Expanding the gradient of the dot products and rearranging terms, the time derivative reduces to 
\begin{align*}
\left. \frac{d}{dt} \right|_{t=0} J_{\beta}(g \cdot \rho, g \cdot \lambda) &= \int \rho (\partial_{x^j} \lambda) \left[ \beta^i\partial_{x^i} \eta^j - \eta^i\partial_{x^i} \beta^j \right] \, dx
\end{align*}
Thus we have $\{ J_{\eta}, J_{\beta} \} = J_{[\eta, \beta]}(\rho, \lambda)$, which is known as the infinitesimal equivariance. Also, it amounts to the time derivative of $J_{\mathrm{Ad}_{g^{-1}}\beta }(\rho, \lambda)$, which we compute explicitly 
\begin{align*}
\left. \frac{d}{dt} \right|_{t=0} J_{\mathrm{Ad}_{g^{-1}}\beta }(\rho, \lambda) &= \left. \frac{d}{dt} \right|_{t=0} \int (\mathrm{Ad}^{\ast}_{g^{-1}}(\rho (\partial_{x^i} \lambda) )) \beta^i \, dx = - \int \mathrm{ad}^{\ast}_{\eta} (\rho (\partial_{x^i} \lambda) )) \beta^i \, dx \\
&=  -\int \rho (\partial_{x^j} \lambda) (\mathrm{ad}_{\beta} \eta)^j \\
&= J_{[\eta, \beta]}(\rho, \lambda),
\end{align*}
thus we have $ J_{\beta}(g \cdot \rho, g \cdot \lambda) = \mathrm{Ad}^{\ast}_{g^{-1}}J_{\beta}(\rho, \lambda)$ and that implies
$$
\mathbf{J}(g \cdot \rho, g \cdot \lambda) = \mathrm{Ad}^{\ast}_{g^{-1}} \mathbf{J}(\rho, \lambda),
$$
which proves that the momentum map $\mathbf{J} : T^{\ast}\mathrm{Dens}(Q) \to \mathfrak{X}^{\ast}(Q)$ is an equivariant momentum map. 
\end{proof}

\section{Multisymplectic variational integrators for deep learning}\label{sec:var_int}
Often case in optimal control problems, one is presented with the choice between an optimise-first-then-discretise approach or discretise-first-then-optimise.  In fact, the methodology presented here does not deviate from discrete mechanics for optimal control of \cite{junge_discrete_2005,ober-blobaum_discrete_2011,de_leon_discrete_2007}, where a discrete variational principle was used to translate the optimal control problem into a nonlinear programming one.  This means we discretise-first-then-optimise
\\

The family of multisymplectic integrators were presented in \cite{marsden1998multisymplectic, bridges2001multi}. There are two ways to obtain such numerical scheme: discretising the variational integrator, and discretising the abstract form of multisymplectic partial differential equation \cite{bridges2001multi, reich_multi-symplectic_2000}. Multisymplectic variational integrators were first developed in \cite{marsden1998multisymplectic} and many discrete analogues such discrete Cartan one-forms, and discrete Noether's theorem were also introduced. Then in \cite{lew2003asynchronous} further generalisations regarding discretisation of the base manifold were investigated, others introduced moving mesh points on the base manifold \cite{tyranowski_r-adaptive_2019}. Variational integrators were applied to various nonlinear differential equations, such as the nonlinear Schro\"{o}dinger equation \cite{ chen_multisymplectic_2002}, the Camassa-Holm equation  \cite{kouranbaeva_variational_2000}, Korteweg-de Vries \cite{holm_new_2018} and the beam equation in $\mathbb{R}^3$ \cite{demoures_multisymplectic_2014,demoures_multisymplectic_2015}.  On the other hand, multisymplectic integrators can be constructed by directly discretising the Hamiltonian partial differential equation as in \cite{bridges2001multi,reich_multi-symplectic_2000,reich_finite_2000,bridges2001multi,moore_multi-symplectic_2003,ascher_multisymplectic_2004,islas_multi-symplectic_2003,cohen_multi-symplectic_2014}. A nice property of such integrators is that there is a heuristic approach to constructing modified multisymplectic Hamiltonian for backward error analysis, as it was shown in \cite{moore_backward_2003,moore_multi-symplectic_2003,ISLAS2005290}.
\\

Here we only focus on variational integrators based on \cite{marsden1998multisymplectic}. Variational integrators are, namely, based on Veselov discretisation for mechanics \cite{veselov_integrable_1988}, and many properties and theorems, such as Noether's theorem and conservation laws, have a discrete equivalent. Without the loss of generalisation, here we consider problems with two dimensions for space and one dimension for time, hence the base manifold is three dimensional manifold.  The first step in constructing a variational integrator is to discretise the base manifold $Q$ as $\mathbb{B} := \mathbb{Z} \times \mathbb{Z}\times \mathbb{Z} = \{ (i,j, n)\}$ and then the with that the fibre bundle $\mathbb{E} = \mathbb{B} \times \mathcal{F}$, where $\mathcal{F}$ is a smooth manifold. The quantity $y_{i,j,n}$ is an element of $\mathbb{E}$ over the point $(i,j,n)$. Since the base manifold now is a discrete space, that requires the partial derivatives with respect to independent variables $(x^1,x^2,t)$ and to be replaced with finite differences. For our application, midpoint approximation in time and space is used and discretisation renders the grid as a collection of right-angled pyramids with triangular bases, which a pyramid is defined by the four points 
$$
 \Delta^n_{i,j } = ((i,j,n), (i,j, n+1), (i+1,j, n), (i,j+1, n)).
$$
The set of all pyramids is defined as
\begin{align*}
\mathbb{B} := \left\{  \Delta_{i,j,n} \quad | \quad (i,j, n) \in (0, \dots, N_{x^1}) \times ( 0, \dots, N_{x^2}) \times (0, \dots, N_t) \right\}.
\end{align*}
With that, the first jet bundle becomes
\begin{align*}
J^1(\mathbb{E}) &= \left\{ \left. (\phi_{i,j,n}, \phi_{i,j,n+1}, \phi_{i+1,j,n},  \phi_{i,j+1,n} ) \right| (i,j,n) \in \mathbb{B}, \phi_{i,j,n}, \phi_{i,j,n+1}, \phi_{i+1,j,n},  \phi_{i,j+1,n} \in \mathcal{F} \right\}  \\
&= \mathbb{B} \times \mathcal{F}^4,
\end{align*}
we note that $(\phi_{i,j,n}, \phi_{i,j,n+1}, \phi_{i+1,j,n},  \phi_{i,j+1,n}  ) \in J^1(\mathbb{E})$ are basically $j^1 (\phi(\bar{x}))$, where $\bar{x}_{i,n}$ is the centre of $ \Delta_{i,j,n}$  and $\phi$ is a smooth section of $E$. As it is in discrete mechanics, the starting point of constructing a variational integrator is the discretisation of the action functional. The discrete action $\mathcal{S}_d$ is given by
\begin{align*}
\mathcal{S}_d &=\sum\limits_{(i,j,n) \in \mathbb{B}} \widehat{L} \left( \Delta_{i,j,n} ,  w^1_{i,j,n}, w^2_{i,j,n}, \lambda_{i,j,n}, \rho_{i,j,n} \right) \Delta x^1  \Delta x^2 \Delta t, 
\end{align*}
where $\widehat{L}: J^1(\mathbb{E}) \to \mathbb{R}$ is the discrete Lagrangian and it is defined by
\begin{align*}
\widehat{L}  &= \int_{t}^{t + \Delta t } \int_{x^2}^{x^2 + \Delta x^2} \int_{x^1}^{x^1 + \Delta x^1}\frac{\rho}{2} \left\| w \right\|^2 +  \frac{\nu^2}{2} \rho \left\langle w, \left(\frac{\partial f}{\partial \zeta}\right)^{-1} \nabla \mathrm{log} {\rho} \right\rangle \\
&+ \frac{ \nu^4 \rho}{8} \left\| \left(\frac{\partial f}{\partial \zeta}\right)^{-1} \nabla \mathrm{log} {\rho} \right\|^2  + \left\langle \lambda, \partial_t \rho + \pounds_{f(w)}\rho \right\rangle dx^1 \, dx^2 \, dt.
\end{align*}
Here, and henceforth, we use $\Delta_{i,j,n}$ as an argument as a short-hand notation for the tuple
\begin{align*}
\widehat{L} \left( \Delta_{i,j,n} ,  w^1_{i,j,n},  w^2_{i,j,n}, \lambda_{i,j,n}, \rho_{i,j,n} \right)  &= \widehat{L} \left(w^{1}_{i,j,n}, w^{1}_{i,j,n+1}, w^{1}_{i+1,j,n}, w^{1}_{i,j+1,n}, \right. \\
&\left. w^{2}_{i,j,n}, w^{2}_{i,j,n+1}, w^{2}_{i+1,j,n}, w^{2}_{i,j+1,n},\lambda_{i,j,n}, \lambda_{i,j,n+1},  \right. \\
&\left.  \lambda_{i+1,j,n}, \lambda_{i,j+1,n}, \rho_{i,j,n}, \rho_{i,j,n+1}, \rho_{i+1,j,n}, \rho_{i,j+1,n}\right).
\end{align*}
Further, to evaluate the Lagrangian on the discrete first jet bundle $J^1(\mathbb{E})$, we require the use of discrete jet prolongations. There are a number of ways to do so, and here we use a centred discretisation and with that for the pyramid $\Delta_{i,j,n}$ and the section $\phi \in J^1(\mathbb{E})$, the space and time derivatives are replaced by
\begin{align*}
\frac{\partial \phi}{\partial x^1} &\approx \frac{\phi_{i+1,j+1/2,n+1/2} - \phi_{i,j+1/2,n+1/2}}{2\Delta x^1}, \quad \frac{\partial \phi}{\partial x^2} \approx \frac{\phi_{i+1/2,j+1,n+1/2} - \phi_{i+1/2,j,n+1/2}}{2\Delta x^2}, \\
\frac{\partial \phi}{\partial t} &\approx \frac{\phi_{i+1/2,j+1/2,n+1} - \phi_{i+1/2,j+1/2,n}}{2\Delta t}.
\end{align*}
Given the discrete action
\begin{align*}
\mathcal{S}_{d}&= \sum\limits_{(i,j,n) \in \mathbb{B}} \widehat{L} \left( \Delta_{i,j,n} ,  w^1_{i,j,n}, w^2_{i,j,n}, \lambda_{i,j,n}, \rho_{i,j,n} \right) \Delta x^1 \Delta x^2 \Delta t,
\end{align*}
here, the Lagrangian on the discrete $J^1(\mathbb{E})$ becomes
{\small
\begin{equation}\label{eq:discrete_lagrangian}
\begin{aligned}
\widehat{L}  &= \frac{1}{2}\bar{\rho}_{i,j,n}((\bar{w}^1)_{i,j,n})^2 +\frac{1}{2}\bar{\rho}_{i,j,n}((\bar{w}^2)_{i,j,n})^2 + \frac{\nu^2}{2(\Delta x^2)} \left(\rho_{i+1/2,j+1,n+1/2}\right)(\bar{w}^2)_{i,j,n}\\
&- \frac{\nu^2}{2(\Delta x^2)} \left(\rho_{i+1/2,j,n+1/2} \right)(\bar{w}^2)_{i,j,n}  + \frac{\nu^2}{2(\Delta x^1)} \left(\rho_{i+1,j+1/2,n+1/2}- \rho_{i,j+1/2, n+1/2} \right)(\bar{w}^1)_{i,j,n} \\
&+\frac{\nu^4}{8}\left(   \frac{\left( \rho_{i+1/2,j+1,n+1/2}- \rho_{i+1/2,j,n+1/2} \right)^2}{(\Delta x^2)  \bar{\rho}_{i,j,n}  } + \frac{\left( \rho_{i+1,j+1/2,n+1/2} - \rho_{i,j+1/2, n+1/2}\right)^2}{(\Delta x^1) \bar{\rho}_{i,j,n} } \right)   \\
&+ \bar{\lambda}_{i,j,n} \left( \frac{\rho_{i+1/2,j+1/2,n+1} - \rho_{i+1/2,j+1/2, n}}{\Delta t} + \frac{\rho_{i+1,j+1/2,n+1/2}w^1_{i+1,j+1/2,n+1/2}}{(\Delta x^1)} \right. \\
&\left. -  \frac{\rho_{i,j+1/2, n+1/2}w^1_{i,j+1/2, n+1/2}}{(\Delta x^1)}+ \frac{\rho_{i+1/2,j+1,n+1/2}w^2_{i+1/2,j+1,n+1/2} }{(\Delta x^2)} \right. \\
&\left.- \frac{ \rho_{i+1/2,j,n+1/2}w^2_{i+1/2,j,n+1/2}}{(\Delta x^2)}\right).
\end{aligned}
\end{equation}
}
Here, the bar over the variables $(\bar{w}^1)_{i,j,n}, (\bar{w}^2)_{i,j,n}, \bar{\rho}_{i,j,n}, \bar{\lambda}_{i,j,n}$ are used to denote the average of vertices of $\Delta_{i,j,n}$.  Similar to the previous schemes, we take the variations of the discrete action and set it equal to zero, i.e. $\delta \mathcal{S}_{d} = 0$, 
\begin{align*}
0 &= \delta \mathcal{S}_{d} \\
&=\sum\limits_{(i,j,n) \in \mathbb{B}}  \sum\limits_{ \alpha_i,   \alpha_j,   \alpha_n=-1}^{1}  \left\langle \frac{\partial \hat{L}}{\partial (w^1)_{i+\alpha_i, j+ \alpha_j, k+ \alpha_n}}, \delta w^1_{i+\alpha_i, j+ \alpha_j, k+ \alpha_n} \right\rangle   \\
& + \left\langle \frac{\partial \hat{L}}{\partial (w^2)_{i+\alpha_i, j+ \alpha_j, k+ \alpha_n}}, \delta w^2_{i+\alpha_i, j+ \alpha_j, k+ \alpha_n} \right\rangle +  \left\langle \frac{\partial \hat{L}}{\partial \rho_{i+\alpha_i, j+ \alpha_j, k+ \alpha_n}}, \delta \rho_{i+\alpha_i, j+ \alpha_j, k+ \alpha_n} \right\rangle \\
&+ \left\langle \delta \lambda_{i+\alpha_i, j+ \alpha_j, k+ \alpha_n}, \frac{\partial \hat{L}}{\partial \lambda_{i+\alpha_i, j+ \alpha_j, k+ \alpha_n}} \right\rangle.
\end{align*}
Here we use the notation $D_{\alpha_i,\alpha_j, \alpha_n, z}\hat{L}(\phi_d(\Delta_{i,j,n}))$ indicates the derivative of $\hat{L}(\phi_d(\Delta_{i,j,n}))$ with respect to $\phi_d(\Delta_{i,j,n})= (w^1_{i,j,n}, w^2_{i,j,n}, \rho_{i,j,n}, \lambda_{i,j,n})$. Shifting the indices, we factor out $\delta w^1_{i,j,n}$, $\delta w^2_{i,j,n}$, $\delta \rho_{i,j,n}$, and $\delta \lambda_{i,j,n}$ and for $\delta \mathcal{S}_{d}  = 0$ and we obtain the \emph{discrete Euler-Lagrange equation} \cite{marsden1998multisymplectic}:
\begin{align}\label{eq:discrete_ELEQ}
\sum\limits^{1}_{ \alpha_i,  \alpha_j,  \alpha_n=-1}  \frac{\partial  \hat{L}(\phi_d(\Delta_{i- \alpha_i,j - \alpha_j ,n - \alpha_n}) ) }{\partial \phi_d(\Delta_{i+ \alpha_i,j + \alpha_j ,n + \alpha_n})} &= 0, 
\end{align}
for $i = 1 \dots N_x$, $j = 1\dots N_y$ and $n=1\dots N_t$. To illustrate how the discrete variations are obtained, and without the loss of generality, we consider the one-dimensional of this problem, where the Lagrangian is given by
\begin{align*}
\widehat{L} &= \frac{1}{2} \bar{\rho}_{i,n}(\bar{w}_{i,n})^2 + \frac{\nu^2}{2} \bar{w}_{i,n} \left( \frac{\rho_{i+1, n+1/2} - \rho_{i-1, n+1/2}}{2\Delta x} \right) \\
& - \frac{\nu^4}{8}\bar{\rho}_{i,n} \left( \frac{\rho_{i+1, n+1/2}  -2\rho_{i, n+1/2} + \rho_{i-1, n+1/2} }{\bar{\rho}_{i,n} (\Delta x)^2}  - \frac{1}{8} \frac{(\rho_{i+ 1, n+1/2} - \rho_{i-1, n+1/2})^2}{(\bar{\rho}_{i,n})^2 (\Delta x)^2}\right)   \\
&+ \bar{\lambda} \left( \frac{\rho_{i+1/2, n+1} - \rho_{i+1/2, n-1}}{2\Delta t} + \frac{\rho_{i+1, n+1/2}w_{i+1, n+1/2} - \rho_{i-1, n+1/2}w_{i-1, n+1/2} }{2\Delta x}   \right).
\end{align*}
For the $0 = \delta \mathcal{S}_d$ to hold, we require the following equations to be satisfied:
{\small
\begin{equation}\label{eq:DELF_2}
\begin{aligned} 
0 &= \frac{\rho_{i, n-1/2}w_{i, n-1/2}}{2} + \frac{\rho_{i, n+1/2}w_{i, n+1/2}}{2} + \frac{\nu^2}{4}\left( \frac{\rho_{i+1, n-1/2} + \rho_{i+1, n+1/2} }{2\Delta x} \right. \\
&\left.- \frac{\rho_{i-1, n-1/2} + \rho_{i-1, n+1/2}}{2\Delta x}    \right) - \frac{1}{4}\left( \frac{\lambda_{i+1, n-1/2}\rho_{i, n-1/2} +\lambda_{i+1, n+1/2}\rho_{i, n+1/2} }{\Delta x} \right. \\
& \left.- \frac{\lambda_{i-1, n-1/2}\rho_{i, n-1/2} +\lambda_{i-1, n+1/2}\rho_{i, n+1/2} }{\Delta x}   \right), \\
0 &= \frac{\rho_{i+1/2, n+1} -\rho_{i+1/2, n-1}}{2\Delta t} + \frac{\rho_{i+1, n-1/2}w_{i+1, n-1/2}  - \rho_{i-1, n-1/2}w_{i-1, n-1/2} }{4\Delta x}  \\
&+ \frac{\rho_{i+1, n+1/2}w_{i+1, n+1/2}  - \rho_{i-1, n+1/2}w_{i-1, n+1/2} }{4\Delta x} , \\
0 &= -\frac{\lambda_{i, n+1} - \lambda_{i, n-1} }{2\Delta t} - \frac{\nu^2}{4}\frac{ w_{i+1, k-1/2} + w_{i+1, k+1/2} - w_{i-1, k-1/2} - w_{i-1, k+1/2} }{\Delta x} \\
&-\frac{(\lambda_{i+1, n-1/2} - \lambda_{i-1, n-1/2}  ) w_{i, n-1/2}   }{4\Delta x}-\frac{(\lambda_{i+1, n+1/2} - \lambda_{i+1, n-1/2}  ) w_{i, n+1/2}   }{4\Delta x} \\
&+\frac{(w_{i, n-1/2})^2}{4} +\frac{(w_{i, n+1/2})^2}{4} -\frac{\nu^4}{8}\left( - \frac{\rho_{i, n-1/2} - \rho_{i-2, k-1/2} }{8 (\rho_{i-1, n-1/2})(\Delta x)^2}    - \frac{\rho_{i, n+1/2} - \rho_{i-2, k+1/2} }{8 (\rho_{i-1, n+1/2})(\Delta x)^2}       \right) \\
&-\frac{\nu^4}{8}\left( - \frac{\rho_{i, n-1/2} - \rho_{i+2, k-1/2} }{8 (\rho_{i+1, n-1/2})(\Delta x)^2}   - \frac{\rho_{i, n+1/2} - \rho_{i+2, k+1/2} }{8 (\rho_{i+1, n+1/2})(\Delta x)^2}     \right) \\
& -\frac{\nu^4}{16}\left( - \frac{1}{8}  \frac{  \left(\rho_{i+1, n-1/2} - \rho_{i-1, n-1/2} \right)^2}{  (\rho_{i, n-1/2})^2 (\Delta x)^2}   - \frac{\rho_{i+1, n-1/2} - \rho_{i-1, n-1/2}}{ \rho_{i, k-1/2} (\Delta x)^2}    \right) \\
&-\frac{\nu^4}{8}\left( \frac{1}{2} \frac{\rho_{i+1, n-1/2} - \rho_{i-1, n-1/2} }{(\rho_{i, n-1/2})(\Delta x)^2}  + \frac{1}{8} \frac{(\rho_{i+1, n-1/2} - \rho_{i-1, n-1/2})^2}{(\rho_{i, n-1/2})^2 (\Delta x)^2}  \right) \\
&-\frac{\nu^4}{16}\left( - \frac{1}{8}  \frac{  \left(\rho_{i+1, n+1/2} - \rho_{i-1, n+1/2} \right)^2}{  (\rho_{i, n+1/2})^2 (\Delta x)^2}   - \frac{\rho_{i+1, n+1/2} - \rho_{i-1, n+1/2}}{ \rho_{i, k+1/2} (\Delta x)^2}    \right) \\
&-\frac{\nu^4}{8}\left( \frac{1}{2} \frac{\rho_{i+1, n+1/2} - \rho_{i-1, n+1/2} }{(\rho_{i, n+1/2})(\Delta x)^2}  + \frac{1}{8} \frac{(\rho_{i+1, n+1/2} - \rho_{i-1, n+1/2})^2}{(\rho_{i, n+1/2})^2 (\Delta x)^2}  \right),
\end{aligned}
\end{equation}
}
for $i = 1..N_x$, $j = 1..N_y$ and $n=1...N_t$.  The discretisation above is multisymplectic, in the sense that it preserves the discrete counterpart of the multisymplectic form formula \eqref{eq:multisymplectic_form_formula}.  To do so, first, we need to compute \emph{the discrete Cartan one-forms} which are the result of the contributions of the exterior derivative $d\mathcal{S}_{d}$ from the boundaries. The discrete Cartan one-forms are computed according to the formula
\begin{equation}\label{eq:discrete_one_forms}
\begin{aligned}
\Theta_{d}^{w,\alpha} &= \frac{\partial \widehat{L}}{\partial (w)_{i + \alpha_1, n+ \alpha_2}}dw_{i + \alpha_1, n+ \alpha_2}, \\
   \Theta_{d}^{\rho,\alpha} &= \frac{\partial \widehat{L}}{\partial \rho_{i + \alpha_1, n+ \alpha_2}}d\rho_{i + \alpha_1, n+ \alpha_2}, \\
    \text{and}\quad  \Theta_{d}^{\lambda,\alpha} &= \frac{\partial \widehat{L}}{\partial \lambda_{i + \alpha_1, n+ \alpha_2}}d\lambda_{i + \alpha_1, n+ \alpha_2},
\end{aligned}
\end{equation}
where $\alpha$ is a multi-index $\alpha = (\alpha_1,\dots , \alpha_{N+1})$. The one-forms in \eqref{eq:discrete_one_forms} are for a single $\Delta_{i,n}$ and for the contribution of the boundaries of the whole action $d\mathcal{S}_d$, we need to sum all these one-forms for all $ \Delta_{i,n}$, $i=1,\dots, N_x$ and $n=1, \dots, N_t$.  As with taking variations of an action functional, as seen in \ref{sec:mul_geo}, the boundary terms contribute to $d\mathcal{S}_{d}$ and when restricted to the vector field $V$ are
\begin{align*}
    \frac{\partial \widehat{L}}{\partial \phi_{i,n}} \cdot (V_i^n) = (j^1(\phi_d))^{\ast} \left(j^1(V) \contract \Theta_d^{z,\alpha} \right),
\end{align*}
and with that, the total contribution of the one-forms to $d\mathcal{S}_{d}$ is 
\begin{align*}
    \mathrm{d}\mathcal{S}_{d} \cdot V = \sum\limits_{\mathbb{B}} \sum\limits_{\alpha} (j^1(\phi_d))^{\ast}\left(j^1(V) \contract \Theta_d^{z,\alpha} \right).
\end{align*}
Taking the exterior derivative of the of $d\mathcal{S}_{d} \cdot V$  and restricting it to vector field $W$ yields $d^2\mathcal{S}_d = 0$ and that tells us the sum is
\begin{align}\label{eq:discrete_multisymplectic_form_formula}
    \sum\limits_{\mathbb{B}} \sum\limits_{\alpha} (j^1(\phi_d))^{\ast} \left(j^1(V) \contract j^1(W) \contract \Omega_d^{z,\alpha} \right)= 0,
\end{align}
where $\Omega_d^{z,\alpha}= -d \Theta_d^{z,\alpha}$.  In fact, \eqref{eq:discrete_multisymplectic_form_formula} is analogous to \eqref{eq:multisymplectic_form_formula} and it is known as \emph{the discrete multisymplectic form formula} \cite{marsden1998multisymplectic}. The importance of this is that when the discrete multisymplectic form formula is satisfied it means the discretisation scheme is multisymplectic. With the multisymplectic integrator at hand, the training problem \eqref{prob:mfg_deep} can now be stated as:

\begin{problem}\label{prob:mfg_deep_discrete}
Let $(Q, \mathcal{F}, \mathbb{P})$ be a probability space and the stochastic control system $\Sigma_{\rho} = (\mathrm{Dens}(Q), \nabla , \mathfrak{X}(Q))$. Let $\{(\rho^{(k)}_{i,j,0}, c^{(k)}) \}$ be a training set with $\rho^{(k)}_{i,j,0} \in \mathrm{Dens}(Q)$ and $c^{(k)} \in \mathcal{C}$ for $k=0, \dots, N$ be . The training of the stochastic ResNet \eqref{eq:resNet_eq} is done by finding a set $( w^1_{i,j,n}, w^2_{i,j,n},\lambda_{i,j,n}, \rho_{i,j,n} ) \in J^1(E)$, such that 
$$
\mathcal{S}_{d,\Sigma_{\rho}}((w^{1})^{\ast}_{i,j,n},(w^{2})^{\ast}_{i,j,n},  \rho^{\ast}_{i,j,n} ) < \mathcal{S}_{d, \Sigma_{\rho}}(w^1_{i,j,n},w^2_{i,j,n},  \rho_{i,j,n} ), \quad \forall (w^1_{i,j,n},w^2_{i,j,n},  \rho_{i,j,n} ) \in J^1(E),
$$
where the cost functional is defined by
\begin{equation}
\begin{aligned}
\mathcal{S}_{d,\Sigma_{\rho}}&= \sum\limits_{(i,j,n) \in \mathbb{B}} \widehat{L} \left( \Delta_{i,j,n} ,  w^{1}_{i,j,n},w^{2}_{i,j,n}, \lambda_{i,j,n}, \rho_{i,j,n} \right) \Delta x^{1}\Delta x^{2} \Delta t \\
&+ \sum\limits_{k=0}^N \Phi( \pi( \rho_{i,j,N_t} ), c^{(k)} ),
\end{aligned}
\end{equation}
where $\widehat{L}$ is the cost Lagrangian defined in \eqref{eq:discrete_lagrangian} and also satisfies the discrete Euler--Lagrange equation \eqref{eq:discrete_ELEQ}.
\end{problem}

\subsection{Applications of multisymplectic integrators}\label{sec:application}
To demonstrate the application of the fluid model for machine learning, here we apply our hydrodynamic formulation and multisymplectic variational integrator to solve machine learning problems: density matching and generating new samples using autoencoder. The direct involvement of probability density function in the hydrodynamic framework for deep learning naturally lends itself to a branch of machine learning known as generative models. The main aim of such models is to learn the joint probability distribution from data and generate new data points, e.g., generate new samples of a dataset. Study of such models has gained a considerable interest in recent years with the introduction of generative adversarial networks (GAN) \cite{goodfellow2014}, normalizing flow based generative model \cite{rezende15,grathwohl_ffjord:_2018, yang2019pointflow,ping_waveflow_2020}, variational autoencoders \cite{an2015variational}, and restricted Boltzmann machines \cite{nair2010rectified} to name a few. Each method implements generative models in a different way. The closest approach to the one presented here is the normalizing flow approach, especially its variant, deep diffeomorphic normalizing flows \cite{salman2018deep}.
\\

\subsubsection{Density matching}\label{sec:density}
The first problem we consider is approximating posterior distributions, which need not to be Gaussian. One solution approach that gained traction was initiated in \cite{rezende15}, where a discrete dynamical system was used to push-forward data sampled from a Gaussian distribution to a target distribution. The push-forward, implemented by a neural network, is applied successively $K$ times.  Here, we begin with the continuous dynamical system
\begin{align}
(\dot{x}^1, \dot{x}^2) = \left( w^1(x^1,x^2, t), w^2(x^1,x^2, t) \right),
\end{align}
where $(x^1,x^2)$ are data samples, $(w^1,w^2): Q \times \mathbb{R} \to TQ$ is the vector field whose flow, $F_{(w^1,w^2)} $, is the transformation we are aiming to find. In other words, the flow of $(w^1,w^2)$ maps the data $(x^1,x^2) \in Q$ to  $(z^1, z^2) \in \widetilde{Q}$
\begin{align}
(z^1, z^2) = \left( F_{(w^1,w^2)}\right)_{\ast}((x^1,x^2)),
\end{align} 
where both $Q, \widetilde{Q} \subset \mathbb{R}^2$ are smooth manifolds. There are many variations on the theme of normalizing flows, where $(w^1,w^2)$ can be fixed to be a specific type, for example, it can be a constant function, whose flow is a linear function and this type is known as the planar flow,  as in the case of the original paper \cite{rezende15}. Another is to impose $(w^1,w^2)$ a diffeomorphism, namely a deep diffeomorphic normalizing flows in \cite{salman2018deep}. Another interesting choice of $(w^1,w^2)$ is choosing it as a potential function, which was introduced by \cite{zhang2018monge}, thus the flow also solves of the Monge--Amp\`{e}re equation. The latter is connected to the approach of \cite{benamou2000computational}, where they showed that Monge--Amp\`{e}re is also the solution of optimal transport problem \cite{benamou2000computational}. The main difference between the approach here and the aforementioned ones, is that here the flow $F_{(w^1,w^2)} $ is computed directly on a fixed Eulerian grid, whereas in \cite{rezende15, zhang2018monge} deep neural networks have been employed to approximate flow using Lagrangian coordinates.
\\

The idea behind finding $F_{(w^1,w^2)} $ is to enable us to learn the joint probability density function from sampled data and once the transformation is obtained, new data samples can be generated by, first, generating samples from distributions on $\mathrm{Dens}(\widetilde{Q})$ and then transforming them to new samples from  $\mathrm{Dens}({Q})$. This is the idea of normalizing flows and it helps in implementing generative models.  In the probability density formulation, the problem becomes what is known as the \emph{Monge--Kantorovich optimal transport problem}. Given two density functions $\rho_0 \geq 0$ and $\rho_T \geq 0$, and optimal transport concerns finding a smooth map, $\eta$, such that $(\eta)_{\ast} \rho_t = \rho_0$, and at the same time the map $g$ minimises the distance between the two prescribed densities. The distance in this case is the $p$-Wasserstein distance and it is defined by
\begin{align}
W_p(\rho_0, \rho_T) = \inf\limits_{g} \int | g(x^1,x^2) - (x^1,x^2) |^p \rho_0(x^1,x^2) dx^1dx^2, 
\end{align}
where $| \cdot |^p$ is the $L^p$ norm on $\mathbb{R}^d$. In \cite{benamou2000computational}, it was shown that this problem is equivalent to the following fluid dynamics problem
\begin{equation}
\begin{aligned}
\mathcal{S} &= \int_0^T \int_{\mathbb{R}^d} \rho(t, x^1,x^2) | (w^1, w^2)|^2  + \lambda \left(\partial_t \rho + \nabla \cdot(\rho (w^1, w^2))  \right) dx^1dx^2\, dt \\
&+ \int_0^T \Pi \left(\int \rho(t,x^1,x^2) dx^1dx^2 - 1\right),
\end{aligned}
\end{equation}
such that $\rho(0,x^1,x^2) = \rho_0(x^1,x^2)$ and $\rho(T, x^1,x^2) = \rho_T(x^1,x^2)$. The term $\Pi$ is the Lagrange multiplier that enforces the incompressibility condition, and in hydrodynamics terms it acts as the pressure force. The optimality condition gives $(w^1, w^2) = \nabla \lambda$ and the Lagrange multiplier satisfies the Hamilton--Jacobi equation. Taking gradient of the Hamilton--Jacobi equation gives an equation for $u$ and it is \emph{incompressible Euler's fluid equation}
\begin{align}
\partial_t (w^1, w^2) + \left( (w^1, w^2) \cdot \nabla \right) (w^1, w^2) = - \nabla \Pi, \quad \nabla \cdot (w^1, w^2) = 0.
\end{align}   
As pointed out in remark \ref{rem:monge_ampere}, the system \ref{prob:mfg_deep}, when $\sigma \to 0$ we recover the optimal transport of Benamou and Brenier \cite{benamou2000computational} and it can be viewed as its regularisation. Given the initial probability density $\rho_0(x^1,x^2) = \mathcal{N}(0, \mathbf{I}) \in \mathrm{Dens} \left(\mathbb{R}^2\right)$ and the target density $\rho_T = e^{-U(x^1,x^2)}$. To test our method, $U(x^1,x^2)$ is chosen as  
\begin{equation}\label{eq:density_cases}
 U(x^1,x^2) = \begin{cases}
    \frac{1}{2} \left( \frac{\| x^1 \| - 2}{0.4}\right)^2 - \mathrm{ln}\left( e^{-\frac{1}{2} \left(\frac{x^1 - 2}{0.4} \right)^2}  + e^{-\frac{1}{2} \left(\frac{x^1 + 2}{0.6} \right)^2} \right),   \\
    \frac{1}{2} \left( \frac{x^2 - \mathrm{sin}\left( \frac{2 \pi x^1}{4}\right)}{0.4} \right)^2,  \\
    - \mathrm{ln}\left( e^{-\frac{1}{2} \left(\frac{x^2 - \mathrm{sin}\left( \frac{2 \pi x^1}{4}\right)}{0.35} \right)^2}  + e^{-\frac{1}{2} \left(\frac{x^2 - \mathrm{sin}\left( \frac{2 \pi x^1}{4}\right) + 3e^{-\frac{1}{2} \left( \frac{(x^1 - 1)}{0.6}\right)^2}}{0.35} \right)^2} \right),   \\
    - \mathrm{ln}\left( e^{-\frac{1}{2} \left(\frac{x^2 - \mathrm{sin}\left( \frac{2 \pi x^1}{4}\right)}{0.4} \right)^2}  + e^{-\frac{1}{2} \left(\frac{x^2 - \mathrm{sin}\left( \frac{2 \pi x^1}{4}\right)+ 3\left( 1 + e^{-\frac{x^1 - 1}{0.3}}\right)^{-1}}{0.35} \right)^2} \right),  
  \end{cases}
\end{equation}
The aim is to find a vector field $\theta$ that would steer $\rho$ form $\rho_0$ to $\rho_T$ while minimising the cost functional 
\begin{align}
\mathcal{S} =  \Phi( \rho(x^1,x^2,T), \rho_T(x^1,x^2) )  + \frac{1}{2}\int_0^T \int_{\mathbb{R}^2}  \rho(x^1,x^2,t) \|\theta(x^1,x^2,t) \|^2 \, dx^1 dx^2 dt.
\end{align}

This test was introduced in \cite{rezende15} and since then it has become the de facto benchmark. Here, density matching is done using $\Sigma_{\rho}$ control system and the training optimal control problem is solved using a variational integrator described in section \ref{sec:var_int}. For this particular problem, centred difference discretisation, whose Lagrangian is given by \eqref{eq:discrete_lagrangian}, is used. The end time $T$ was chosen to be $1$ and $16$ discretisation points in space are used. For $x^1$ and $x^2$-spaces, have $32$ grid points for each space. As the optimality condition is a system of algebraic equations \eqref{eq:discrete_ELEQ},it is solved using a newton-type solver at each time step. The result is shown in figure \ref{fig:results}.
\begin{figure}[t]
\includegraphics[width=0.75\textwidth]{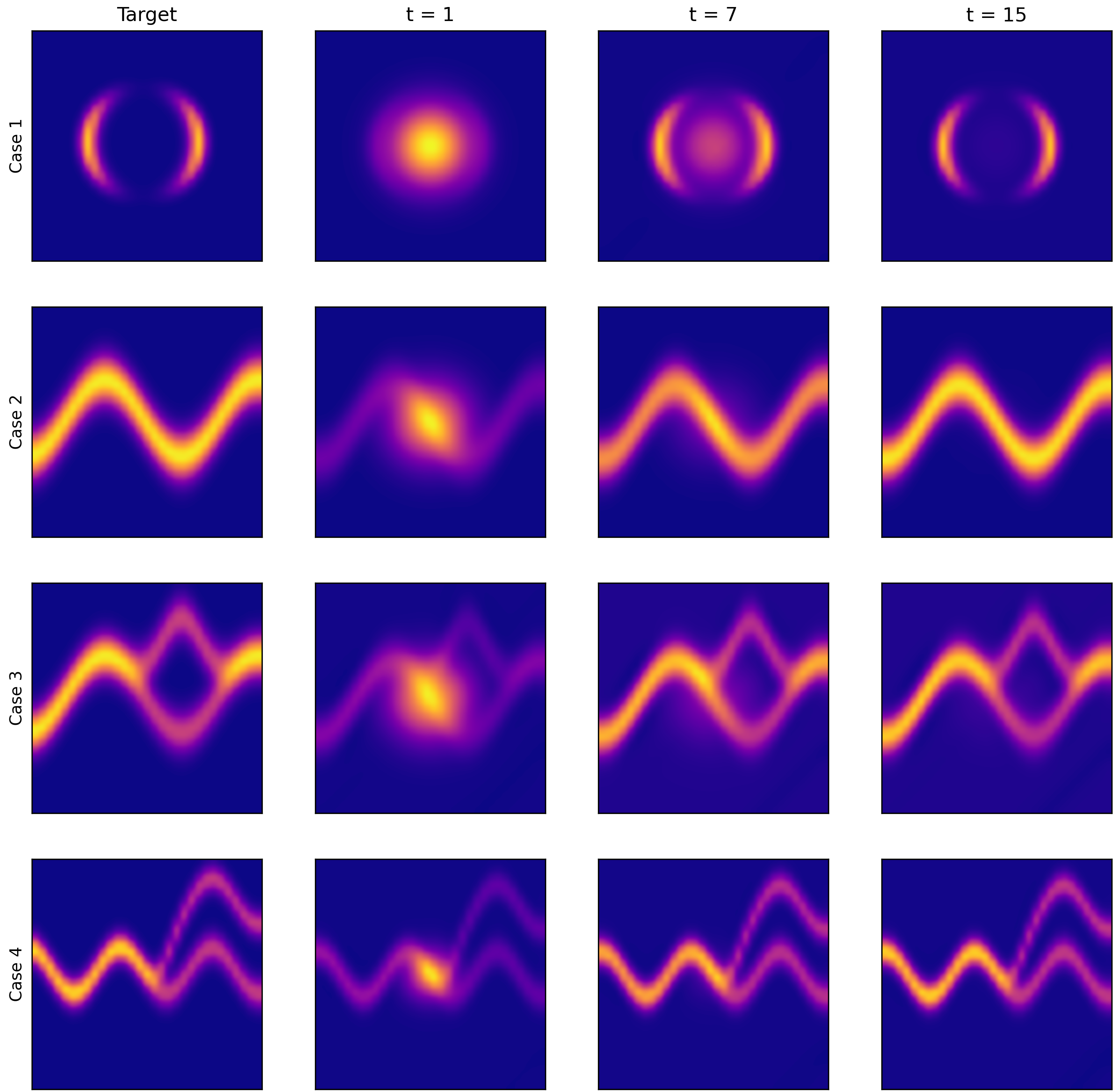}
\centering
\caption{ The problem \ref{prob:mfg_deep_discrete} is solved numerically, with potentials in \eqref{eq:density_cases} are chosen for target densities $\rho_T$ and in each case, the initial density is specified as a Gaussian normal distribution whose mean is zero and its covariance matrix is the identity matrix.  Each row corresponds to a density in \eqref{eq:density_cases}. The left most column is a contour plot of the target density. The second column from the left represents the numerical solution of problem \ref{prob:mfg_deep_discrete} after one time step. The third column from the left is the solution after seven time iterations. Finally, the rightmost column is the solution at the final iteration. }
 \label{fig:results}
\end{figure}

\subsubsection{Generative model based on multisymplectic integrators}\label{sec:autoencoder}
The concept of autoencoders is an interesting one, which are artificial neural networks \cite{kramer_nonlinear_1991, Goodfellow-et-al-2016}. The training for such networks is done in an unsupervised manner by performing input data mimicry, that is learn coded representation, in terms of latent state/features, of the data and from it reconstruct the inputs.  Let $\pi_{Q\mathrm{Z}}:Q \to \mathrm{Z}$ be a fibre bundle and we also define the map $\Gamma_{\mathrm{Z}Q}: \mathrm{Z} \to Q$. Here, an autoencoder consists of an encoder $\pi_{Q\mathrm{Z}}$ and a decoder $\Gamma_{\mathrm{Z}Q}$, where the base manifold $\mathrm{Z}$ is the space of latent state/features. This renders autoencoders suitable for data compression and dimensionality reduction \cite{hinton_reducing_2006} applications. 
\\

Looking at autoencoders from a probabilistic view point, where the latent states are sampled from a probability distribution, it allows us to use them as a generative model. Such models learn the representation of the dataset and allow the generation of new data. Such a class of autoencoders is known as \emph{variational autoencoders} \cite{kingma_auto-encoding_2014}. To generate new data, a random variable is sampled from the distribution $\rho_0 \in \mathrm{Dens}(\mathsf{Z})$ and then feeding into the decoder, $\Gamma_{\mathrm{Z}Q}(z)$, to obtain the reconstructed data. Variational autoencoders are often trained such that latent variables $z$ are sampled from $\mathrm{Dens}(\mathsf{Z})$ would have either a normal or uniform probability distribution to simplify sampling from a distribution. This is done by introducing Kullback-–Leibler regularising term to the cost function.  The downside of forcing the distribution of the parameter space to be an elementary one renders reconstruction less accurate.  In recent years, there has been an interest in the using optimal transport in conjunction with autoencoders to develop an architecture as a substitute or a complement to generative adversarial networks \cite{tanaka_discriminator_2019,salimans_improving_2018,sim_optimal_2020,luise_generalization_2020}. For example, in  \cite{bousquet_optimal_2017} the problem of generative modelling was cast as an optimal transport problem and in \cite{pmlr-v84-genevay18a} the measure and cost function of optimal transport, solved using Sinkhorn algorithm, are approximated by neural networks. Whereas in \cite{yang_predicting_2020}, single-cell images are mapped to a latent space and then the cell trajectories are predicted using optimal transport.  The use of normalizing flows as a generative model is has been studied in works such as \cite{10.5555/3524938.3525652, NEURIPS2018_d139db6a} in what is known as \emph{flow-based generative models}, however such models do not incorporate an autoencoder structure. The use of autoencoder with normalizing flows, replacing optimal transport,  is used for the sampling step generative models and this  have been studied in \cite{ jawahar_improving_2022,morrow2020variational}.
\\

Inspired by \cite{an_ae-ot_2020} and flow-based generative models, the idea is to not to force the variational autoencoder to have the distribution of the latent states as an elementary distribution, instead train the autoencoder to solely perform dimensionality reduction.  Here, a control system $\Sigma_{\rho}$ is used to map $\rho_0 \in \mathrm{Dens}(\mathsf{Z})$, where it is the probability density function of the latent states,  to a normal distribution, $\rho_e \simeq \mathcal{N}(0, \mathbf{I})$. In \cite{an_ae-ot_2020} that was done by solving a Monge--Kantorovich problem. Random variables are sampled from a normal distribution and then pushed-forward by the flow of solution of the Monge--Kantorovich  and finally they are fed into the decoder. However, solving optimal transport problems is challenging and instead we use the multisymplectic variational integrator for that. The density of the space of latent states $\rho_0 \in \mathrm{Dens}(\mathsf{Z})$ is estimated using the presented method here. In the second step of the reconstruction, a latent variable $z$ is sampled from a normal distribution and then transported using diffeomorphisms to a random variable $\widetilde{z}$ whose probability distribution has the density $\rho_0$. The diffeomorphism here is an integral curve of the vector field $u \in \mathfrak{X}(Q)$ and it is learned by solving problem \ref{prob:mfg_deep} using a multisymplectic integrator. Once training is over,  new data can be generated by sampling from $\mathcal{N}(0, \mathbb{I})$ and then mapping it to $\mathsf{Z}$. The transformed data is then fed into the decoder, and the result of one experiment can be seen in figure \ref{fig:results_auto}. The construction presented here can be viewed as an instance of pullback bundle. 
\begin{figure}[t] 
\includegraphics[width=0.6\textwidth]{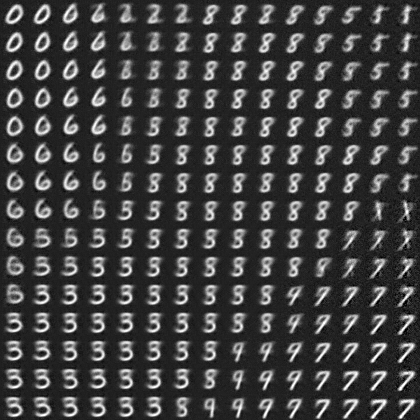}
\centering
\caption{ New handwritten digits are generated using the autoencoder with a continuous ResNet trained using multisymplectic optimal control.  Here, an autoencoder is trained to reconstruct images from MNIST dataset \cite{lecun2010mnist}, where $Q$ is $\mathbb{R}^{28 \times 28}$ and each element is an image of a handwritten digit. The space of latent variable is chosen as $\mathsf{Z} \in \mathbb{R}^2$. A set of points that form a square lattice defined on $\mathbb{R}^2$ are mapped to $\mathsf{Z}$ using the flow of the deep neural network \eqref{eq:resNet_eq}, whose parameters, $\theta = u$, are obtained by solving problem \ref{prob:mfg_deep_discrete}. The points in the transformed lattice are then mapped to $Q$ using the decoder.  }
\label{fig:results_auto}
\end{figure}

\section{Backward error analysis of multisymplectic integrator}\label{sec:bea}
Different numerical schemes can be applied to solve the optimal control problem for training, but as mentioned, schemes based on multisymplectic geometry are chosen in order to obtain an approximation that is closely related to the exact system.  In numerical analysis, there are mainly two ways of analysing the qualitative performance of the scheme; forward analysis, and backward error analysis. In forward error analysis, the solution being a curve is compared to a known integral curve of the system. Many error measurements in machine learning are based on this and it requires knowing the solution or availability of the dataset. Backward error analysis is concerned computing a differential equation whose exact solution is the numerical solution and this equation is known as \emph{modified equation} and it offers a paradigm to conduct a study that enables direct comparison between a continuous multisymplectic system and its discrete counterpart \cite{reich_backward_1999}.
\\

One of the main favourable properties of multisymplectic integrators is that their modified counterparts are also multisymplectic, as demonstrated in \cite{moore_backward_2003, islas_backward_2005}. Once the modified equation is known, we obtain time evolution of the network that explicitly includes parameters such as the number of hidden layers and nodes per layer. This is valuable, as it provides a way to study the qualitative behaviour of the solution in response to the change in the network's parameters. 
\\

In order to compute the modified equations, the solution $ \phi = (w^1, w^2, \lambda, \rho)  \in J^1(E)$ is assumed to be a $4$-tuple of smooth functions. As the multisymplectic integrator uses points in $J^1(\mathbb{E})$ to approximate the solution, their Taylor expansion  around the centre of $\Delta_{i,j}^n$, $\bar{x}$, on the base manifold $Q$ allows to move from discrete points on $\mathbb{E}$ to a section on $\mathbb{E}$. To illustrate that, the Taylor expansion of $\phi_{i+1/2, j+1/2,n+1} $ and $\phi_{i+1/2, j+1/2,n}$ yields:
\begin{equation}
\begin{aligned}
\phi_{i+1/2, j+1/2,n+1} &= \phi(\bar{x}) + \frac{(\Delta t)}{2} \partial_t  \phi(\bar{x})  + \frac{(\Delta t)^2}{8} \partial^2_t\phi(\bar{x})  + \frac{(\Delta t)^3}{48}\partial_{ttt}\phi(\bar{x}) + \mathcal{O}((\Delta t) ^4), \nonumber \\ 
\phi_{i+1/2, j+1/2,n} &= \phi(\bar{x}) - \frac{(\Delta t)}{2} \partial_t  \phi(\bar{x})  + \frac{(\Delta t)^2}{8} \partial^2_t\phi(\bar{x})  - \frac{(\Delta t)^3}{48}\partial_{ttt}\phi(\bar{x}) + \mathcal{O}((\Delta t) ^4). \\  
\end{aligned}
\end{equation}
In addition, the approximation of the section over the $\bar{x}$, which is 
\begin{equation}
\begin{aligned}
\bar{\phi} &= \frac{1}{8} \left( \phi_{i+1, j+1,n+1} +  \phi_{i, j+1,n+1}  +  \phi_{i+1, j,n+1}   +  \phi_{i, j,n+1} \right. \\
&\left. + \phi_{i+1, j+1,n} +  \phi_{i, j+1,n}  +  \phi_{i+1, j,n}   +  \phi_{i, j,n} \right) \\
&= \phi(\bar{x}) + \frac{1}{8} \left( (\Delta t)^2\partial^2_{t} \phi +  (\Delta x^1)^2\partial^2_{x^1} \phi + (\Delta x^2)^2\partial^2_{x^2} \phi \right) \\
&+ \mathcal{O} \left( |(\Delta t)^3|+  |(\Delta x^1)^3| + |(\Delta x^2)^3| \right).
\end{aligned}
\end{equation}
The time discretisation using scheme in \eqref{eq:discrete_lagrangian} is 
\begin{equation}
\begin{aligned}
    \frac{\phi_{i+1/2,j+1/2,n + 1} - \phi_{i+1/2,j+1/2,n}}{(\Delta t)} &= \partial_t \phi  + \frac{(\Delta t)^2}{24} \partial_{ttt} \phi+  \mathcal{O}((\Delta t)^4).  
\end{aligned}
\end{equation}
Likewise, the Taylor expansion of spatial discretisations are 
\begin{equation}\label{eq:midpoint_space}
\begin{aligned}
    \frac{\phi_{i+1,j+1/2,n +1/2} - \phi_{i,j+1/2,n+1/2}}{(\Delta x^1)} &= \partial_{x^1} \phi  + \frac{(\Delta x^1)^2}{24} \partial_{x^1x^1x^1} \phi+  \mathcal{O}((\Delta x^1)^4), \\
    \frac{\phi_{i+1/2,j+1,n+1/2 } - \phi_{i+1/2,j,n+1/2}}{(\Delta x^2)} &= \partial_{x^2} \phi + \frac{(\Delta x^2)^2}{24} \partial_{x^2x^2x^2} \phi+  \mathcal{O}((\Delta x^2)^4).
\end{aligned}
\end{equation}
Substituting the Taylor expansion of the finite difference relations into the stationary variations of the action whose Lagrangian is give by \eqref{eq:discrete_lagrangian} and rearranging the terms of the discretised Fokker--Planck equation, the system of algebraic equations are then become the following system of partial differential equations:
\begin{equation}\label{eq:modified_system}
{\scriptscriptstyle
\begin{aligned}
\delta \lambda &: \quad 0 = \partial_t \rho    + \partial_{x^1} (\rho \phi( w^1)) +\partial_{x^2} (\rho \phi( w^2)) +\frac{(\Delta x^1)^2}{24} \partial_{x^1x^1x^1} (\rho \phi( w^1))    \\
&+\frac{(\Delta x^2)^2}{24} \partial_{x^2x^2x^2} (\rho \phi( w^2))-\frac{(\Delta t)^2}{24}\partial_{ttx^1} (\rho \phi( w^1)) -\frac{(\Delta t)^2}{24}\partial_{ttx^2} (\rho \phi( w^2)), \\
\delta \rho  &: \quad 0 = \partial_t \lambda  +  \phi(w^1)\partial_{x^1} \lambda + \phi(w^2)\partial_{x^2} \lambda  +\frac{(\Delta x^1)^2}{24} \phi(w^1)\partial_{x^1x^1x^1} \lambda   \\
&+\frac{(\Delta x^2)^2}{24} \phi(w^2)\partial_{x^2x^2x^2} \lambda  -\frac{(\Delta x^1)^2}{24} \phi(w^1)\partial_{x^1tt} \lambda  -\frac{(\Delta x^1)^2}{24} \phi(w^2)\partial_{x^2tt} \lambda - \frac{\delta \ell_m}{\delta \rho}, \\
\delta w^1 &: \quad 0 = \frac{\delta \ell_m}{\delta w^1} -\rho \partial_{x^1} \lambda  -\frac{(\Delta x^1)^2}{24}  \rho \partial_{x^1x^1x^1} \lambda +\frac{(\Delta t)^2}{24}  \rho \partial_{x^1tt} \lambda,  \\
\delta w^2 &: \quad 0 = \frac{\delta \ell_m}{\delta w^2} -\rho \partial_{x^2} \lambda   -\frac{(\Delta x^2)^2}{24}  \rho \partial_{x^2x^2x^2} \lambda +\frac{(\Delta t)^2}{24}  \rho \partial_{x^2tt} \lambda.
\end{aligned} }
\end{equation}   
The above system of partial differential equations satisfies the stationary variations of the the action functional 
\begin{align}
\mathcal{S} = \int_0^t \int_{Q} L_m \, d^nx\, dt,
\end{align}
i.e.  $0 = \delta \mathcal{S}$, and with a modified Lagrangian given by
\begin{equation}\label{eq:modified_lagrangian}
\begin{aligned}
L_m &= L + \frac{(\Delta x^1)^2}{24} L_1 +  \frac{(\Delta x^2)^2}{24} L_2 +  \frac{(\Delta t)^2}{24} L_3,
\end{aligned}
\end{equation}
where the terms of the Lagrangian are given by
{\footnotesize
\begin{align*}
L &= + \frac{1}{2} \rho (w^1)^2 + \frac{1}{2} \rho (w^2)^2 + \frac{\nu^{2} }{2}w^1 \partial_{x^1} \rho + \frac{\nu^{2} }{2}w^2 \partial_{x^2} \rho    +\frac{\nu^4 \left(\partial_{x^1} \rho\right)^{2}}{8\rho +  (\Delta x^1)^2 \partial_{x^1x^1} \rho + (\Delta x^2)^2 \partial_{x^2x^2} \rho + (\Delta t)^2  \partial_{tt} \rho}\\
& + \frac{\nu^4 \left(\partial_{x^2} \rho\right)^{2}}{8\rho +  (\Delta x^1)^2 \partial_{x^1x^1} \rho + (\Delta x^2)^2 \partial_{x^2x^2} \rho + (\Delta t)^2  \partial_{tt} \rho } + \lambda \left(\partial_{t} \rho  +  \partial_{x^1}  (\rho (w^1))+   \partial_{x^2}  (\rho (w^2)) \right),  
\end{align*}
}
{\footnotesize
\begin{align*}
L_1 &= \frac{(\Delta x^1)^2}{4} \rho w^1 \partial_{x^1x^1} w^1 + \frac{(\Delta x^1)^2}{4} \rho w^2 \partial_{x^1x^1} w^2 + \frac{(\Delta x^1)^2}{8} \rho \left(\partial_{x^1} w^1\right)^{2}+ \frac{(\Delta x^1)^2}{8} \rho \left(\partial_{x^1} w^2\right)^{2}  \\
& + \frac{\nu^2 (\Delta x^1)^2}{48} w^1 \partial_{x^1x^1x^1} \rho + \frac{\nu^2 (\Delta x^1)^2}{16} (\partial_{x^1} \rho) \partial_{x^1x^1} w^1  + \frac{\nu^2 (\Delta x^1)^2}{16} (\partial_{x^2} \rho) \partial_{x^1x^1} w^2 + \lambda \left(  \partial_{x^1x^1x^1} (w^1 \rho)  \right) 
,  
\end{align*}
}
{\footnotesize
\begin{align*}
L_2 &= \frac{(\Delta x^2)^2}{4} \rho w^1 \partial_{x^2x^2} w^1 + \frac{(\Delta x^2)^2}{4} \rho w^2 \partial_{x^2x^2} w^2 + \frac{(\Delta x^2)^2}{8} \rho \left(\partial_{x^2} w^1\right)^{2} + \frac{(\Delta x^2)^2}{8} \rho \left(\partial_{x^2} w^2\right)^{2} \\
&+ \frac{\nu^2 (\Delta x^2)^2}{48} w^2 \partial_{x^2x^2x^2} \rho + \frac{\nu^2 (\Delta x^2)^2}{16} (\partial_{x^1} \rho) \partial_{x^2x^2} w^1 + \frac{\nu^2 (\Delta x^2)^2}{16} (\partial_{x^2} \rho) \partial_{x^2x^2} w^2 + \lambda \left(   \partial_{x^2x^2x^2}( \rho (w^2) )  \right), 
\end{align*}
}
{\footnotesize
\begin{align*} 
L_3 &=  \frac{(\Delta t)^2}{4} \rho w^1 \partial_{tt} w^1 + \frac{(\Delta t)^2}{4} \rho w^2 \partial_{tt} w^2 + \frac{(\Delta t)^2}{8} \rho \left(\partial_{t} w^1\right)^{2} + \frac{(\Delta t)^2}{8} \rho \left(\partial_{t} w^2\right)^{2}- \frac{\nu^2 (\Delta t)^2}{48} w^1 \partial_{x^1tt} \rho \\
&  - \frac{\nu^2 (\Delta t)^2}{48} w^2 \partial_{x^2tt} \rho  + \frac{\nu^2 (\Delta t)^2}{16} (\partial_{x^1} \rho) \partial_{tt} w^1   + \frac{\nu^2 (\Delta t)^2}{16} (\partial_{x^2} \rho) \partial_{tt} w^2 - \lambda \left( \partial_{x^1tt} (\rho (w^1)) + \partial_{x^2tt} (\rho (w^2))\right)  
.
\end{align*}
}

For the modified equation, the cost Lagrangian $\ell_m$, that appears in \eqref{eq:modified_system}, is 
\begin{equation}
{\small
\begin{aligned}
\ell_m &=+ \frac{1}{2} \rho (w^1)^2 + \frac{1}{2} \rho (w^2)^2 + \frac{\nu^2 w^1 \partial_{x^1} \rho}{2} + \frac{\nu^2 w^2 \partial_{x^2} \rho}{2} \\
&+ \frac{\nu^{4} \left(\partial_{x^1} \rho\right)^{2}}{8\rho +  (\Delta x^1)^2 \partial_{x^1x^1} \rho + (\Delta x^2)^2 \partial_{x^2x^2} \rho + (\Delta t)^2  \partial_{tt} \rho}  + \frac{(\Delta x^1)^2}{4} \rho w^1 \partial_{x^1x^1} w^1\\
&+ \frac{\nu^{4} \left(\partial_{x^2} \rho\right)^{2}}{8\rho +  (\Delta x^1)^2 \partial_{x^1x^1} \rho + (\Delta x^2)^2 \partial_{x^2x^2} \rho + (\Delta t)^2  \partial_{tt} \rho} + \frac{(\Delta x^1)^2}{4} \rho w^2 \partial_{x^1x^1} w^2\\
 & + \frac{(\Delta x^1)^2}{8} \rho \left(\partial_{x^1} w^1\right)^{2} + \frac{(\Delta x^1)^2}{8} \rho \left(\partial_{x^1} w^2\right)^{2}+ \frac{\nu^2 (\Delta x^1)^2}{48} w^1 \partial_{x^1x^1x^1} \rho \\
&+ \frac{\nu^2 (\Delta x^1)^2}{16} (\partial_{x^1} \rho) \partial_{x^1x^1} w^1  + \frac{\nu^2 (\Delta x^1)^2}{16} (\partial_{x^2} \rho) \partial_{x^1x^1} w^2 + \frac{(\Delta x^2)^2}{4} \rho w^1 \partial_{x^2x^2} w^1  \\
& + \frac{(\Delta x^2)^2}{4} \rho w^2 \partial_{x^2x^2} w^2+ \frac{(\Delta x^2)^2}{8} \rho \left(\partial_{x^2} w^1\right)^{2}+ \frac{(\Delta x^2)^2}{8} \rho \left(\partial_{x^2} w^2\right)^{2}   \\
&+ \frac{\nu^2 (\Delta x^2)^2}{48} w^2 \partial_{x^2x^2x^2} \rho + \frac{\nu^2 (\Delta x^2)^2}{16} (\partial_{x^1} \rho) \partial_{x^2x^2} w^1+ \frac{\nu^2 (\Delta x^2)^2}{16} (\partial_{x^2} \rho) \partial_{x^2x^2} w^2 \\
&+ \frac{(\Delta t)^2}{4} \rho w^1 \partial_{tt} w^1 + \frac{(\Delta t)^2}{4} \rho w^2 \partial_{tt} w^2 + \frac{(\Delta t)^2}{8} \rho \left(\partial_{t} w^1\right)^{2} \\
&+ \frac{(\Delta t)^2}{8} \rho \left(\partial_{t} w^2\right)^{2} - \frac{\nu^2 (\Delta t)^2}{48} w^1 \partial_{x^1tt} \rho - \frac{\nu^2 (\Delta t)^2}{48} w^2 \partial_{x^2tt} \rho  + \frac{\nu^2 (\Delta t)^2}{16} (\partial_{x^1} \rho) \partial_{tt} w^1 \\
&+ \frac{\nu^2 (\Delta t)^2}{16} (\partial_{x^2} \rho) \partial_{tt} w^2 \\
\end{aligned}
}
\end{equation}
Unlike the problem \ref{prob:mfg_deep}, the modified system belongs to the third jet prolongation of the Pontryagin's bundle $E := T^{\ast} \mathrm{Dens}(Q) \times \mathfrak{X}(Q)$, i.e $j^3(E)$ which is a manifold with local coordinates
\begin{align*}
j^3(z) &= \left(  x^1,x^2,t , z, z_{x^1x^1}, z_{x^1x^2}, z_t, z_{x^2x^1}, z_{tt},  z_{x^2x^2}, z_{x^1t}, z_{tx^1},  z_{x^2x^1},  z_{x^1x^2},  z_{x^2t},  z_{tx^2}, \right. \\
&\left. z_{x^1x^1x^1},  z_{x^1x^1x^2}, z_{x^1x^1t}, z_{x^1x^2x^1}, z_{x^1x^2x^2}, z_{x^1x^2t},z_{x^1tx^1}, z_{x^1tx^2}, z_{x^1tt}, z_{x^2x^1x^1}, z_{x^2x^1x^2}, \right. \\
&\left. z_{x^2x^1t}, z_{x^2x^2x^1}, z_{x^2x^2x^2}, z_{x^2x^2t}, z_{x^2tx^1}, z_{x^2tx^2}, z_{x^2tt}, z_{tx^1x^1}, z_{tx^1x^2}, z_{tx^1t},  z_{tx^2x^1},  z_{tx^2x^2},\right. \\
&\left.  z_{tx^2t},  z_{ttx^1}, z_{ttx^2}, z_{ttt}   \right),
\end{align*}
where the section $z(x^1,t)$ is a shorthand notation for the tuple
\begin{align*}
(w^1(x^1,x^2,t), w^2(x^1,x^2,t), \lambda(x^1,x^2,t), \rho(x^1,x^2,t)).
\end{align*}
With that the modified Lagrangian density $\widetilde{\mathcal{L}}: j^3(z) \to \Lambda^{n+1}(E)$ described by
\begin{align*}
\widetilde{\mathcal{L}}\left( j^3(z) \right) =  L_m \, dx^1 \wedge dx^2 \wedge dt,
\end{align*}
where $L_m$ is given by \eqref{eq:modified_lagrangian}. 
\\

The introduction of higher order terms in the modified Lagrangian suggests that the solution of the modified equations is a section of the third jet bundle $J^3(E)$. The construction of a third jet bundle whose sections are third jet prolongation of a section $\phi: B \to E$ and they are done by iterating the first jet prolongation $j^1(j^1(j^1(\phi)))$. In coordinates, the third jet bundle $j^3(\phi)$ is given by
\begin{align*}
x^{\mu} \mapsto \left(x^{\mu}, \phi^A, \frac{\partial \phi^A}{\partial x^{\mu} }, \frac{\partial^2 \phi^A}{\partial x^{\mu_1}\partial x^{\mu_2} } , \frac{\partial^3 \phi^A }{\partial x^{\mu_1}\partial x^{\mu_2} \partial x^{\mu_3} } \right).
\end{align*}
With that, we consider a modified Lagrangian density $\mathcal{L}_m: J^3(E) \to \Lambda^{n+1}(B)$, and in coordinates it is written as
\begin{align*}
\mathcal{L}_m = L_m\left(x^{\mu}, \phi^A, \frac{\partial \phi^A}{\partial x^{\mu} }, \frac{\partial^2 \phi^A}{\partial x^{\mu_1}\partial x^{\mu_2} } , \frac{\partial^3 \phi^A }{\partial x^{\mu_1}\partial x^{\mu_2} \partial x^{\mu_3} } \right) dx^{n+1}.
\end{align*}
One quantity that changes with the jet prolongation is the Cartan form on $\Theta_{\mathcal{L}_m}$, which now given by
\begin{equation}
\begin{aligned}
\Theta_{\mathcal{L}_m} &= \left( \frac{\partial L}{\partial \phi^A_{\mu_1} } - D_{\mu_2}\left(\frac{\partial L}{\partial \phi^A_{\mu_1, \mu_2} }  \right)  + D_{\mu_2}D_{\mu_3} \left( \frac{\partial L}{\partial \phi^A_{\mu_1, \mu_2 \mu_3} } \right) \right)d\phi^A \wedge (\frac{\partial}{\partial \mu_1} \contract dx^{n+1})  \\
&+ \left(\frac{\partial L}{\partial \phi^A_{\mu_1, \mu_2} }  - D_{\mu_3}\left( \frac{\partial L}{\partial \phi^A_{\mu_1, \mu_2, \mu_3} } \right) 
\right)d\phi_{\mu_1}^A \wedge \left(\frac{\partial}{\partial \mu_2} \contract dx^{n+1} \right)  \\
&+  \frac{\partial L}{\partial \phi^A_{\mu_1, \mu_2, \mu_3} } d\phi_{\mu_1, \mu_2}^A \wedge \left(\frac{\partial}{\partial \mu_3} \contract dx^{n+1} \right) + \left( L - \frac{\partial L}{\partial \phi^A_{\mu_1} }\phi^A_{\mu_1}  \right. \\
&\left.  + D_{\mu_2}\left(\frac{\partial L}{\partial \phi^A_{\mu_1, \mu_2} }\right) \phi^A_{\mu_1} - D_{\mu_2}D_{\mu_3} \left( \frac{\partial L}{\partial \phi^A_{\mu_1, \mu_2 \mu_3} }\right)\phi^A_{\mu_1}  -   \frac{\partial L}{\partial \phi^A_{\mu_1, \mu_2} } \phi^A_{\mu_1, \mu_2} \right. \\
 &\left. + D_{\mu_3}\left( \frac{\partial L}{\partial \phi^A_{\mu_1, \mu_2, \mu_3} } \right)\phi^A_{\mu_1, \mu_2}  -   \frac{\partial L}{\partial \phi^A_{\mu_1, \mu_2, \mu_3} }  \phi^A_{\mu_1, \mu_2, \mu_3}  \right)dx^{n+1}.
\end{aligned}
\end{equation}
Here, $D_{\mu}$  is the directional derivative of a function $L \in C^{\infty}(J^k(E))$ in the direction $x^{\mu}$ and it is given by
\begin{align}
D_{\mu} \phi = \partial_{\mu} \phi + \frac{\partial L}{\partial \phi_{\mu_1}^A}\phi_{\mu}^A + \dots +  \frac{\partial L}{\partial \phi_{\mu_1, \dots, \mu_k}^A}\phi_{\mu_1, \dots, \mu_k}^A.
\end{align}

Once we have the modified Lagrangian, it is possible to find the conserved quantities and this is discussed at length in the next section. 

\section{Nonlinear stability of the training algorithm}\label{sec:conserv_dl}
One of the most attractive features of the hydrodynamics formulation for training deep neural networks is that allows us to study the nonlinear stability. In order to do so, we need to find the conserved quantities and the corresponding conservation laws associated with the hydrodynamic system \eqref{eq:ep_cont_eqn}. Applying Noether's theorem \cite{Noether1918}, and especially its version for Euler-Poincar\'{e} with advection \cite{Cotter:2013:NTE:3115457.3115685}, conserved quantities are identified using vector fields that generates a symmetry.
\\

First, we define the following manifold $\mathcal{M}$, which is the Hilbert space closure of the manifold
\begin{align*}
\mathcal{M}^{\infty} = \left\{ \phi: U \to E \quad | \quad \pi_{QE} \circ \phi : U \to Q \right\},
\end{align*}
where $U$ is a smooth manifold that is a subset of $Q$ and  $\pi_{QE}: U \to U_{Q}$ is a diffeomorphism, where $U_{Q} = \pi_{QE} \circ \phi(U)$. Further, we define $\mathcal{P}$ to be a manifold defined as
\begin{align*}
\mathcal{P} = \left\{ \phi \in \mathcal{M} \quad | \quad j^1(\phi \circ \phi^{-1}_{Q})^{\ast} \left[ V \contrac \Omega_{\mathcal{L}} \right]  = 0 \quad \forall \, V \in TJ^1(E) \right\}.
\end{align*}
In other words, the elements of $\mathcal{P}$ are the solution of the Euler-Lagrange equation. Next consider the action functional $\mathcal{S}$ on $\mathcal{M}$ and consider the infinitesimal transformation given by $\delta g = \eta \circ g$, where $\eta \in \mathfrak{X}(Q)$ and $g \in \mathrm{Diff}(Q)$ and from it we can derive the infinitesimal transformation for the vector field $v$ which is given by
\begin{align}
\delta w = \dot{\eta} + [ w, \eta], \quad \delta \rho = - \pounds_{\eta} \rho.
\end{align}
Substituting these infinitesimal transformations into $\delta \mathcal{S} = 0$, we obtain
\begin{align}
0 = \int_0^T \int_{U} \left( \frac{\partial }{\partial t} \left\langle \frac{\delta \ell}{\delta w} + \mathrm{ad}^{\ast}_{f(w)}\frac{\delta \ell}{\delta w} + \rho \diamond \frac{\delta \ell}{\delta \rho}\right),\eta \right\rangle d^nx \, dt + \left. \int_{\partial U} \left\langle  \frac{\delta \ell}{\delta w } , \eta \right\rangle d^nx \right|_{0}^T
\end{align}
When the section $\phi \in \mathcal{P}$, the first integral is equal to zero and that leaves us with 
\begin{align}
0 = \left. \int_{\partial U} \left\langle  \frac{\delta L}{\delta w } , \eta \right\rangle d^nx  \right|_{0}^T.
\end{align}
The importance of this is that it allows us to find integrals of motion for the Euler--Poincar\'{e} system and one of them is the exterior derivative of the one-form part of $\frac{\delta \ell}{\delta w}$. For fluids, it is known as \emph{vorticity}. In \cite{Cotter:2013:NTE:3115457.3115685}, it was shown how Noether's theorem can be used to derive local conservation of vorticity. The same calculation is repeated here and it starts by equating the following time derivative to zero
\begin{align}\label{eq:noether_vorticity}
0 = \frac{d}{dt} \left\langle  \frac{\delta \ell}{\delta w}, \eta \right\rangle = \frac{d}{dt}  \int_{U}  \frac{\delta \ell}{\delta w} \cdot dx \wedge \left( \eta \contract \rho d^nx\right),
\end{align} 
then substituting the vector field $\eta$ with $ \frac{1}{\rho}\left(  \star d\left( \boldsymbol{\Psi}^{\flat} \right)\right)^{\sharp} \cdot \nabla$, where $\boldsymbol{\Psi}$ is an arbitrary function. The operator $\star: \Lambda^k \to \Lambda^{N-k}$ is the Hodge star operator and it is a linear operator, which maps the $k$ forms on $N$ dimensional manifold to $n-k$ forms provided that $0 \leq k \leq n$. Meanwhile, the operators $\flat: \Lambda^k \to \left( \Lambda^k \right)^{\ast}$ and $\sharp: \left( \Lambda^k \right)^{\ast} \to \Lambda^k$ are musical isomorphisms. In $2$ and $3$ dimensions, $\left(  \star d\left( \Psi^{\flat} \right)\right)^{\sharp} = \mathrm{curl}(\boldsymbol{\Psi})$.  The calculation carried out yields
\begin{equation}\label{eq:vorticity_eq}
\begin{aligned}
d\left( \partial_t + \pounds_{\theta} \right) \left(  \frac{1}{\rho} \frac{\delta \ell}{\delta w}   \cdot d x \right) &= \left( \partial_t + \pounds_{\theta} \right) d\left(  \frac{1}{\rho} \frac{\delta \ell}{\delta w}   \cdot d x \right) = 0 , 
\end{aligned}
\end{equation}
Replacing $d\left(  \frac{1}{\rho} \frac{\delta \ell}{\delta w}   \cdot d x \right)$ by $\omega \cdot dS$, where the differential form $d S$ is defined, by components, by $dS_i = \frac{1}{2}\epsilon_{i_1, i_2, \dots, i_N}dx^{i_j} \wedge dx^{i_k}$ , where $\epsilon_{i_1, i_2, \dots, i_N}$ is the Levi-Civita symbol. Then equation \eqref{eq:vorticity_eq}, becomes
\begin{equation}\label{eq:vorticity_eq2}
\begin{aligned}
\left( \partial_t + \pounds_{\theta} \right) \left(  \omega \cdot dS \right) =  \left( \partial_t \omega +\left( \omega \cdot \nabla \right)\theta - \left( \theta \cdot \nabla \right)\omega  + \omega \nabla \cdot \theta  \right) \cdot dS = 0.
\end{aligned}
\end{equation}
In two or three dimensions, it is written in terms of curl operator
\begin{align*}
\left( \partial_t + \pounds_{\theta} \right) \left( \mathrm{curl}\left( \frac{1}{\rho} \frac{\delta \ell}{\delta w} \right) \cdot d S \right) = 0.
\end{align*}
Going back to equation \eqref{eq:vorticity_eq2}, the term $\left( \omega \cdot \nabla \right)\theta - \left( \theta \cdot \nabla \right)\omega$ is the Lie derivative $\pounds_{\theta} \omega$, thus the coefficient of the two-form $dS$, $\omega$, evolves according to the following directional derivative
\begin{align*}
\frac{d}{dt} \omega = \left( \partial_t + \pounds_{\theta} \right) \omega = - \omega \nabla \cdot \theta.
\end{align*}
Using the fact that the coefficient of the continuity equation can be written as
\begin{align}
\frac{d}{dt} \rho = -\rho \nabla \cdot \theta,
\end{align}
and we use that in the equation $\frac{d}{dt} \omega = \frac{1}{\rho}\frac{d}{dt} \rho$, which leads to 
\begin{align}\label{eq:conserve_quantity_1}
\frac{d}{dt} \left(\frac{\omega}{\rho} \right) = 0,
\end{align}
and it means that the quantity $\frac{\omega}{\rho}$ is a conserved quantity. In section \ref{sec:nonlinear_stability}, this quantity plays an important role for investigating the nonlinear stability of the deep learning training problem. 
\\

A special version of Noether's theorem for Euler--Poincar\'{e} equations and it is known as \emph{Kelvin-Noether's theorem} \cite{holm1998euler}.It includes Kelvin's circulation theorem where it states that the circulation of a closed loop transported in an barotropic fluid is a constant. Given a circulation map $\kappa : \mathfrak{C} \to \mathfrak{X}(Q)^{\ast \ast}$, where $\mathfrak{C}$ is the space of continuous loops in $Q$. The circulation map is defined by
\begin{align} \label{eq:circulation}
\left\langle \kappa( C) , \alpha \right\rangle = \oint_{C} \frac{\alpha}{\rho},
\end{align}
where $C \in \mathfrak{C}$ is a loop and $\alpha \in  \mathfrak{X}(Q)^{\ast}$ is a one-form density and dividing it by $\rho$ renders it a one-form. This path integral is known as \emph{circulation} \cite{holm1998euler}. According to Kelvin-Noether's theorem (theorem $4.1$ in \cite{holm1998euler}), and given the quantity, known as \emph{Kelvin-Noether quantity}, $I : \mathfrak{C} \times \mathfrak{X}(Q) \to \mathbb{R}$ defined by
\begin{align*}
I(C, v,t) = \left\langle \kappa(C), \frac{\delta \ell}{\delta w} \right\rangle,
\end{align*}
and its time evolution is governed by the following differential equation
\begin{align*}
\frac{d}{dt}I(C, w,t) = \left\langle \kappa(C(t)),  \rho \diamond \frac{\delta \ell}{\delta \rho} \right\rangle,
\end{align*}
where $C(t) = C_0 g(t)^{-1}$ is a loop in $\mathfrak{C}$, $C_0 \in \mathfrak{C}$ and $g$ is the solution of $\dot{g} = R_{g^{-1}}w$ and $w, \rho $ satisfy the Euler--Poincar\'{e} equation. Then, Kelvin's circulation theorem can be stated in terms of Kelvin-Noether's theorem as
\begin{align*}
I(C, w,t)  = \oint_{C} \frac{1}{\rho} \left( \frac{\delta \ell}{\delta w} \right).
\end{align*}

\begin{lemma}[Conservation of Kelvin's circulation for deep learning]\label{lem:kelvin_circ}
Given $(w, \rho)$, which satisfies the Euler--Poincar\'{e} equation \eqref{eq:ep_cont_eqn}, and let the diffeomorphism $g$ is the solution of $\dot{g} = w \circ g(t)$. With the loop $C(t) = g \circ C_0$, then Kelvin's circulation
\begin{align}
I(C,w,t) = \oint_{C} \frac{1}{\rho} \left( \frac{\delta \ell}{\delta w} \right),
\end{align}
is conserved.
\end{lemma}

\begin{proof}
In order to prove the conservation of $I$, Kelvin's circulation is rewritten as
\begin{align}
I(C, w,t) = \oint_{C_0} (g)^{\ast} \left( \frac{1}{\rho} \left( \frac{\delta \ell}{\delta w} \right) \right) = \oint_{C_0} \frac{1}{\rho_0} (g)^{\ast} \left( \frac{\delta \ell}{\delta w} \right),
\end{align}
Taking the time derivative of $I$ and noticing that $\frac{d}{dt} g^{\ast}(\frac{\delta \ell}{\delta w} )$ is the Lie derivative and it constitute the left--hand side of the Euler--Poincar\'{e} equation \eqref{eq:ep_cont_eqn}, then
\begin{equation}
\begin{aligned}
\frac{d}{dt} I = \frac{d}{dt}\oint_{C_0} \frac{1}{\rho_0} (g)^{\ast} \left( \frac{\delta \ell}{\delta w} \right) &=  \frac{d}{dt}\oint_{C_0} \frac{1}{\rho_0} g^{\ast} \left( \rho \diamond \frac{\delta \ell}{\delta \rho} \right) \\
&= \oint_{C} \frac{1}{\rho}\left( \rho \nabla \frac{\delta \ell}{\delta \rho} \cdot dx \right), \\
&=  \oint_{C} \frac{1}{\rho}\left( \rho d \left( \frac{\delta \ell}{\delta \rho} \right) \right),
\end{aligned}
\end{equation}
Then invoking Stoke's theorem, the last integral is equal to the integral of the exterior derivative of the potential $\nabla \frac{\delta \ell}{\delta \rho}$ and it amounts to zero, i.e.
\begin{align}
\frac{d}{dt} I = \oint_{C}    d \left( \frac{\delta \ell}{\delta \rho} \right)= \int d^2 \left( \frac{\delta \ell}{\delta \rho} \right) = 0.
\end{align}
\end{proof}
\begin{remark}
An alternative derivation of conservation laws is by using the approach of \cite{bridges1997multi}, where it uses the multisymplectic conservation law
\begin{align}\label{eq:conservation_ms}
\partial_t \omega + \partial_{x_i} \vartheta_i = 0,
\end{align}
with the pre-symplectic forms
\begin{align*}
\omega(W, V) = \left\langle MW, V \right\rangle , \quad \text{and} \quad \vartheta_i(W,V) = \left\langle K_i W,V \right\rangle,
\end{align*}
where vector fields $W,V$ are arbitrary vector fields on the phase space and $M$ and $K_i$ are degenerate skew-symmetric matrices defined in \eqref{eq:pre_time}.  The formula for conservation symplecticity \eqref{eq:conservation_ms}, in this form can be derived from the multisymplectic form formula \eqref{eq:multisymplectic_form_formula} as shown in \cite{marsden1998multisymplectic}.\emph{The local conservation of energy} is defined by
\begin{align*}
\partial_t E(z) + \partial_{x_i} F_i(z) = 0, \quad \text{where} \quad E(z) = H(z) - \frac{1}{2} \vartheta_i(z_{x_i},z)  , \quad F(z) = \frac{1}{2} \vartheta_i(z_{t},z).
\end{align*}
The corresponding global energy conservation law is obtained by first taking the integral of $E(z)$ and then computing the total derivative of the integral, i.e.
\begin{align}
\frac{d}{dt} \int E(z) \, dx = 0.
\end{align}
Likewise, \emph{the local conservation of momentum} is when $E(z) = \frac{1}{2} \omega(z_{x_i},z)$ and $F(z) = H(z) - \frac{1}{2} \omega(z_{t},z)$.
\end{remark}

\subsection{Hamiltonian formulation of deep learning}\label{sec:Hamiltonian}
In the section \ref{sec:mul_geo_dl}, we have introduced the Euler--Poincar\'{e} equation \eqref{eq:EP_deep_learning} for deep learning. The equation uses the Lagrangian description of dynamics, however, for studying the nonlinear stability of deep learning, it is required to use the Hamiltonian formulation. The main difference is that in the Lagrangian picture the dynamics are defined on the Lie algebra, while the Hamiltonian dynamics belong to the dual of a Lie algebra, which is a Poisson manifold \cite{marsden_coadjoint_1983,AST_1985__S131__421_0}.  For this particular deep learning problem, the Hamiltonian point of view was used in \cite{ganaba_deep_2021-2}. The manifold is the space of densities on manifold $Q$, thus the technical details are left out of this note. Let $G$ be a Lie group with $\mathfrak{g}$ as the associated Lie algebra and  $\mathfrak{g}^{\ast}$  is the dual of the Lie algebra.  On $\mathfrak{g}^{\ast}$ we have the following bracket
\begin{align}
\left\{ F, G\right\}_{\pm}(\mu) = \pm \left\langle \mu, \left[ \frac{\delta F}{\delta \mu}, \frac{\delta G}{\delta \mu} \right] \right\rangle,
\end{align}
where $F, G: \mathfrak{g}^{\ast} \to \mathbb{R}$ are functions on $\mathfrak{g}^{\ast}$, and it the bracket $\left\{ F, G\right\}_{\pm}(\mu)$ is known as \emph{Lie--Poisson bracket}. The deep learning Hamiltonian density $\mathcal{H}:  J^1(E)^{\ast} \to \Lambda^{n+1}(Q)$, first mentioned in \eqref{eq:hamilonian_generic} and it is given by
\begin{align}\label{eq:Hamiltonian_deeplearning}
\mathcal{H} = \left(  \frac{1}{2}   \rho \| w  \|^2  + \frac{\nu^4}{8}   \rho \left\|  \nabla \mathrm{log}(\rho) \right\|^2  \right)d^nx \, dt.
\end{align}
The main assumption here is that $\nabla \cdot w = 0$. Here, we consider the Hamiltonian with the variables the momentum, $\mu = \rho w + \frac{\nu^2}{2} \nabla \rho$, and the probability density function, $\rho$.  The Lie--Poisson bracket of deep learning 
\begin{equation} \label{eq:LP_bracket}
\begin{aligned}
\{ F, G \}(\mu, \rho) &= \int_{Q} \mu \cdot \left[ \left(   \frac{\delta G}{\delta \mu} \cdot \nabla \right) \frac{\delta F}{\delta \mu} - \left(   \frac{\delta F}{\delta \mu} \cdot \nabla \right) \frac{\delta G}{\delta \mu}\right] d^nx \\
&+ \int_{Q}  \rho \left[ \left(   \frac{\delta G}{\delta \mu }\cdot \nabla \right) \frac{\delta F}{\delta \rho} - \left(   \frac{\delta F}{\delta \mu }\cdot \nabla \right) \frac{\delta G}{\delta \rho}\right] d^nx \\
&= \left\langle (\mu, \rho), \left[  \frac{\delta F}{\delta (\mu, \rho)}, \frac{\delta G}{\delta (\mu, \rho)}\right] \right\rangle.
\end{aligned}
\end{equation}
This bracket is in fact a Poisson bracket on the dual of a semidirect product Lie algebra $\mathfrak{s} = \mathfrak{X}(Q) \times \mathcal{F}(Q)$. The bracket \eqref{eq:LP_bracket} does not only appear in deep learning, but originally it appeared in hydrodynamics \cite{bialynicki-birula_canonical_1973, PhysRevLett.45.790, holm_poisson_1983, marsden1983hamiltonian} and in semidirect reduction \cite{marsden_semidirect_1984}.  With that the Euler--Poincar\'{e} equation \eqref{eq:EP_deep_learning} is written as the Lie--Poisson on semidirect product space $\mathfrak{s}^{\ast} = \mathfrak{X}^{\ast}(Q) \times \mathrm{Dens}(Q)$
\begin{equation}\label{eq:LP_deep_eqn}
\begin{aligned}
\frac{d}{dt} \mu &= \pm \mathrm{ad}^{\ast}_{\frac{\delta H}{\delta \mu}} \mu \mp \frac{\delta H}{\delta \rho} \diamond \rho, \\
\frac{d}{dt} \rho &= \mp \frac{\delta H}{\delta \mu} \rho,
\end{aligned}
\end{equation}
and it can be written in terms of Lie--Poisson bracket \eqref{eq:LP_bracket} as
\begin{equation}
\begin{aligned}
\partial_t (\mu, \rho)&= \pm\{ (\mu, \rho), H \}.
\end{aligned}
\end{equation}

\subsection{Review of the energy--Casimir method}\label{sec:ECM}
Up to this point, terms like linearised and nonlinear stability were used indistinguishably, and in fact, stability can refer to various ideas. There are different concepts of stability such as: spectral stability, linear stability, formal stability and nonlinear stability, and the difference between each of them is detailed in \cite{HOLM19851} and references therein.  They all describe how the system responds to perturbations around its equilibrium solution.  In general, stability requires a measure of distance, $\mathrm{d}$, and a trajectory $\mu_e$ is said to be stable if a trajectory that starts close to $\mu_e$, it remains close for all $ t > 0$.  Each of these meanings require different stability criteria, and in some cases one implies the other. However, the system at hand here is nonlinear and infinite-dimensional, and, for example, formal stability does not imply nonlinear stability. In this article, the concept of stability that is of concern is nonlinear stability obtained by the energy--Casimir method \cite{HOLM19851}.

\begin{definition}[Nonlinear Stability \cite{HOLM19851}]
For the equilibrium trajectory $\mu_e$, we define a neighbourhood of $\mu_e$, $U_1$, by $\mathrm{d}( \mu(t), \mu_e) < \delta$, where $\delta \in \mathbb{R}_{>0}$. We also define another neighbourhood of $\mu_e$,  $U_2$, by $\mathrm{d}( \mu(t), \mu_e) < \epsilon$, where $\epsilon \in \mathbb{R}_{>0}$. If a trajectory $\mu(t)$, whose initial value $\mu(0) \in U_2$, for every $\epsilon$ used to define $U_2$, there is $\delta$, such that $\mu(t) \in U_1$ for all $t> 0$, then $\mu(t)$ is said to be nonlinearly stable. 
\end{definition}
What is different from linearised stability is that here, the trajectory $\mu(t)$ satisfies a nonlinear differential equation. The goal here is to find the norm used in defining the distance $\mathrm{d}$. Energy--Casimir method provides a way of doing that.  Traditionally, stability conditions are based on a linearised dynamics of the system being examined, and the definiteness of the Lyapunov function. One popular choice for the Lyapunov function is the energy function. For nonlinear stability, the energy function being the Hamiltonian is augmented with conserved quantities arising from symmetries. From there a norm is constructed using $H$ and $C$ in which the perturbations to the equilibrium are bounded.  Here, the conserved quantity $C$ is chosen to be the \emph{Casimir function}, which satisfies $\{ C, F \} = 0$ for any function $F$ that maps from the phase-space to $\mathbb{R}$.  Using a function of single variable $\varphi$, we can construct a family of Casimir functions $C_{\varphi} = \varphi \circ C$. Proceeding forward, given an fixed-point $\mu_e$ solution of \eqref{eq:LP_deep_eqn}, the energy--Casimir method requires the following steps to determine the stability of $\mu_s$:
\begin{itemize}
\item Find $\varphi$, such that when evaluated at $\mu_s$, the first variation of $H + C$ vanishes.
\item Check the convexity of $H + C$ by first finding the quadratic forms $Q_1$, and $Q_2$ such that they satisfy the following inequalities
\begin{align*}
Q_1(\Delta \mu) &\leq H(\mu_s + \Delta \mu) - H(\mu_s) - \left\langle \delta H(\mu_s), \Delta \mu \right\rangle, \\
Q_2(\Delta \mu) &\leq C_{\varphi}(\mu_s + \Delta \mu) - C_{\varphi}(\mu_s) - \left\langle \delta C_{\varphi}(\mu_s), \Delta \mu \right\rangle, 
\end{align*}
where $\Delta \mu = \mu - \mu_s$. Further, an additional requirement is that
\begin{align*}
Q_1(\Delta \mu) + Q_2(\Delta \mu)  > 0,
\end{align*}
for all $\Delta u \neq 0$.
\item Ensure the continuity of $H_C$ by asserting that
\begin{align*}
\left| H(\mu_s + \Delta \mu) - H(\mu_s) - \left\langle \delta H(\mu_s),  \Delta \mu \right\rangle  \right|&\leq C_1 \| \Delta \mu \|^{\alpha}, \\
\left| C_{\varphi}(\mu_s + \Delta \mu) - C_{\varphi}(\mu_s) - \left\langle \delta C_{\varphi}(\mu_s),  \Delta \mu \right\rangle \right| &\leq C_2 \| \Delta \mu \|^{\alpha},
\end{align*}
where $C_1, C_2, \alpha > 0$.
\end{itemize}
Earlier, a distance $\mathrm{d}$ is used to define nonlinear stability, and once the aforementioned steps are fulfilled, we set the the distance as
\begin{align}
\mathrm{d}(w_1, w_2) = \| w_1 - w_2 \|^2 = Q_1(w_1 - w_2) + Q_2(w_1 - w_2) > 0, 
\end{align}
for $w_1 - w_2 \neq 0$. Moreover, the existence of $C_1, C_2$ imply that the distance between the trajectory $\mu(t)$ and $\mu_s$ is bounded
\begin{align*}
\mathrm{d}(\mu(t), \mu_s) \leq (C_1 + C_2)\mathrm{d}(\mu(0), \mu).
\end{align*}
One consequence of directly examining the nonlinear stability of the system \eqref{eq:LP_deep_eqn} is that an equilibrium solution $(\mu_s, \rho_s)$ might not be meaningful for deep learning. This necessitates a different approach. Here, we aim to use analysis that is based on both forward and backward error analysis. The latter is used to derive the modified system and use it to derive an evolution equation for the error between the numerical and the exact solution. This way, the error equation's equilibrium solution is the trivial solution and the aim of the analysis here becomes the nonlinear stability when the error is close to zero. 
\begin{remark}
For equations in the Euler--Poincar\'{e} family, especially the dispersionless Camassa-Holm equation \cite{PhysRevLett.71.1661}, choosing the zero solution as the initial condition and perturbing the equation results in the emergence of nonzero solution. 
\end{remark}
Assume that there exists a solution \eqref{eq:LP_deep_eqn}, which solves the training problem \ref{prob:mfg_deep_discrete}. The numerical solution of obtained by the multisymplectic integrator is Lagrangian and it can be formulated as a Hamiltonian system using Legendre transform $\mathbb{F}\mathcal{L}_m: J^3(E) \to (J^3(E))^{\ast}$ defined by
\begin{equation}
\begin{aligned}
H_m &= \left( \frac{\partial L}{\partial \phi_{\mu_1}^A} - D_{\mu_2}\left( \frac{\partial L}{\partial \phi_{\mu_1 \mu_2}^A}  \right)+ D_{\mu_3\mu_2}\left( \frac{\partial L}{\partial \phi_{\mu_1 \mu_2 \mu_3}^A}  \right)   \right) \frac{\partial \phi^A}{\partial {\mu_1}} \\
&+ \left( \frac{\partial L}{\partial \phi_{\mu_1\mu_2}^A} -  D_{\mu_3}\left( \frac{\partial L}{\partial \phi_{\mu_1 \mu_2 \mu_3}^A}  \right)   \right)\frac{\partial \phi^A }{\partial {\mu_1 \mu_2}} + \left( \frac{\partial L}{\partial \phi_{\mu_1\mu_2 \mu_3}^A}  \right)\frac{\partial \phi^A }{\partial {\mu_1 \mu_2 \mu_3}} - L,
\end{aligned}
\end{equation}
and the Hamiltonian analogue of the modified Lagrangian \eqref{eq:modified_lagrangian} is 
{\footnotesize
\begin{equation}\label{eq:modified_Hamiltonian}
\begin{aligned}
H_m &=\frac{1}{2} \rho (w^1)^2 + \frac{1}{2} \rho (w^2)^2  + \frac{(\Delta x^1)^2}{4} \rho w^1 \partial_{x^1x^1} w^1 + \frac{(\Delta x^1)^2}{4} \rho w^2 \partial_{x^1x^1} w^2 + \frac{(\Delta x^1)^2}{8} \rho \left(\partial_{x^1} w^1\right)^{2} \\
&+ \frac{(\Delta x^1)^2}{8} \rho \left(\partial_{x^1} w^2\right)^{2}   + \frac{(\Delta x^2)^2}{4} \rho w^1 \partial_{x^2x^2} w^1 + \frac{(\Delta x^2)^2}{4} \rho w^2 \partial_{x^2x^2} w^2 + \frac{(\Delta x^2)^2}{8} \rho \left(\partial_{x^2} w^1\right)^{2} \\
&+ \frac{(\Delta x^2)^2}{8} \rho \left(\partial_{x^2} w^2\right)^{2}  + \frac{(\Delta t)^2}{4} \rho w^1 \partial_{tt} w^1 + \frac{(\Delta t)^2}{4} \rho w^2 \partial_{tt} w^2 + \frac{(\Delta t)^2}{8} \rho \left(\partial_{t} w^1\right)^{2} \\
&+ \frac{(\Delta t)^2}{8} \rho \left(\partial_{t} w^2\right)^{2} + \frac{\nu^{4} \left(\partial_{x^1} \rho\right)^{2} + \nu^{4} \left(\partial_{x^2} \rho\right)^{2}}{8\rho +  (\Delta x^1)^2 \partial_{x^1x^1} \rho + (\Delta x^2)^2 \partial_{x^2x^2} \rho + (\Delta t)^2  \partial_{tt} \rho}  \\
\end{aligned}
\end{equation}
}
Unlike the case with linear dynamical systems where the error is taken as the difference between the numerical and exact trajectories, here the error is defined as a Lie group element. The reason for that is that we are comparing two trajectories on the same manifold, but their dynamics live on different spaces, which is an obstacle in deriving an evolution equation for the error. Moreover, the error on Lie group allows the derivation of the Lie--Poisson equation for the error in order to apply the energy--Casimir method.  The first step is to define the \emph{configuration error function}  \cite{FB-ADL:04}, denoted by $\Psi: G \times G \to \mathbb{R}$. The configuration error function need to satisfy a the following properties: it needs to be smooth, symmetric, i.e. for $g_0, g_1, \in G$ we have $\Psi(g_0, g_1) = \Psi(g_1, g_0)$, positive definite and $\Psi(g,g) = 0$ for any $g \in G$. One candidate for configuration error function is
\begin{align*}
\Psi(g, h) = \mathrm{Tr}\left( \| \mathrm{Id} - g^{-1}h \|^2 \right) = \mathrm{Tr}\left( \| \mathrm{Id} - h^{-1}g \|^2 \right),
\end{align*}
where $g, h \in G$ and $\mathrm{Id}$ is the identity element of $G$. Here, $g^{-1}h$ is the left action of $g^{-1}$ on $h$. We set $g_e = g_s^{-1}g_m$ as \emph{the group error}, where $g_s \in G$ is the analytic solution of deep learning training problem and $g_m$ is the numerical solution obtained using the variational multisymplectic integrator, whose modified Hamiltonian is given by \eqref{eq:modified_Hamiltonian}. The idea is the closest $g_m$ to the analytic solution, the smaller is $\Psi(g_r, g)$ value. Before computing the evolution equation of the error on $TG$, we use the fact the solution of the Lie--Poisson equation on $\mathfrak{g}^{\ast}$ can be used to reconstruct the curves on $T^{\ast}G$. Part of this is the reconstruction equation and, for the analytic solution $\mu_s$ of the training problem, we assume that with the Hamiltonian, $H: \mathfrak{s}^{\ast} \to \mathbb{R}$, it satisfies
\begin{align*}
\partial_t \left( \mu_s, \rho_s \right) =  \left( \mathrm{ad}^{\ast}_{\frac{\delta H_s}{\delta \mu_s}} \mu_s + \frac{\delta H_s}{\delta \rho_s} \diamond \rho_s,  \frac{\delta H_s}{\delta \mu_s}\rho_s \right), 
\end{align*}
The reconstruction equation on $T^{\ast}S$, where $S$ is the group whose Lie algebra is $\mathfrak{X}(Q) \times \mathcal{F}(Q)$, for the analytic solution is then
\begin{align*}
\frac{d g_s}{dt}&= L_{g_s}\frac{\delta H_s}{\delta \mu_s}, \\
\text{and} \quad \frac{d v_s}{dt}&= L_{g_s}\frac{\delta H_s}{\delta \rho_s}.
\end{align*}
Similarly, the numerical solution of the training problem satisfies the modified equation with Hamiltonian $H_m: (J^3(E))^{\ast} \to \mathbb{R}$ and its Lie--Poisson and corresponding reconstruction equation
{\small
\begin{align*}
\partial_t \left( \mu_m, \rho_m \right) &=  \left( \mathrm{ad}^{\ast}_{\frac{\delta H_m}{\delta \mu_m}} \mu + \frac{\delta H_m}{\delta \rho_m} \diamond \rho_m,  \frac{\delta H_m}{\delta \mu_m}\rho_m \right), \\
\quad \frac{d g_m}{dt} &= L_{g_m}\frac{\delta H_m}{\delta \mu_m}, \\
\text{and} \quad  \frac{d v_m}{dt} &= L_{g}\frac{\delta H_m}{\delta \rho_m}.
\end{align*}
}
Now, with $(g_e, v_e) = (g_s^{-1}g_m,  -g_s^{-1}v_m + g_s^{-1}v_s)$, its time evolution is
{\footnotesize
\begin{align*}
\frac{d}{dt}(g_e, v_e) &= \left( - g_s^{-1}\frac{d g_s}{dt} g_s^{-1}g_m + g^{-1}_s\frac{d g_m}{dt}, -g_s^{-1}\frac{d g_s}{dt} g_s^{-1}v_m + g_s^{-1}\frac{d v}{dt} + g_s^{-1}\frac{d g_s}{dt} g_s^{-1}v_s - g_s^{-1}\frac{d v_s}{dt}\right) \\
&= \left( - \left[\frac{\delta H_s}{\delta \mu} - \mathrm{Ad}_{g_e}\frac{\delta H_m}{\delta \mu} \right]g_e, -\frac{\delta H_s}{\delta \mu} v_e  + g_e \frac{\delta H_m}{\delta \rho} - \frac{\delta H_s}{\delta \rho} \right). 
\end{align*}
}
Defining the coadjoint action of the error group action $(g_e, v_e)$ on $(\mu_{e,0}, \rho_0)$, by $ (\mu_{e}, \rho_{e}) = \mathrm{Ad}^{\ast}_{(g_e, v_e)^{-1} }  (\mu_{e,0}, \rho_0) \in \mathfrak{g}^{\ast}$, its time evolution is governed by
\begin{equation}\label{eq:error_poisson}
\begin{aligned}
\partial_t \left( \mu_{e}, \rho_{e} \right) &=  \left( \mathrm{ad}^{\ast}_{ \left(\frac{\delta H_s}{\delta \mu} - \mathrm{Ad}_{g_e}\frac{\delta H_m}{\delta \mu} \right)} \mu_{e} + \left( -\left[ \mathrm{Ad}_{g_e} \frac{\delta H_m}{\delta \mu}\right] v_e + g_e \frac{\delta H_m}{\delta \rho} - \frac{\delta H_s}{\delta \rho}  \right)\diamond  \rho_{e} ,  \right. \\
  &\left. \left( \frac{\delta H_s}{\delta \mu} - \mathrm{Ad}_{g_e}\frac{\delta H_m}{\delta \mu} \right)\rho_{e} \right),
\end{aligned}
\end{equation}
Further, let us define a Hamiltonian function $H_{\mathrm{err}}$, as
\begin{equation}\label{eq:Error_Hamiltonian}
\begin{aligned}
H_{\mathrm{err}}(\mu_{e}, \rho_{e}, t)&=  H_s(\mu_e, \rho_e) +  H_m(\mathrm{Ad}^{\ast}_{g_e^{-1}} (\mu_e, \rho_e), t).
\end{aligned}
\end{equation} and the functional derivative are
\begin{align*}
\frac{\delta H_{\mathrm{err}}}{\delta \mu_{e}} &= \frac{\delta H_s}{\delta \mu } - \mathrm{Ad}_{g_e}\frac{\delta H_m}{\delta \mu }, \\
\frac{\delta H_{\mathrm{err}}}{\delta \rho_{e}} &= -\left[ \mathrm{Ad}_{g_e} \frac{\delta H_m}{\delta \mu }\right] v_e + g_e \frac{\delta H_m}{\delta \rho } - \frac{\delta H_s}{\delta \rho },
\end{align*}
and a candidate for that is $H_{\mathrm{err}} = \frac{1}{2}\left\langle  \mu_{e}, \frac{\delta H_s}{\delta \mu } - \mathrm{Ad}_{g_e}\frac{\delta H_m}{\delta \mu } \right\rangle$. Here, for the deep learning training, $H_s$ is \eqref{eq:Hamiltonian_deeplearning} and $H_m$ is exactly \eqref{eq:modified_Hamiltonian}. In order to carry out with stability analysis, we need to write $H_{\mathrm{err}}$ in terms of $(\mu_{e}, \rho_{e})$ and that is done by using $(\mu_{e}, \rho_{e}) = \mathrm{Ad}^{\ast}_{(g_e, v_e)^{-1}} (\mu_{e,0}, \rho_{e,0})$. Let $(\mu_{e,0}, \rho_{e,0})$ be the error between the exact solution and the numerical one at time $t = 0$ and it is set as $(\mu_{e,0}, \rho_{e,0}) = (\mu_{s, 0} - \mu_{m, 0}, \rho_{s,0} - \rho_{m,0})$. At the same time, we have the following $(\mu_{s}, \rho_s) = \mathrm{Ad}^{\ast}_{(g_s, v_s)^{-1}} (\mu_{s,0}, \rho_{s,0})$ and $(\mu_m, \rho_m) = \mathrm{Ad}^{\ast}_{(g, v)^{-1}} (\mu_{m,0}, \rho_{m,0})$. Then, we can write $(\mu_{e}, \rho_{e})$ in terms of $(\mu_{s}, \rho_s)$ and $(\mu_m, \rho_m)$ using
\begin{equation}
\begin{aligned}
\left( \mu_{e}, \rho_{e} \right) &=  \mathrm{Ad}^{\ast}_{(g_e, v_e)^{-1}} \left(\mathrm{Ad}^{\ast}_{(g_s, v_s)}( \mu , \rho ) -  \mathrm{Ad}^{\ast}_{(g_m,v_m)}(\mu, \rho)\right) \\
&= \mathrm{Ad}^{\ast}_{(g_e, v_e)^{-1}} \left(  \mathrm{Ad}^{\ast}_{g_s} \mu -  \mathrm{Ad}^{\ast}_{g_m} \mu   + v_s \diamond ((g_s^{-1})^{\ast}\rho ) \right. \\
&+\left. v_m \diamond ((g_m^{-1})^{\ast}\rho ), ((g_s^{-1})^{\ast}\rho  - ((g_m^{-1})^{\ast}\rho  \right) \\
&=  \left( \mathrm{Ad}^{\ast}_{g_e^{-1}} \mathrm{Ad}^{\ast}_{g_s} \mu - \mathrm{Ad}^{\ast}_{g_e^{-1}} \mathrm{Ad}^{\ast}_{g_m} \mu  + \mathrm{Ad}^{\ast}_{g_e^{-1}}\left(  v_s \diamond ((g_s^{-1})^{\ast}\rho ) \right)  \right. \\
&\left.  + \mathrm{Ad}^{\ast}_{g_e^{-1}}\left(  v_m \diamond ((g_m^{-1})^{\ast}\rho ) \right)    + v_e \diamond \left( ((g_m^{-1})^{\ast}\rho  - ((g_s^{-1})^{\ast}\rho \right), \right. \\
&\left  ((g_m^{-1})^{\ast}\rho  - ((g_s^{-1})^{\ast}\rho   \right).
\end{aligned}
\end{equation}



\subsection{Nonlinear stability of deep learning}\label{sec:nonlinear_stability}
From the Lie--Poisson structure of the deep learning problem defined on the semidirect product group of diffeomorphisms and densities, $\mathrm{Diff}(Q) \circledS \mathrm{Dens}(Q)$, the Casimir function should be the same as for the two-dimensional barotropic flow \cite{HOLM198315} and it is given by
\begin{align*}
C_{\varphi} = \int_Q \rho \varphi \left(\frac{\omega}{\rho} \right) d^nx,
\end{align*}
where $\varphi: \mathbb{R} \to \mathbb{R}$ is an arbitrary function and
\begin{align}
\omega = d\left( \frac{1}{\rho} \frac{\delta \ell}{\delta w} \cdot dx \right).
\end{align}
and the quantity $\frac{\omega}{\rho}$ is conserved according to equation \eqref{eq:conserve_quantity_1}.  Another conserved quantity that, stated in lemma \ref{lem:kelvin_circ}, is the Kelvin's circulation
\begin{align*}
I_i(\mu, \rho) = \oint_{(\partial Q)_i} \frac{1}{\rho}\frac{\delta \ell}{\delta w} \cdot dx,
\end{align*}
where $(\partial Q)_i$ is a component of the boundary of the manifold $Q$. Now, we have the Casimir $C_{\varphi}$ and Kelvin's circulations $I$, we augment to the error Hamiltonian $H_{\mathrm{err}}$, to get the following new Hamiltonian 
\begin{align}
H_C(\mu, \rho) = H_{\mathrm{err}}(\mu, \rho) + C_{\varphi}(\mu, \rho) + \sum\limits_{i}a_iI_i(\mu, \rho).
\end{align}
We assume that the equilibrium solution of \eqref{eq:error_poisson} is $(\mu_e, \rho_e)$. Applying the the energy--Casimir method, where the calculation it closely follows \cite{HOLM198315} and for the sake of self-containing of this article, it is reproduced in Appendix \ref{sec:calculations}. For this particular application, the nonlinear stability is measured with the metric $\mathrm{d}(\mu, \rho)$, the same as in \cite{HOLM198315}, defined by
\begin{align*}
\mathrm{d}(\mu, \rho) = \left\| (\Delta \mu, \Delta \rho) \right\|^2 = Q_1(\Delta \mu, \Delta \rho) + Q_2(\Delta \mu, \Delta \rho),
\end{align*}
where 
\begin{align*}
Q_1(\Delta \mu, \Delta \rho) &= \frac{1}{2}\int_Q \| \Delta (\mu \rho) \|^2  + \left( \frac{c_1^2}{\rho_{\mathrm{max}}} - \frac{|\mu_e |^2}{\rho_{\mathrm{min}}}\right) \left\| \Delta \rho \right\|^2 d^n x, \\
\text{and} \quad Q_2(\Delta \mu, \Delta \rho) &= \frac{c_2}{2}  \rho_{\mathrm{min}} \int_{Q} \left(\Delta \left(\frac{\omega}{\rho}\right)\right) ^2 d^nx.
\end{align*}
and $c_1$ and $c_2$ must chosen so that 
\begin{align*}
c_1 \leq H_{\mathrm{err}}'', \quad c_2 \leq \frac{1}{z}K'(z), 
\end{align*}
where 
\begin{align}\label{eq:Bernoulli_Func}
K \left(\frac{\omega}{\rho}\right) =  -\left[ \mathrm{Ad}_{g_e} \frac{\delta H_m}{\delta \mu}\right] v_e + g_e \frac{\delta H_m}{\delta \rho} - \frac{\delta H_s}{\delta \rho }.
\end{align}
\begin{remark}
In fluid dynamics, the function $K$ is the Bernoulli Law.
\end{remark}
\begin{theorem}[Nonlinear stability theorem for deep learning]
Let $\Sigma = (Q, f, \mathfrak{X}(Q))$ be a stochastic control system, which describe the limit when the number of hidden layers of a residual neural networks approaches infinity, and $\Sigma_{\rho} = (\mathrm{Dens}(Q), f, \mathfrak{X}(Q))$ be the mean--field type control equivalent.  Given the training dataset $\{(x_0^{(i)}, c^{(i)} \}_{i=0}^N$, and the cost functional
\begin{align*}
\mathcal{S}_{\Sigma_{\rho} } = \sum\limits_{i=0}^N \int_Q \Phi(\pi(\rho(x,T)), c^{(i)}) dx + \int_0^T \int_Q \rho L( w + \nabla \mathrm{log}(\rho))dx dt,
\end{align*}
where $L: \mathfrak{X}(Q) \to \mathbb{R}$ is the cost function. The network parameters $\theta^{\ast}$ and trajectory $\rho^{ \ast}$ that satisfy
\begin{align*}
\mathcal{S}_{\Sigma_{\rho} }(\rho^{ \ast}, \theta^{\ast}) < \mathcal{S}_{\Sigma_{\rho} }(\rho, \theta), \quad \forall (\rho, \theta) \in \mathrm{Dens}(Q) \times \mathfrak{X}(Q).
\end{align*}
Let $(\widetilde{\rho}, \widetilde{\theta})$ be the numerical solution obtained by the multisymplectic variational integrator whose Lagrangian is \eqref{eq:discrete_lagrangian}, and it satisfies the modified Euler--Poincar\'{e} equation \eqref{eq:ep_cont_eqn}. Given the norm defined by
{\small
\begin{align}
\| (\mu, \rho) \|^2 = \frac{1}{2} \int_Q  \| \Delta (\mu \rho) \|^2  + \left( \frac{c_1^2}{\rho_{\mathrm{max}}} - \frac{|\mu_e |^2}{\rho_{\mathrm{min}}}\right) \left\| \Delta \rho \right\|^2 d^nx + c_2 \int_Q \left\| \Delta \left( \frac{\omega}{\rho}\right) \right\|^2 d^nx.
\end{align}
}
The error between the numerical error $(\widetilde{\rho}, \widetilde{\theta})$ and the exact analytical solution $(\rho^{\ast}, \theta^{\ast})$, the following inequality holds
\begin{align}
\| (\widetilde{\mu}(t) -\mu^{\ast} , \widetilde{\rho}(t) -\rho^{\ast} ) \|^2 \leq (C_1 + C_2) \| (\widetilde{\mu}(0) -\mu^{\ast} , \widetilde{\rho}(0) -\rho^{\ast} ) \|^2,
\end{align}
provided it satisfies
\begin{equation}\label{eq:stability_bound}
\begin{aligned}
0 &< c_1 < H_{\mathrm{err}}''< \infty, \\
0 &< c_2 < K' < \infty,
\end{aligned}
\end{equation}
where $K$ is \eqref{eq:Bernoulli_Func} and the Hamiltonian $H_{\mathrm{err}}$ defined by \eqref{eq:Error_Hamiltonian}. 
\end{theorem}
From the second line in \eqref{eq:stability_bound}, the bound depends on the derivative of the function, $K$, and the its interest for deep learning is that parameters such as the number of nodes per hidden layers, number of hidden layers and learning rate appear directly. Thus, such parameters directly influence the nonlinear stability of the training's error. Consider the numerical solution of the deep learning training problem $(\widetilde{\mu}, \widetilde{\rho})$ and assume that there exists an exact solution of the training problem $(\mu^{\ast}, \rho^{\ast})$. Here, we consider the case when $(\widetilde{\mu}, \widetilde{\rho})$ is close to $(\mu^{\ast}, \rho^{\ast})$, and making the approximation that $(g_e, v_e) \approx (\mathrm{Id}, 0)$. As $K$ is in terms of $z$, we introduces the vector valued map $\zeta: \mathbb{R} \to Q$. Thus, we can  parametrise $w$ and density $\rho$ in $K$ by $z$. The reduced function \eqref{eq:Bernoulli_Func} is 
{\footnotesize
\begin{align*}
K\left( z\right) &\approx \frac{\delta H_m}{\delta \rho} - \frac{\delta H_s}{\delta \rho }\\
&\approx \frac{(\Delta x^j)^2}{4} \left(  w^j \partial_{x^jx^j} w^j +  w^i \partial_{x^jx^j} w^j + \frac{1}{2} \left(\partial_{x^j} w^j\right)^{2} + \frac{1}{2} \left(\partial_{x^j} w^i\right)^{2} \right) \\
&+ \frac{(\Delta t)^2}{4} \left( w^j \partial_{tt} w^j + \frac{1}{2} \left(\frac{\partial}{\partial t} w^j\right)^{2} + \frac{1}{2} \left(\frac{\partial}{\partial t} w^i\right)^{2} \right) \\
&+ \frac{\nu^4(\Delta x^j)^2}{8}\frac{ (\partial_{x^j} \rho)  (\partial_{x^jx^jx^j} \rho)  + (\partial_{x^i} \rho)  (\partial_{x^ix^jx^j} \rho) }{\left( \rho + \frac{(\Delta x^j)^2}{8}\partial_{x^jx^j} \rho  + \frac{(\Delta t)^2}{8}\partial_{tt} \rho   \right)^2} + \frac{\nu^4(\Delta t)^2}{8} \frac{ (\partial_{x^j} \rho)  \partial_{x^jtt} \rho + (\partial_{x^i} \rho)  \partial_{x^itt} \rho}{\left( \rho + \frac{(\Delta x^j)^2}{8}\partial_{x^jx^j} \rho + \frac{(\Delta t)^2}{8}\partial_{tt} \rho   \right)^2}  \\
&- \frac{\nu^4(\Delta x^j)^2}{4}\frac{ \rho^2 (\partial_{x^j} \rho)  (\partial_{x^jx^jx^j} \rho)  +  \rho^2 (\partial_{x^i} \rho)  (\partial_{x^ix^jx^j} \rho)  +  \rho \left( \partial_{x^j} \rho  \right)^{2} (\partial_{x^jx^j} \rho)  +  \rho (\partial_{x^jx^j} \rho)  \left( \partial_{x^i} \rho \right)^{2}}{\left( \rho + \frac{(\Delta x^j)^2}{8}\partial_{x^jx^j} \rho  + \frac{(\Delta t)^2}{8}\partial_{tt} \rho   \right)^4} \\
&- \frac{\nu^4(\Delta t)^2}{4} \frac{ \rho^2 (\partial_{x^j} \rho)  \partial_{x^jtt} \rho + \rho (\partial_{tt} \rho)  \left( \partial_{x^j} \rho  \right)^{2} + \rho (\partial_{tt} \rho)  \left( \partial_{x^i} \rho  \right)^{2}}{\left( \rho + \frac{(\Delta x^j)^2}{8}\partial_{x^jx^j} \rho  + \frac{(\Delta t)^2}{8}\partial_{tt} \rho   \right)^4} \\
\end{align*}
}
Taking the derivative of $K$ with respect to $z$, requires the chain rule, and for now we write it abstractly we get 
\begin{align*}
K'(z) &= \frac{\partial}{\partial \zeta^j}\left(\frac{\delta H_m}{\delta \rho} - \frac{\delta H_s}{\delta \rho } \right) \left(\partial_z \zeta^j \right).
\end{align*}
The derivative $\partial_{z} \zeta^j$ is computed by taking the derivative of $\zeta$ with respect to $x^j$ using the chain rule and then solve for $\partial_{z} \zeta^j$ to obtain
\begin{align}\label{eq:zeta_deriv}
\partial_{z} \zeta^j= \frac{\frac{1}{\rho}\frac{\partial \omega}{\partial x^j} -  \frac{\omega}{\rho^2} \frac{\partial \rho}{\partial x^j}}{\sum\limits_{i = 1}^n(\frac{1}{\rho}\frac{\partial \omega}{\partial x^i} -  \frac{\omega}{\rho^2} \frac{\partial \rho}{\partial x^i})^2}.
\end{align}
In order to ensure that the inequality is valid, we need $\| \partial_z \zeta^j \| \neq 0$ and it is true only if
\begin{align}\label{eq:conserved_deriv}
\frac{1}{\rho}\frac{\partial \omega}{\partial x^j} - \frac{\omega}{\rho^2} \frac{\partial \rho}{\partial x^j} \neq 0.
\end{align}

The reason for that is because $0< c_2 < K'$ ceases to be valid when $\partial_z \zeta^j  = 0$ and this is easy to see when writing it in terms of chain rule and H\"{o}lder inequality
\begin{align*}
c_2 < \| \partial_{\zeta^j} K \|^2 \| \partial_z \zeta^j \|^2.
\end{align*}

The exterior derivative of the one for $\omega = d\left( \frac{1}{\rho} \frac{\delta \ell}{\delta w} \cdot dx \right)$ when expanded, we obtain
\begin{align*}
\omega = d\left( w\cdot dx + \nabla \mathrm{log}{\rho}\cdot dx + \frac{(\Delta x^j)^2}{8} \partial_{x^jx^j} w\cdot dx+ \frac{(\Delta t)^2}{8} \partial_{tt} w\cdot dx\right) 
\end{align*}
We substitute it into \eqref{eq:conserved_deriv} and it becomes
\begin{equation}
\begin{aligned}
\frac{1}{\rho}\frac{\partial \omega}{\partial x^j} &+ \frac{(\Delta x^j)^2}{8}\frac{1}{\rho}\frac{\partial \omega_1}{\partial x^j} + \frac{(\Delta t)^2}{8}\frac{1}{\rho}\frac{\partial \omega_2}{\partial x^j} \\
&- \frac{\omega}{\rho^2} \frac{\partial \rho}{\partial x^j} - \frac{(\Delta x^j)^2}{8}\frac{\omega_1}{\rho^2} \frac{\partial \rho}{\partial x^j}- \frac{(\Delta t)^2}{8}\frac{\omega_2}{\rho^2} \frac{\partial \rho}{\partial x^j} \neq 0.
\end{aligned}
\end{equation}
and here we impose that this relation always positive and we write it as the following inequality
\begin{align*}
\partial_{x^j} \left( \frac{\omega}{\rho} \right) + \frac{(\Delta x^j)^2}{8}\partial_{x^j} \left( \frac{\omega_1}{\rho} \right) + \frac{(\Delta t)^2}{8}\partial_{x^j} \left( \frac{\omega_2}{\rho} \right) > 0.
\end{align*}
Multiplying both sides by a function $\chi$, such that  $\partial_{x^j} \chi \neq 0$ and $\| - \partial_{x^j} \chi \|_{\infty}  > 0$, we get the following inequality
\begin{align*}
0 &< \int_{Q} \left( \partial_{x^j} \left( \frac{\omega}{\rho} \right) \right)\chi + \frac{(\Delta x^j)^2}{8}\left(\partial_{x^j} \left( \frac{\omega_1}{\rho} \right) \right)\chi + \frac{(\Delta t)^2}{8} \left(\partial_{x^j} \left( \frac{\omega_2}{\rho} \right) \right)\chi \\
&< \int_{Q}\frac{\omega}{\rho} \left( - \partial_{x^j} \chi \right) + \frac{(\Delta x^j)^2}{8} \frac{\omega_1}{\rho} \left( - \partial_{x^j} \chi \right) + \frac{(\Delta t)^2}{8}\frac{\omega_2}{\rho} \left( - \partial_{x^j} \chi \right) \\
&< \| \frac{\omega}{\rho} + \frac{(\Delta x^j)^2}{8} \frac{\omega_1}{\rho} + \frac{(\Delta t)^2}{8}  \frac{\omega_2}{\rho}   \|_{1} \| - \partial_{x^j} \chi \|_{\infty}.
\end{align*}
Utilising the H\"{o}lder inequality again, the above inequality can turns into
\begin{align*}
0 &< \| \frac{\omega}{\rho} + \frac{(\Delta x^j)^2}{8} \frac{\omega_1}{\rho} + \frac{(\Delta t)^2}{8}  \frac{\omega_2}{\rho}   \|_{1} \\
&< \|  \omega  + \frac{(\Delta x^j)^2}{8}  \omega_1 + \frac{(\Delta t)^2}{8}   \omega_2  \|_{\infty} \| \frac{1}{\rho} \|_{1}
\end{align*}
and with from the Triangle inequality, we have the following relationship
\begin{align*}
0 <  \|  \omega\|_{\infty}  + \frac{(\Delta x^j)^2}{8}  \| \omega_1 \|_{\infty}+ \frac{(\Delta t)^2}{8}  \| \omega_2  \|_{\infty},
\end{align*}
which results in
\begin{align*}
\frac{\frac{(\Delta x^j)^2}{8}  \| \omega_1 \|_{\infty}+ \frac{(\Delta t)^2}{8}  \| \omega_2  \|_{\infty}}{\|  \omega\|_{\infty} } < 1.
\end{align*}
\begin{corollary}
Problem \eqref{prob:mfg_deep_discrete} for training of stochastic neural network is nonlinearly stable, if the number of of hidden layers $N_t$ and the number of nodes per layer $\prod\limits_{k = 0}^d N_{x^k}$ satisfy
\begin{align}
\frac{ \|\omega_1 \|_{\infty}}{8 (N_{x^j})^2 \|\omega \|_{\infty}}+ \frac{\|\omega_2\|_{\infty}}{8(N_{t})^2 \|\omega \|_{\infty}} < 1.
\end{align} 
\end{corollary}

\section{Conclusion}\label{sec:conclusion}
In this paper, we have formulated deep learning training as a multisymplectic system that appears in hydrodynamics. Using mean--field type control to solve the training of a stochastic deep network, the resultant necessary conditions for optimality yield a system of equations which we have shown that it can be reduced to the Euler-Poincar\'{e} equation. The family of equations are the link between deep learning and fluids and they have rich geometric properties. The underlying geometric structure of the reduced equations have allowed us to derived a structure--preserving multisymplectic integrator to solve the training problem. We have derived the modified equation for the integrator and we have applied energy--Casimir method and we have shown that the error of the training procedure is nonlinearly stable. Using other related variational principles and different cost Lagrangian that yield interesting singular solutions is the focus of future work.
\bibliography{references}


\bibliographystyle{abbrv}

\begin{appendices}
\section{Energy--Casimir method for deep learning: calculation}\label{sec:calculations}
Here, we carry out the calculation of the energy--Casimir method and it follows many of the steps in the energy--Casimir calculation for the two-dimensional barotropic flow \cite{HOLM198315}. Starting with the error Hamiltonian with integrals of motions
\begin{align}
H_C(\mu, \rho) = H_{\mathrm{err}}(\mu, \rho) + C_{\varphi}(\mu, \rho) + \sum\limits_{i}a_iI_i(\mu, \rho).
\end{align}
The equilibrium solution makes the first variation of the Hamiltonian $H_C$ vanish when the following condition for all $\delta \mu, \delta \rho$ holds:
\begin{equation}\label{eq:first_var_HC}
\begin{aligned}
0 &= \left\langle \frac{\delta H_C}{(\delta \mu, \delta \rho)}, (\delta \mu, \delta \rho) \right\rangle  \\
&= \int_{\mathrm{Q}}  \left\langle \rho_{\mathrm{err}}\left[  \frac{\delta H_s}{\delta \mu } - \mathrm{Ad}_{g_e}\frac{\delta H_m}{\delta \mu}  -  \sum\limits_{i_1, i_2, i_3 = 1}^{d+1} \epsilon_{i_1, i_2,i_3 } \mathbf{e}_{i_1} \delta_{i_2,d+1} \partial_{x^{i_3}} \varphi'\left(\frac{\omega}{\rho}\right)   \right], \delta \mu \right\rangle  \\
&+ \int_{\mathrm{Q}} \left[   -\left( \mathrm{Ad}_{g_e} \frac{\delta H_m}{\delta \mu}\right) v_e   + g_e \frac{\delta H_m}{\delta \rho} - \frac{\delta H_s}{\delta \rho }  + \varphi\left(\frac{\omega}{\rho}\right)  -  \varphi'\left(\frac{\omega}{\rho}\right)\frac{\omega}{\rho}  \right]  \delta \rho \\
&+\sum\limits_{i} \oint_{\partial Q_i} \varphi'\left(\frac{\omega}{\rho}\right)   \delta \mu \cdot dx + a_i\delta \mu \cdot dx. 
\end{aligned}
\end{equation}
From the first variation of $H_C$, computed in \eqref{eq:first_var_HC}, holds provided the following are satisfied
\begin{align}
a_i = -  \varphi' \left( \left. \frac{\omega_e}{\rho_e} \right|_{ ( \partial Q)_i} \right),
\end{align}
and 
\begin{align}\label{eq:Cauchy_Euler}
K(z) + \varphi(z) -z\varphi'(z) = 0,
\end{align}
where $K$ is the variational derivative of $H_{\mathrm{err}}$ with respect to $\rho_{\mathrm{err}}$ and it is
\begin{align}\label{eq:Bernoulli_Func1}
K \left(\frac{\omega}{\rho}\right) =  -\left[ \mathrm{Ad}_{g_e} \frac{\delta H_m}{\delta \mu}\right] v_e + g_e \frac{\delta H_m}{\delta \rho} - \frac{\delta H_s}{\delta \rho }.
\end{align}
Solving the differential equation \eqref{eq:Cauchy_Euler}, we obtain
\begin{align*}
\varphi(z) = z\left( \int^z \frac{K(t)}{t^2} \, dt  + C \right),
\end{align*}
where $C \in \mathbb{R}$ is a constant. For the convexity estimates, we only perform it for $Q_2$, as $Q_1$ is set 
\begin{align*}
Q_1(\Delta \mu, \Delta \rho) &= \frac{1}{2}\int_Q \| \Delta (\mu \rho) \|^2  + \left( \frac{c_1^2}{\rho_{\mathrm{max}}} - \frac{|\mu_e |^2}{\rho_{\mathrm{min}}}\right) \left\| \Delta \rho \right\|^2 d^n x,
\end{align*}
where $c_1 \in \mathbb{R}$ and it should satisfy $c_1 < \delta^2 H$. The $L^2$ norm for $\Delta \rho$ is suitable as the terms in $H$ are derivatives of $\Delta \rho$ and according to Poincar\'{e} inequality, then $L^2$ is bounded by Sobolev norms. On the other hand, $Q_2$ is required to satisfy 
\begin{align*}
Q_2(\Delta \mu, \Delta \rho) \leq \int_Q  \varphi \left(\frac{\omega }{\rho } + \Delta \left(\frac{\omega}{\rho}\right) \right)    - \varphi\left(\frac{\omega}{\rho}\right) - \varphi'\left(\frac{\omega}{\rho}\right)\cdot \left( \Delta \left(\frac{\omega}{\rho}\right)\right) d^nx.
\end{align*}
One candidate is the $L^2$ norm 
\begin{align}
Q_2 \left(\Delta \left(\frac{\omega}{\rho}\right)\right) = \frac{c_2}{2}  \rho_{\mathrm{min}} \int_{Q} \left(\Delta \left(\frac{\omega}{\rho}\right)\right) ^2 d^nx,
\end{align}
where $c_2 \in \mathbb{R}$ is a constant and it is required to satisfy  $c_2 \leq \varphi''(z)$ and $c_2 > 0$ for all $z$. Thus, the convexity holds when $c_1, c_2 >0$, $c_1 \leq  H_{\mathrm{err}}''$ and $c_2 \leq \varphi''$. Next, we determine a priori estimates of $Q_1$ and $Q_2$ and using the inequalities for them we arrive at
\begin{align*}
Q_1 \left(\Delta \mu(t), \Delta \rho(t) \right) + Q_2 \left(\Delta \mu(t), \Delta \rho(t) \right)  &\leq H_C(\mu + \Delta \mu(t) , \rho + \Delta \rho(t) ) \\
&- H_C(\mu , \rho ) - \left\langle \delta H_C, (\delta \mu, \delta \rho) \right\rangle  \\
&\leq H_C(\mu + \Delta \mu(t) , \rho + \Delta \rho(t) ) - H_C(\mu , \rho ).
\end{align*}
The term $\left\langle \delta H_C, (\delta \mu, \delta \rho) \right\rangle$ vanishes as a result of setting the first variation of $H_C$ to be zero at the equilibrium. Since $H_C$ is conserved by the Lie--Poisson flow, it is conserved, hence 
\begin{align*}
Q_1 \left(\Delta \mu(t), \Delta \rho(t) \right) + Q_2 \left(\Delta \mu(t), \Delta \rho(t) \right)&\leq H_C(\mu + \Delta \mu(0) , \rho + \Delta \rho(0) ) - H_C(\mu, \rho).
\end{align*}
This means that $Q_1 + Q_2$ is bounded by the initial values $(\Delta \mu(0), \Delta \rho(0))$.  After establishing the convexity and a priori estimate for $Q_1$ and $Q_2$, we set the norm 
\begin{align}
\left\| (\Delta \mu, \Delta \rho) \right\|^2 = Q_1(\Delta \mu, \Delta \rho) + Q_2(\Delta \mu, \Delta \rho),
\end{align}
and with 
\begin{align*}
c_1 \leq H_{\mathrm{err}}'', \quad c_2 \leq \frac{1}{z}K'(z),
\end{align*}
we have Lyapunov stability and the error is bounded. 
\end{appendices}

\end{document}